%% file: main.tex
\documentclass[11pt,letterpaper,twocolumn,teaser]{planstyle}
\input{macro} 

\definecolor{fire0}{HTML}{FFF2B2} 
\definecolor{fire1}{HTML}{C1121F} 
\definecolor{fire2}{HTML}{FFB347} 
\definecolor{fire3}{HTML}{FF8A3D} 
\definecolor{fire4}{HTML}{FF4D00} 
\definecolor{fire5}{HTML}{FF4A3A} 
\definecolor{fire6}{HTML}{CE0A18} 

\definecolor{UTAustin}{HTML}{BF5700}
\definecolor{pyraAmber}{HTML}{FFB347} 
\definecolor{pyraOrange}{HTML}{FF6B3D} 
\definecolor{pyraRed}{HTML}{E63946} 

\definecolor{pyraPurple}{HTML}{824db5}
\definecolor{pyraYellow}{HTML}{d8cd7c}
\definecolor{pyraGreen}{HTML}{54931c}
\definecolor{pyraBlue}{HTML}{36aece}
\definecolor{pyraGray}{HTML}{b3b3b3}
\definecolor{pyraPink}{HTML}{c1447d}

\newtheorem{proposition}{Proposition}

\newcommand{\affmark}[2]{
  {\textsuperscript{\textcolor{#1}{\raisebox{0.15ex}{\hspace{-0.1em}\scriptsize$#2$}}}}%
}

\newcommand{\AffUIUC}{\affmark{IllinoisOrange}{\blacklozenge}}
\newcommand{\AffInd}{\affmark{Blue}{\clubsuit}}
\newcommand{\AffUT}{\affmark{UTAustin}{\varheartsuit}}
\newcommand{\AffGoogle}{\affmark{ForestGreen}{\spadesuit}}

\newcommand{\modelname}{\textbf{{\textcolor{pyraPurple}{P}\textcolor{pyraPink}{y}\textcolor{pyraGreen}{r}\textcolor{pyraBlue}{a}\textcolor{pyraGray}{Tok}}}\xspace}
\newcommand{\modelnamenc}{PyraTok\xspace}
\newcommand{\modelnamecpa}{LaPQ\xspace}

\title{%
  \begin{tabular}{@{}c@{\hspace{8pt}}l@{}}
    \raisebox{-1.7em}{\includegraphics[width=2cm]{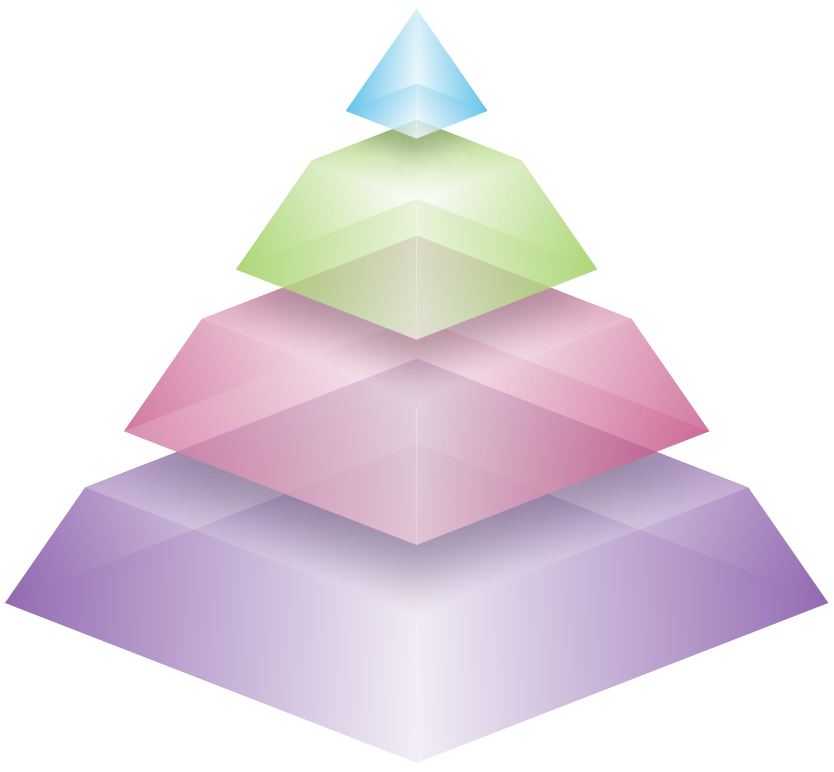}} &
    \begin{tabular}[t]{@{}l@{}}
      \textbf{\modelname: Language-Aligned \textcolor{pyraPurple}{P}\textcolor{pyraPink}{y}\textcolor{pyraGreen}{r}\textcolor{pyraBlue}{a}midal} {\textcolor{pyraGray}{Tok}enizer} \\
      \textbf{for Video Understanding and Generation}
    \end{tabular}
  \end{tabular}%
}

\author{
\vspace{-0.5cm}
\begin{tabular}{c}
\textbf{Onkar Susladkar\AffUIUC \quad Tushar Prakash\AffInd \quad Adheesh Juvekar\AffUIUC \quad Kiet A. Nguyen\AffUIUC}\\ 
\textbf{Dong-Hwan Jang\AffUIUC \quad Inderjit S Dhillon\AffUT\AffGoogle \quad Ismini Lourentzou\AffUIUC}
\end{tabular}
}

\renewcommand{\insertteaserfigure}{
  \begin{center}
    \centering
    \captionsetup{type=figure} \includegraphics[width=0.99\linewidth]{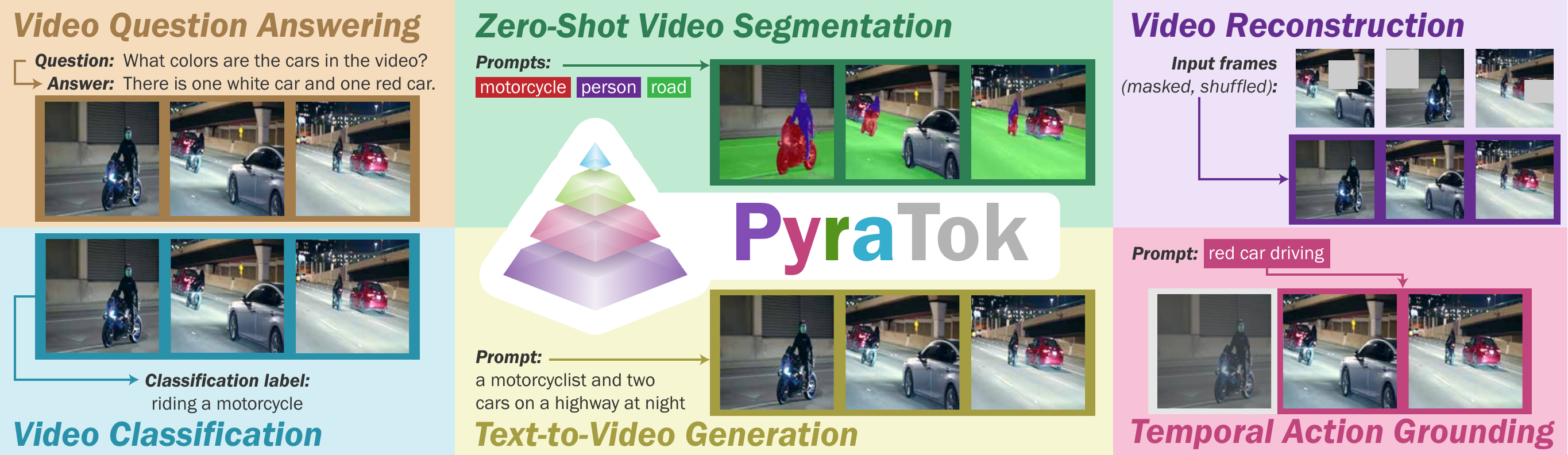}  
    \captionof{figure}{Given a video and text prompt, \modelname{} encodes compact latents, facilitating high-quality reconstruction and a wide range of video-language understanding tasks.}
    \label{fig:abstract_diagram}
\end{center}%
}

\affil{
\AffUIUC University of Illinois Urbana-Champaign \quad
\AffInd Independent Researcher \quad
\AffUT UTAustin \quad
\AffGoogle Google}

\begin{document}
\input{assets/tex/figures}
\input{assets/tex/tables}
\setabstractlogo[9mm]{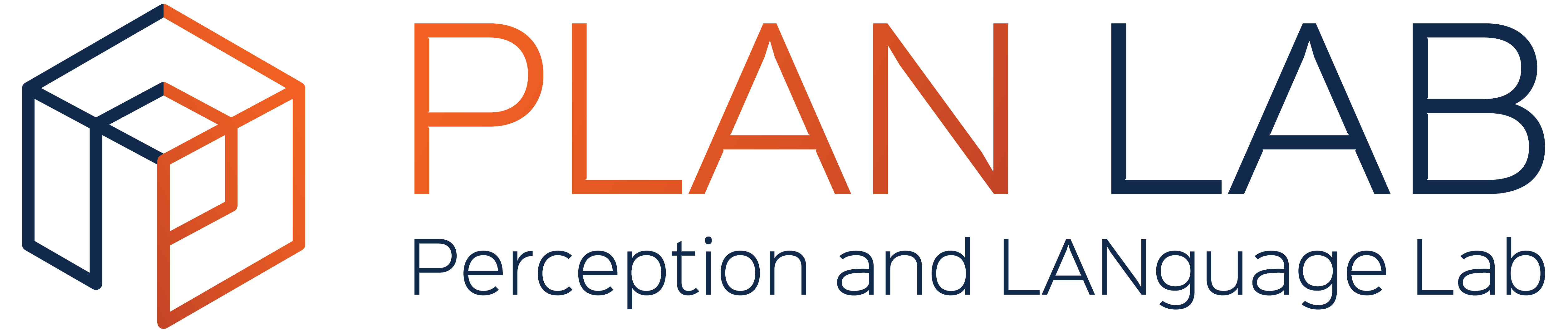} 

\begin{abstract} Discrete video VAEs underpin modern text-to-video generation and video understanding systems, yet existing tokenizers typically learn visual codebooks at a single scale with limited vocabularies and shallow language supervision, leading to poor cross-modal alignment and zero-shot transfer. We introduce \modelname, a language-aligned pyramidal tokenizer that learns semantically structured discrete latents across multiple spatiotemporal resolutions. \modelnamenc builds on a pretrained video VAE and a novel {Language aligned Pyramidal Quantization (LaPQ)} module that discretizes encoder features at several depths using a shared large binary codebook, yielding compact yet expressive video token sequences. To tightly couple visual tokens with language, \modelnamenc jointly optimizes multi-scale text-guided quantization and a global autoregressive objective over the token hierarchy. Across ten benchmarks, \modelnamenc delivers state-of-the-art (SOTA) video reconstruction, consistently improves text-to-video quality, and sets new SOTA zero-shot performance on video segmentation, temporal action localization, and video understanding, scaling robustly to up to 4K/8K resolutions.

\vspace{2mm}
\url{https://plan-lab.github.io/pyratok}
\vspace{-2mm}
\end{abstract}

\maketitle

\input{sections/01_introduction}

\input{sections/02_related_work}
\input{sections/03_method}
\input{sections/04_experiments}

\input{sections/05_conclusion}

\bibliographystyle{plainnat}
\bibliography{main}

\newpage
\appendix
\input{sections/06_appendix}

\end{document}

%% file: macro.tex
\usepackage[T1]{fontenc}
\usepackage{microtype}           
\usepackage{nicefrac}            
\usepackage{xspace}              
\usepackage{amsmath, amssymb, amsfonts, amsthm, mathtools, mathrsfs}
\usepackage{dsfont}              
\usepackage{fdsymbol}            

\usepackage{booktabs}            
\usepackage{nicematrix}         
\usepackage{multirow}
\usepackage{makecell}
\usepackage{bigstrut}
\usepackage{enumitem}            
\setitemize{label=\textbullet, leftmargin=*, nolistsep}
\usepackage{graphicx}
\usepackage{adjustbox}           
\usepackage{float}               
\usepackage{wrapfig}             
\usepackage{subcaption}          
\expandafter\def\csname ver@subfig.sty\endcsname{}  
\usepackage[dvipsnames]{xcolor}    
\usepackage{color}
\usepackage{fontawesome5}        
\usepackage{pifont}              
\usepackage{hyperref}            
\hypersetup{colorlinks=true, citecolor=LightBlue}
\usepackage[all]{hypcap}         
\usepackage{cleveref}            
\usepackage{url}                 
\usepackage{algorithm}
\usepackage{algorithmicx}
\usepackage{algpseudocode}
\usepackage[normalem]{ulem}      
\usepackage{lipsum}              
\usepackage{rotating}            
\usepackage{stackengine}         
\usepackage{caption}             
\usepackage{listings}            
\usepackage{soul}                
\usepackage{gradient-text}       
\usepackage{pgfplots}            
\pgfplotsset{compat=newest}
\usepackage{pgfplotstable}       
\usepackage{tikz}                
\usepackage{svg}                 

\usepackage[comma,numbers,sort,compress]{natbib}
\definecolor{demphcolor}{RGB}{125,125,125}    

\hypersetup{
    colorlinks = true,
    citecolor = {YaleBlue},
}
\tcbuselibrary{breakable, skins}
\newtcolorbox{planbox}[1]{
  enhanced,
  breakable,
  colback=white,
  colframe=IllinoisOrange!80,
  coltitle=IllinoisBlue,
  fonttitle=\bfseries\sffamily,
  title=#1,
  titlerule=0.8pt,
  boxrule=1pt,
  left=3mm, right=3mm, top=2mm, bottom=2mm,
  boxsep=1mm,
  before upper=\smallskip,
}

\newcommand{\ie}{\textit{i.e.},\xspace}      
\newcommand{\eg}{\textit{e.g.},\xspace}      
\newcommand{\etc}{\textit{etc}.\xspace}      

\crefname{equation}{Eq.}{Eqs.}
\crefformat{section}{\S#2#1#3}
\crefformat{subsection}{\S#2#1#3}
\crefformat{subsubsection}{\S#2#1#3}
\crefrangeformat{section}{\S\S#3#1#4 to~#5#2#6}
\crefmultiformat{section}{\S\S#2#1#3}{ and~#2#1#3}{, #2#1#3}{ and~#2#1#3}

\newcommand{\cmark}{\textcolor{ForestGreen}{\ding{51}}}  
\newcommand{\xmark}{\textcolor{red}{\ding{55}}}     


%% file: assets/tex/figures.tex
\newcommand{\FigAbstract}{
\twocolumn[{
\maketitle
\begin{center}
    \captionsetup{type=figure} \includegraphics[width=\linewidth]{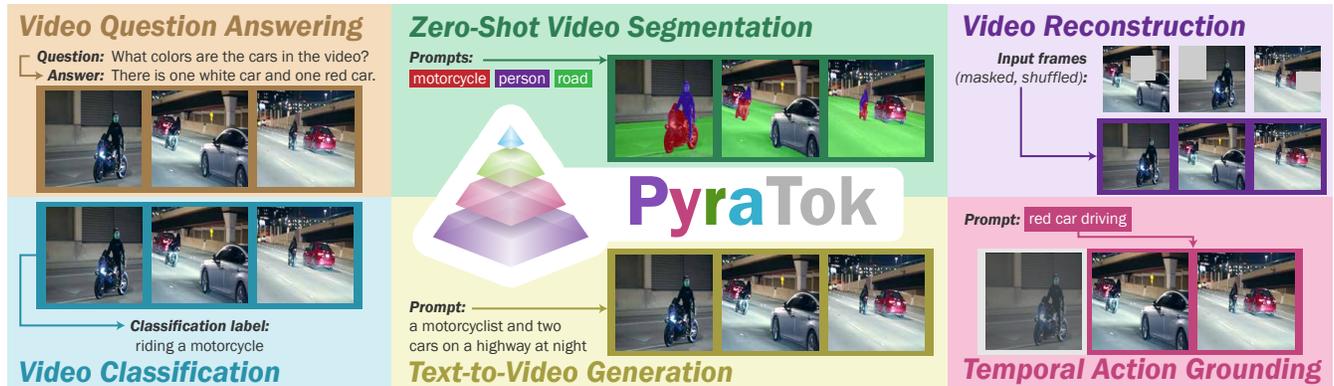}  
    \captionof{figure}{Overview of \modelnamenc. Given a video and text prompt, our multi-scale semantically aligned quantizer-based VAE encodes compact latents, facilitating high-quality reconstruction and a wide range of video-language understanding tasks.}\label{fig:abstract_diagram}
\end{center}
}]}

\newcommand{\FigIntro}{
\begin{figure}[t!]
\centering
\includegraphics[width=0.99\linewidth]{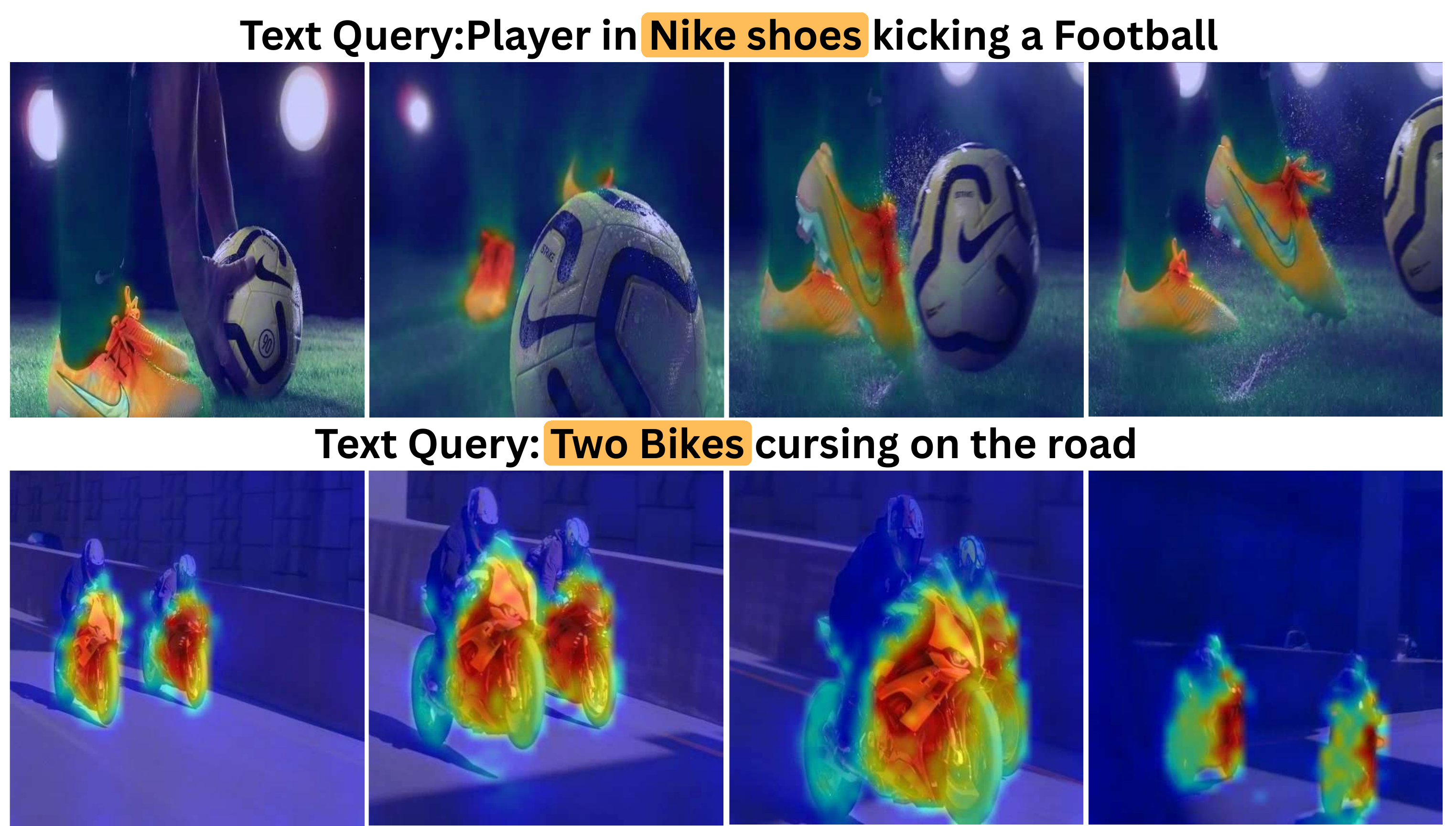}
    \caption{\textbf{\modelnamenc attention maps illustrating fine-grained cross-modal alignment.} Highlighted regions indicate language-guided semantic localization (\eg Nike shoes, bikes).\looseness-1}
    \label{fig:activation_map}
\end{figure}
\vspace{-0.3cm}
}

\newcommand{\FigArch}{
\begin{figure*}[t!]
    \centering    \includegraphics[width=0.99\linewidth]{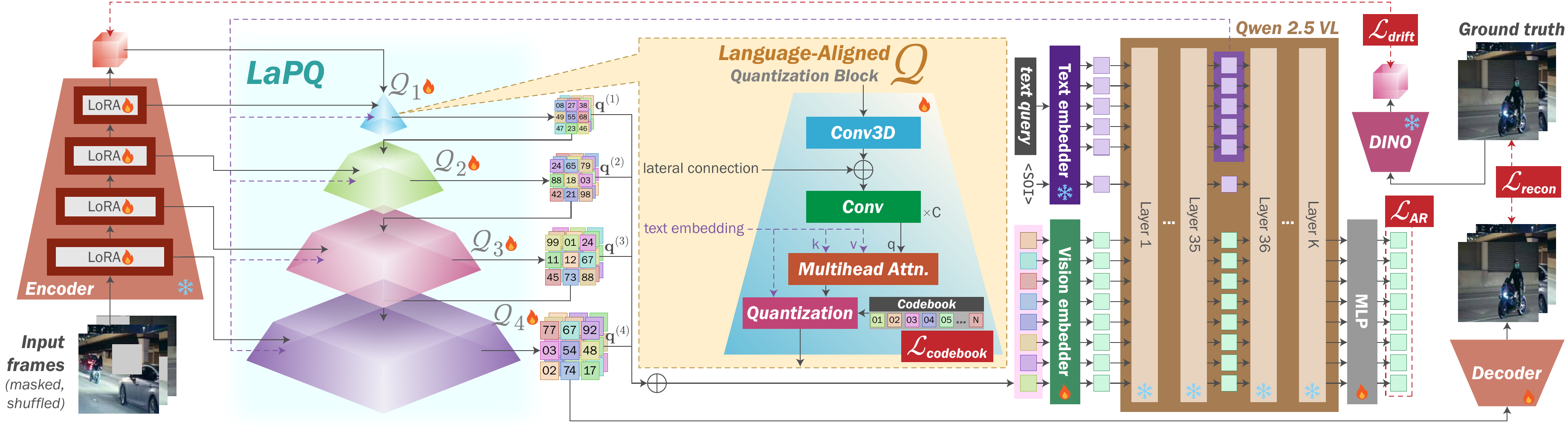}
    \caption{\textbf{Overview of the proposed \modelname{} architecture.} Masked video frames are encoded and quantized at multiple scales via Language-aligned Pyramidal Quantization (LaPQ) blocks guided by text embeddings. The resulting multi-scale discrete tokens are aligned through a vision-language model for semantic consistency, enabling high-fidelity and text-aware video reconstruction.}
    \label{fig:PVQ-VAE-arch}
    \vspace{-0.2cm}
\end{figure*}} 

\newcommand{\FigPCA}{
\begin{figure}[t!]
\includegraphics[width=0.99\linewidth]{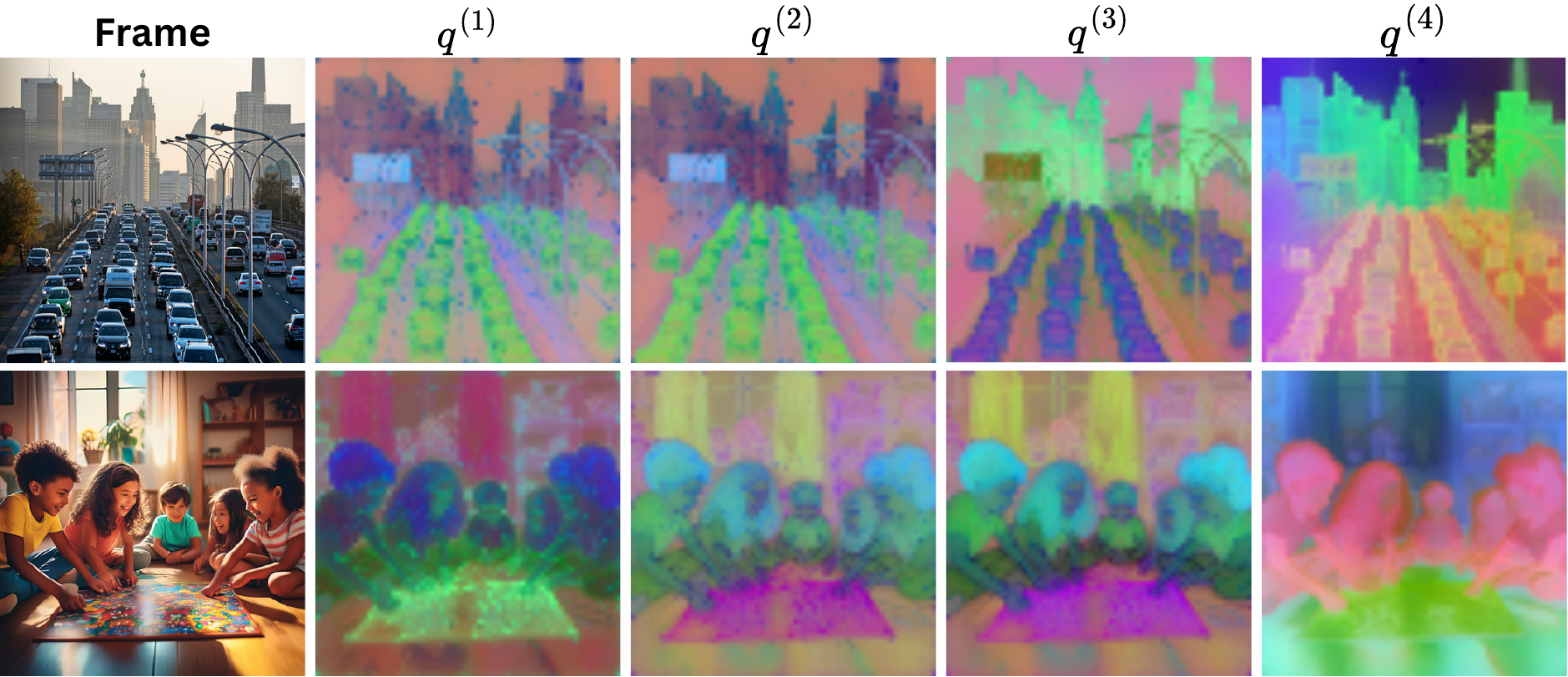}
    \caption{\textbf{PCA projections of quantized tokens from each \modelnamecpa's stage.} Columns ($q^{(1)}$–$q^{(4)}$) show hierarchical outputs capturing progressively refined and semantically aligned regions.}
    \label{fig:pca}
    \vspace{-0.3cm}
\end{figure}} 

\newcommand{\FigVideoRecon}{
 \begin{figure}[htbp]
    \includegraphics[width = \linewidth]{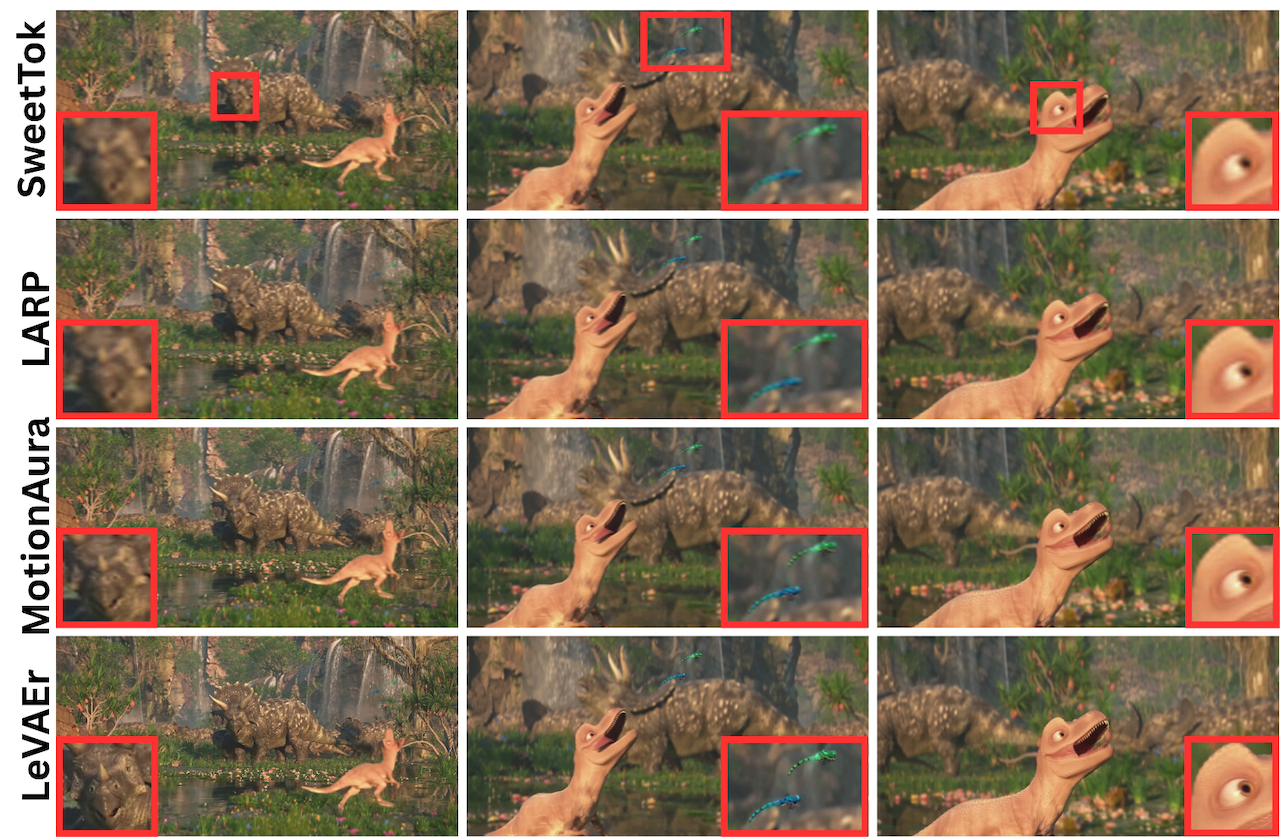}
    \caption{\textbf{Video frame reconstruction comparison on a dynamic scene.}}
    \label{fig:frames_recon}
\end{figure}}

\newcommand{\FigFrameRecon}{
\begin{figure*}
    \centering
    \includegraphics[width=0.9\linewidth]{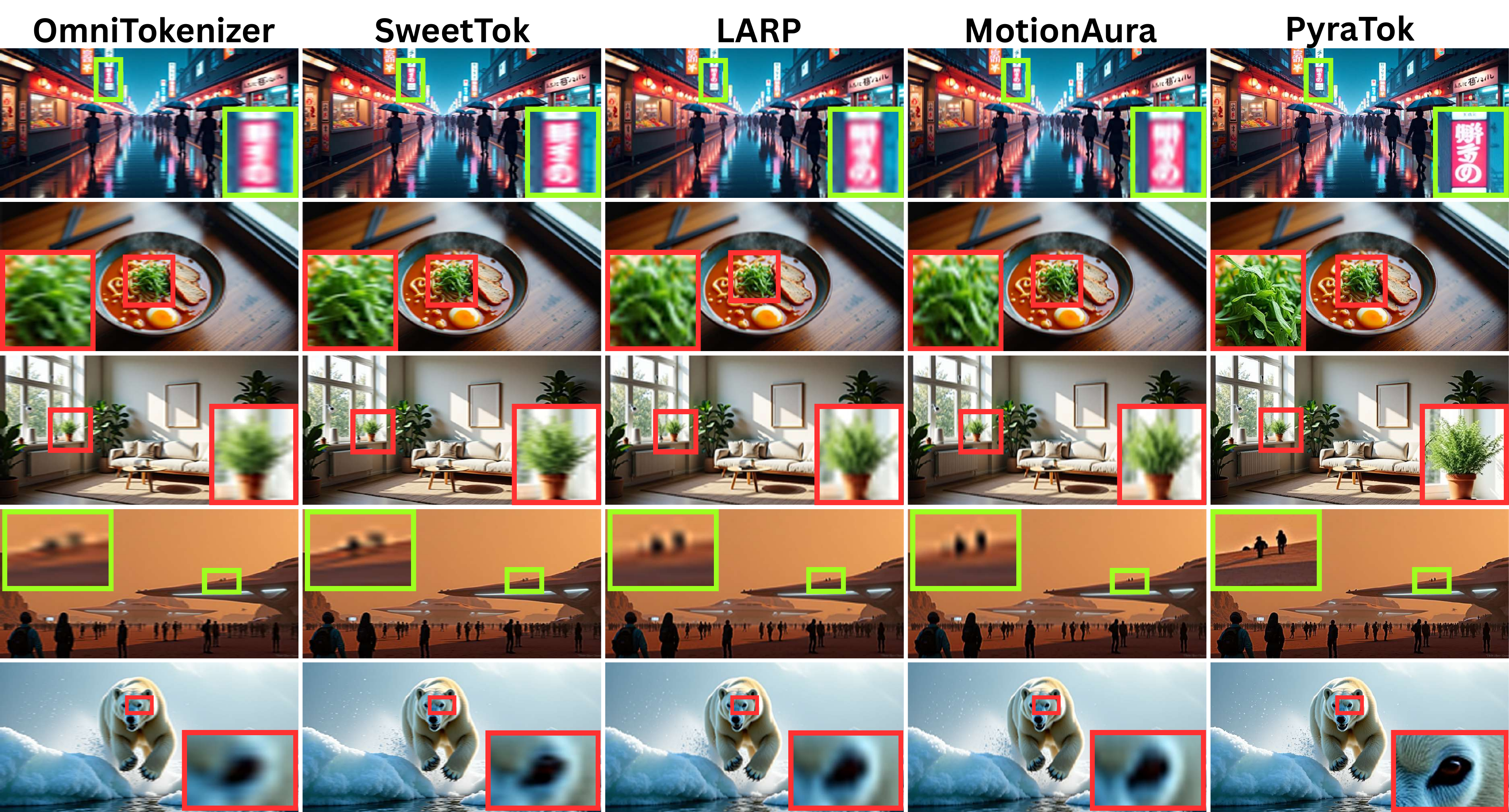}
    \caption{\textbf{Frame reconstruction qualitative comparison.} \modelnamenc generates sharper details, clearer textures, and better spatial structure than baselines, demonstrating better fidelity and semantic consistency.}
    \label{fig:video_reconstruction}
    \vspace{-0.4cm}
\end{figure*}}

\newcommand{\FigTSNE}{
\begin{figure}[t!]
\centering
\includegraphics[width=0.9\linewidth]{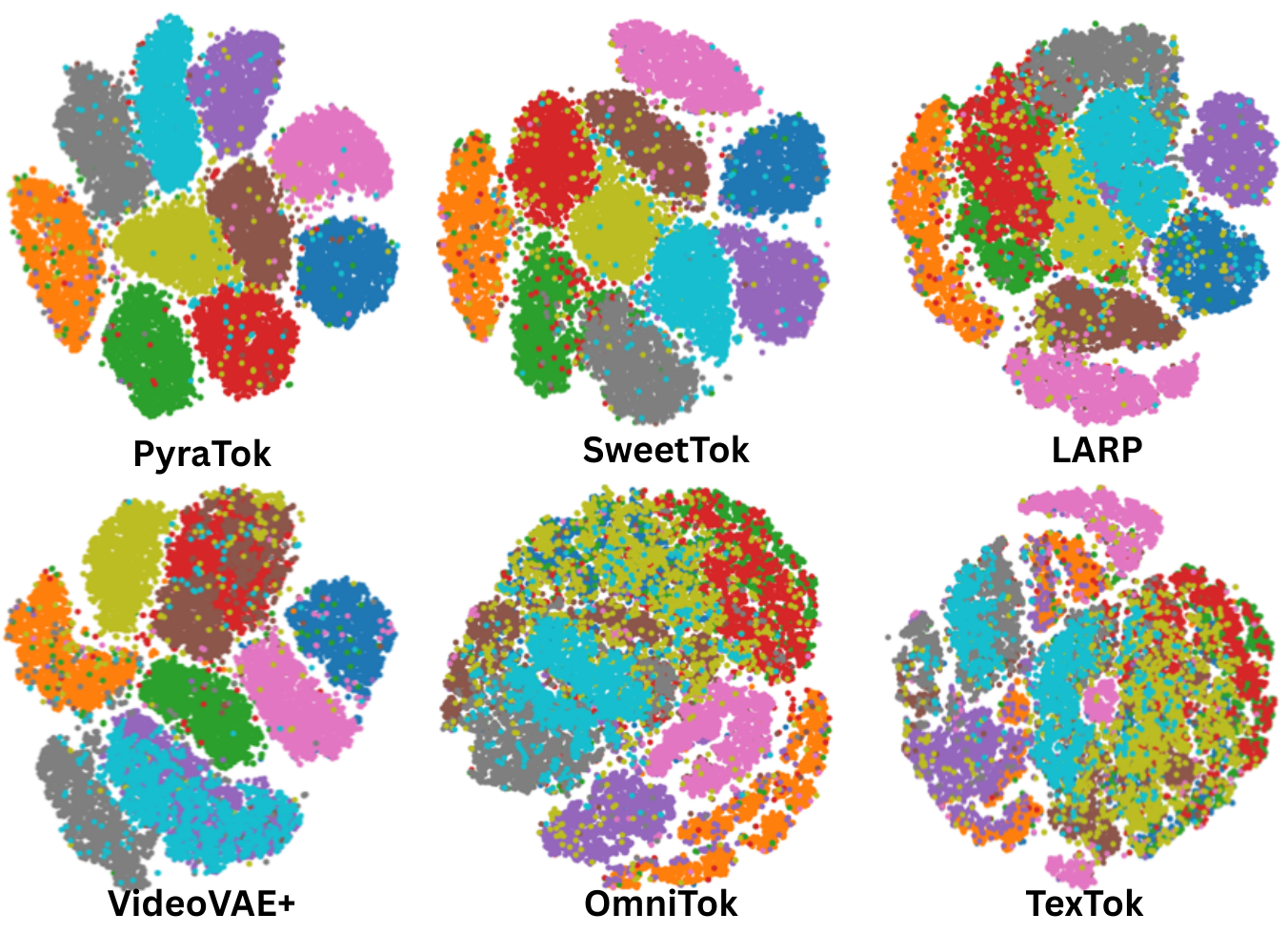}
\caption{\textbf{t-SNE visualization} showing \modelnamenc with more distinct, well-separated clusters, indicating improved semantic organization.
}
\label{fig:t_sne_plots}
\end{figure}}

\newcommand{\FigTV}{
\begin{figure}[t!]
    \centering  \includegraphics[width=0.95\linewidth]{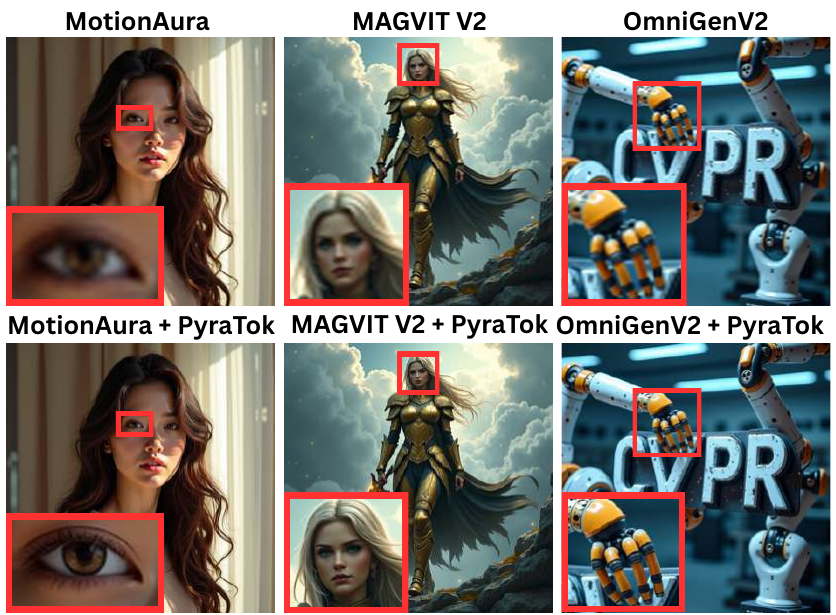}
    \caption{\textbf{T2V generation across various backbones.} Integrating \modelnamenc enhances detail, sharpness, and spatial consistency.}
    \label{fig:video_generation}
    \vspace{-0.2cm}
\end{figure}}

\newcommand{\FigAction}{
\begin{figure}[t!]
    \centering 
    \includegraphics[width=0.99\linewidth]{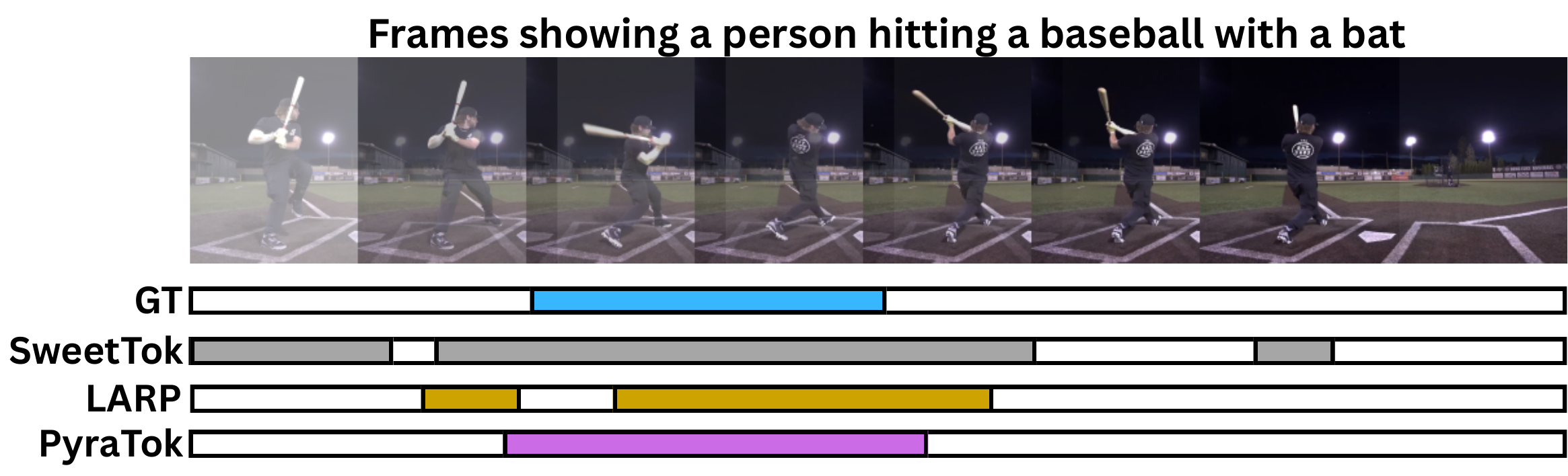}
    \caption{\textbf{Video action localization results.} \modelnamenc aligns action boundaries more accurately.}
    \label{fig:video_action} 
\end{figure}}

\newcommand{\FigSeg}{
\begin{figure}[t!]
    \centering   
    \includegraphics[width=0.99\linewidth, height=5cm]{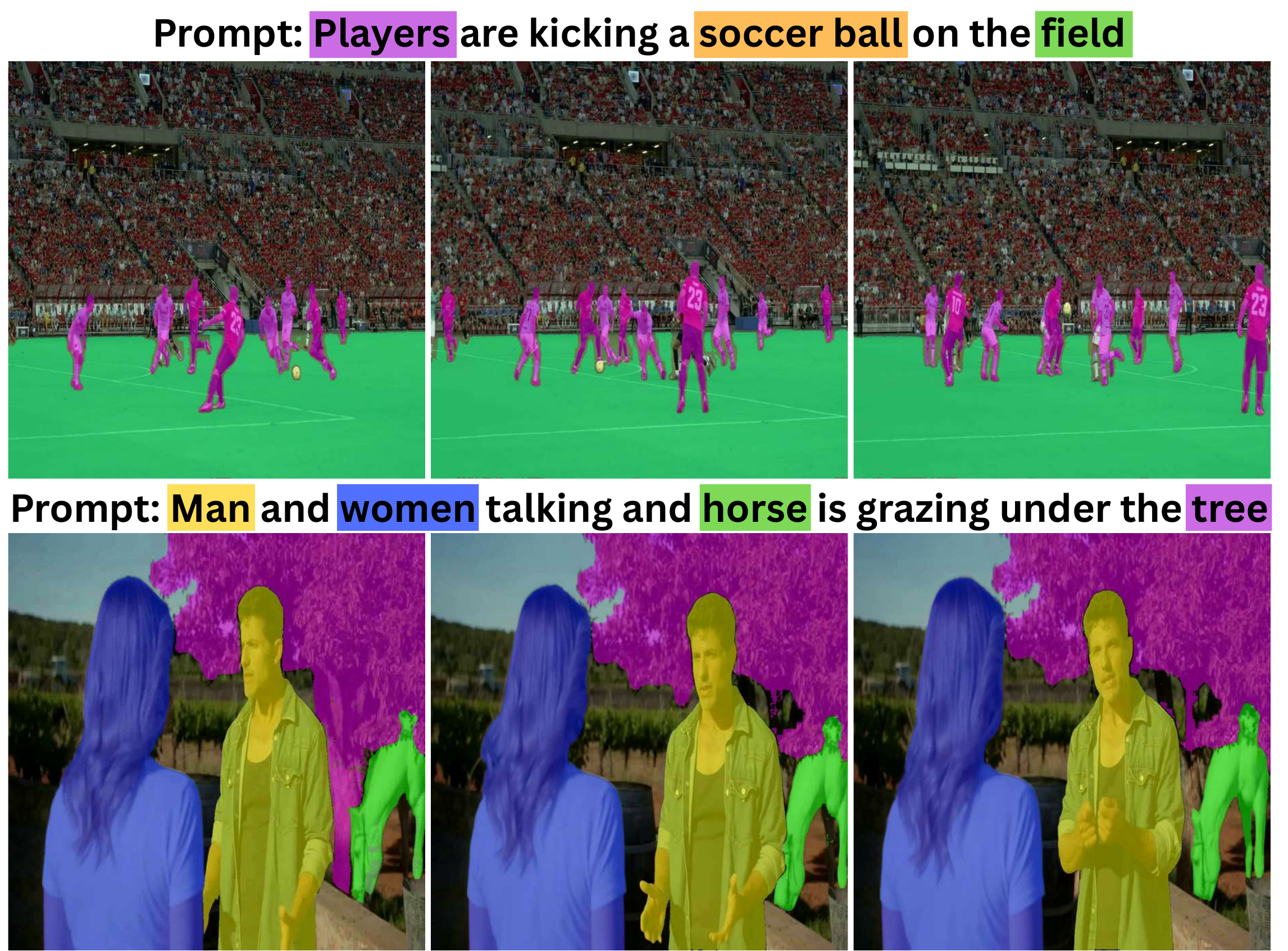}
    \caption{\textbf{Zero-shot segmentation results} showing \modelnamenc{}'s precise text-guided segmentation of multiple objects in complex scenes.}
    \label{fig:video_seg}
    \vspace{-0.2cm}
\end{figure}}

\newcommand{\FigCodebook}{
\begin{figure}[t!]
\includegraphics[width=\linewidth]{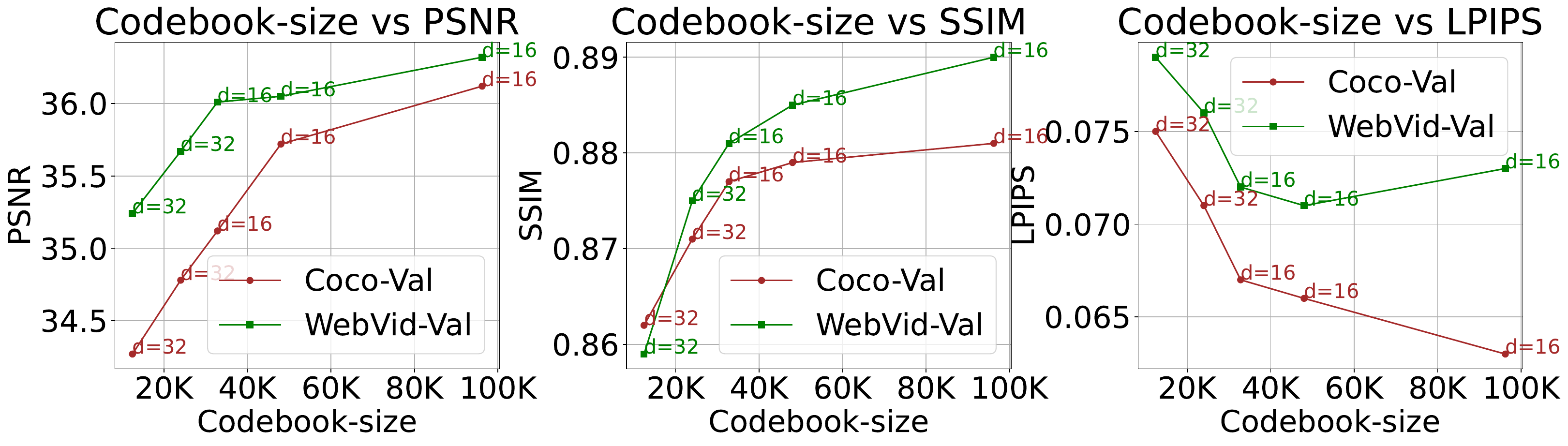}
    \caption{\textbf{Effect of codebook size on reconstruction quality.\looseness-1}}
    \label{fig:codebook_vs_loss} 
    \vspace{-0.3cm}
\end{figure}}

\newcommand{\FigSupFramebig}{
\begin{figure*}[t!]
    \centering
    \includegraphics[width=0.9\linewidth]{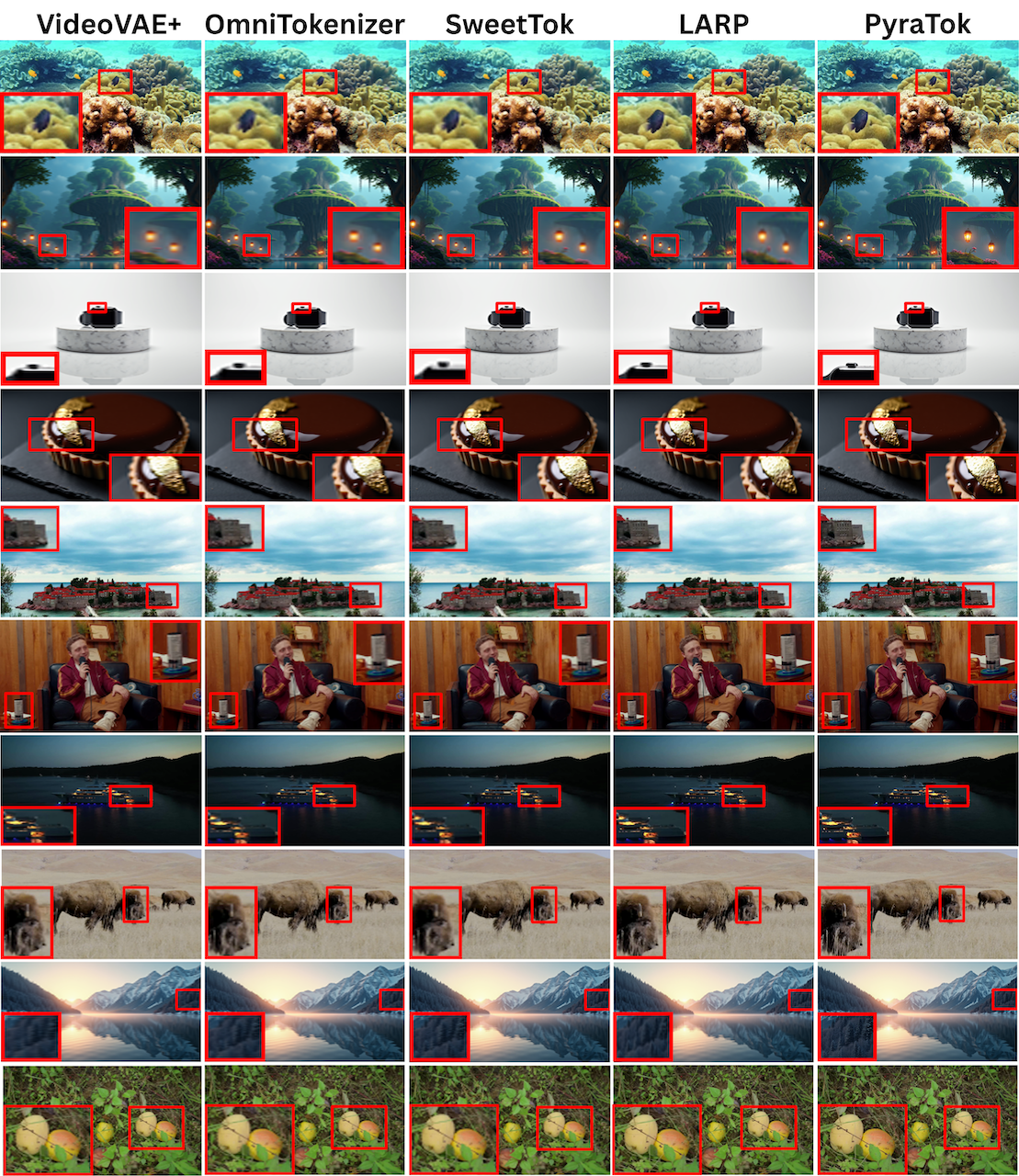}
    \caption{\textbf{Qualitative comparison of single-frame reconstruction across diverse scenes}, including underwater environments, fantasy landscapes, product renders, food close-ups, urban views, interview settings, night scenes, wildlife, mountain vistas, and natural textures. Each row shows outputs from one method for the same input frame, with red boxes highlighting fine details, such as small objects, textures, reflections, and thin structures, used to compare reconstruction fidelity, sharpness, and color consistency. \modelnamenc{} preserves fine details reliably and delivers consistent, high-quality reconstructions across all scene types. Discussion in \ref{app:recon}.}
    \label{fig:frame_reconcs_sup}
\end{figure*}}

\newcommand{\FigSupvideoone}{
\begin{figure*}[t!]
    \centering
    \includegraphics[width=0.99\linewidth]{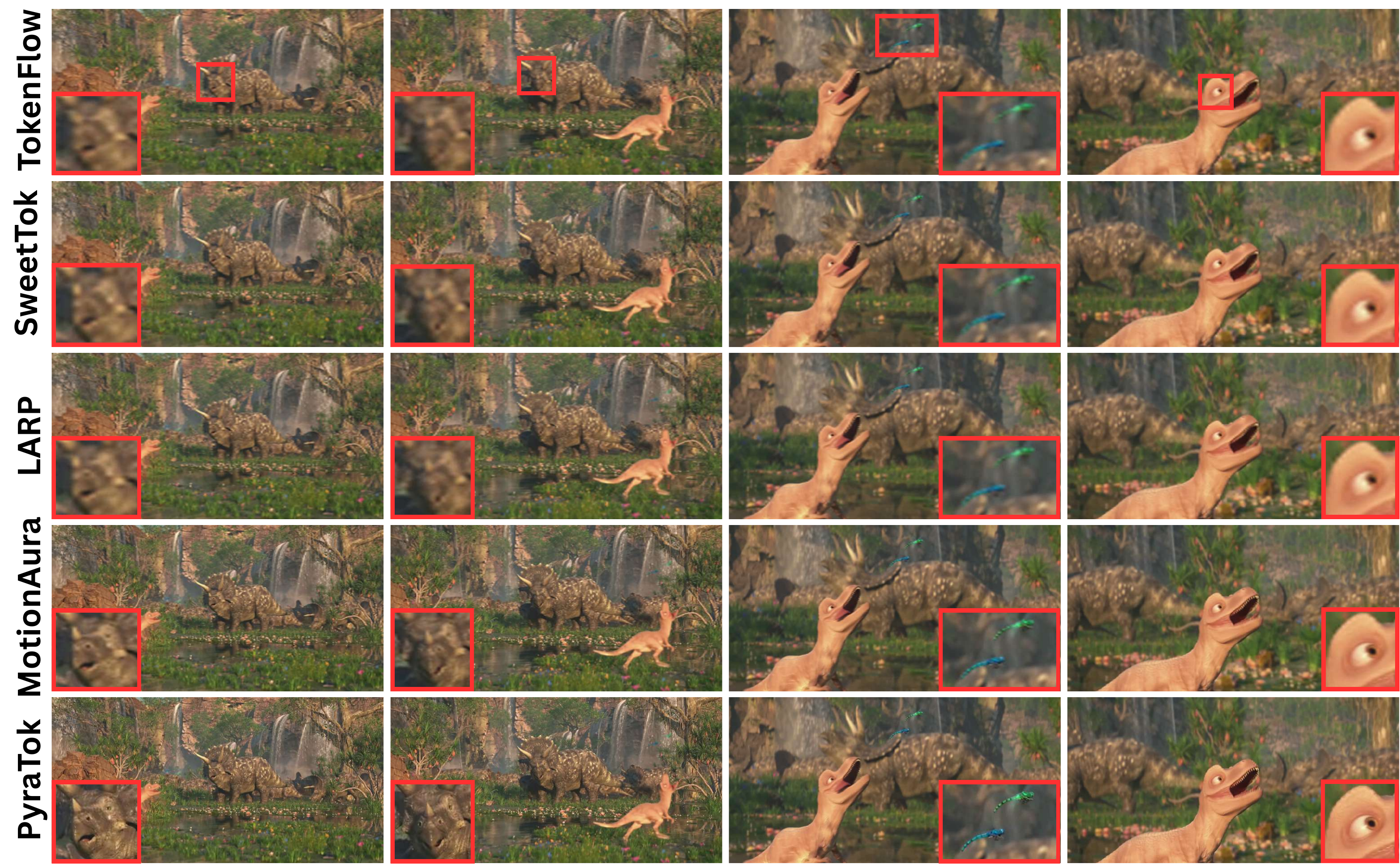}
    \caption{\textbf{Qualitative comparison of video reconstruction methods} on a fast-moving dinosaur sequence containing dense foliage, small background animals, and detailed facial motions. Each row shows outputs from one method on the same frames. Red boxes highlight challenging regions such as vegetation, moving creatures, and fine facial details, where differences in sharpness, temporal coherence, and motion fidelity are most apparent. Discussion in \ref{app:recon}.}
    \label{fig:video_recon_one}
\end{figure*}}

\newcommand{\FigSupvideotwo}{
\begin{figure*}[t!]
    \centering
    \includegraphics[width=0.99\linewidth]{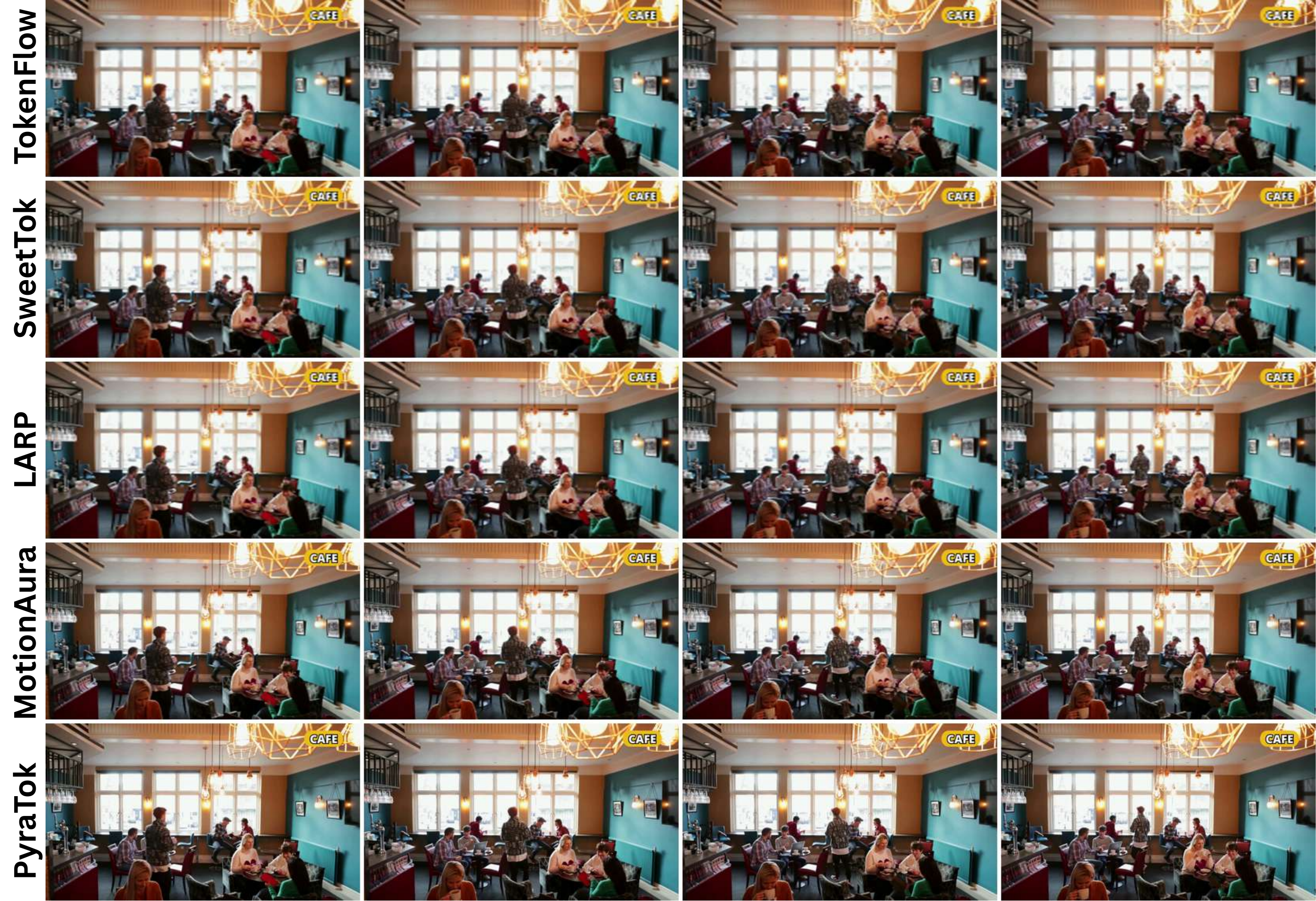}
    \caption{\textbf{Qualitative comparison of \modelnamenc{} with other video reconstruction methods} on a dynamic café scene containing multiple people, complex indoor lighting, and detailed textures. Each row shows outputs from one method on the same frames. The scene highlights challenges such as preserving facial details, clothing patterns, reflections, and background structures, allowing visual comparison of reconstruction sharpness and temporal consistency.  Details in \ref{app:recon}.}
    \label{fig:video_recon_two}
   
\end{figure*}}

\newcommand{\FigSupvideothree}{
\begin{figure*}[t!]
    \centering
    \includegraphics[width=0.99\linewidth]{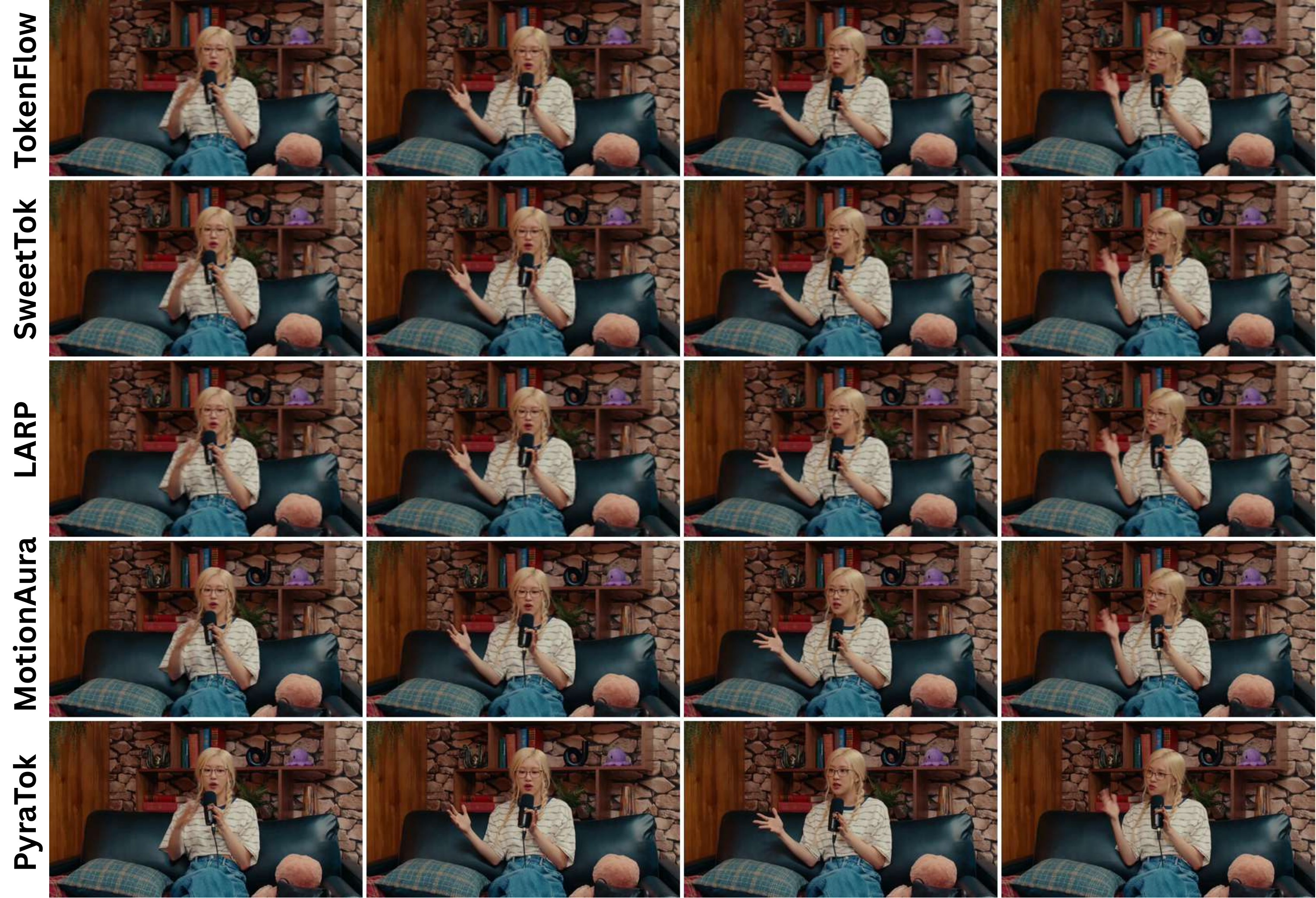}
    \caption{\textbf{Qualitative comparison of \modelnamenc{} with other video reconstruction methods} on an indoor interview scene featuring expressive hand motions, detailed facial appearance, and complex background textures, revealing differences in preserving facial clarity, hand motion coherence, and fine background details such as books, fabrics, and stone textures.  Details in \ref{app:recon}.}
    \label{fig:video_recon_three}
\end{figure*}}

\newcommand{\FigSupvideofour}{
\begin{figure*}[t!]
    \centering
    \includegraphics[width=0.99\linewidth]{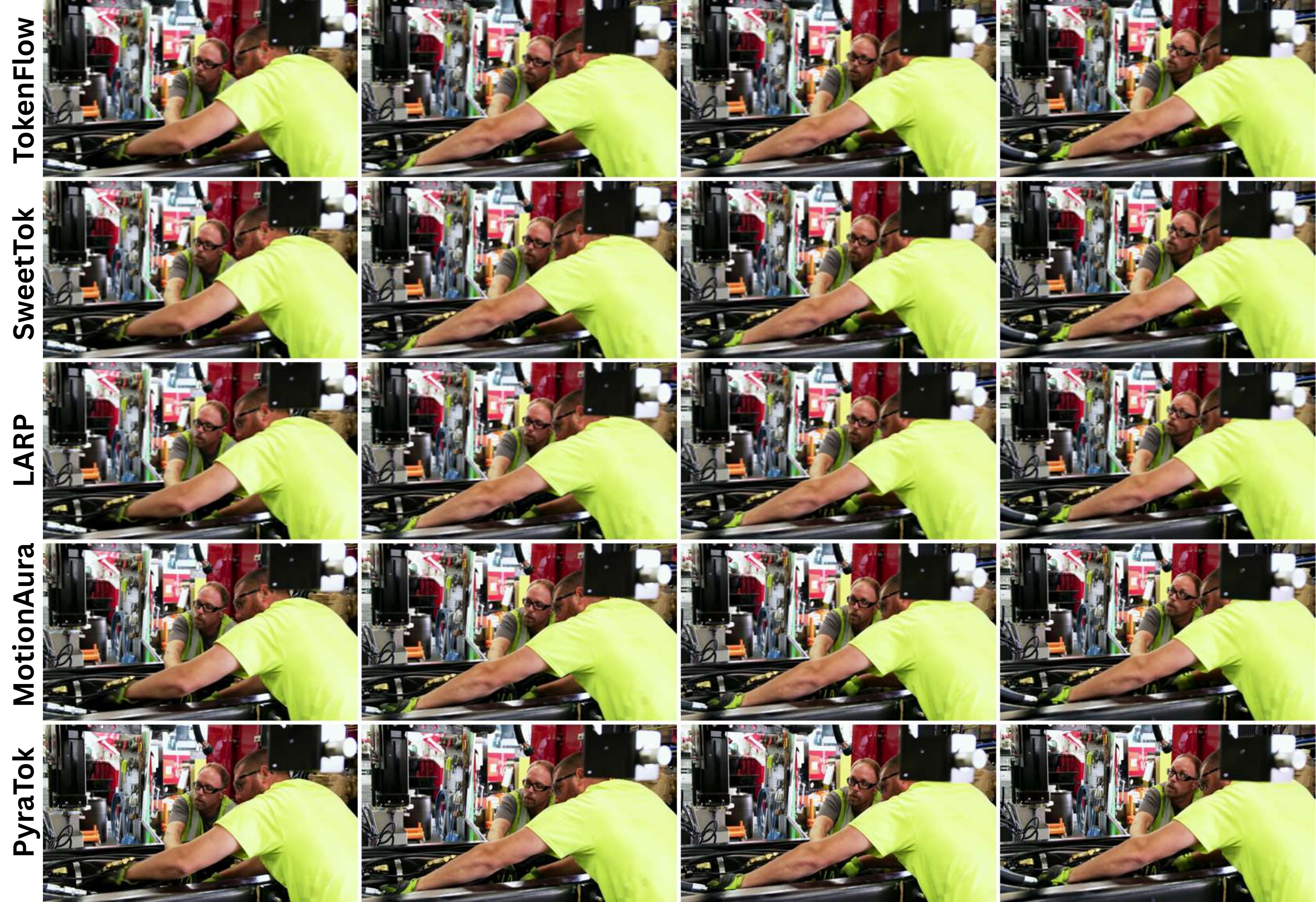}
    \caption{\textbf{Qualitative comparison of \modelnamenc{} with other video reconstruction methods} on a workshop scene involving fast arm movements, reflective machinery, and detailed background clutter, highlighting differences in preserving motion clarity, fine textures on tools and equipment, and the stability of subtle visual details under rapid motion. Details in \ref{app:recon}.}
    \label{fig:video_recon_four}
\end{figure*}}

\newcommand{\FigSupsegone}{
\begin{figure*}[t!]
    \centering
    \includegraphics[width=0.99\linewidth]{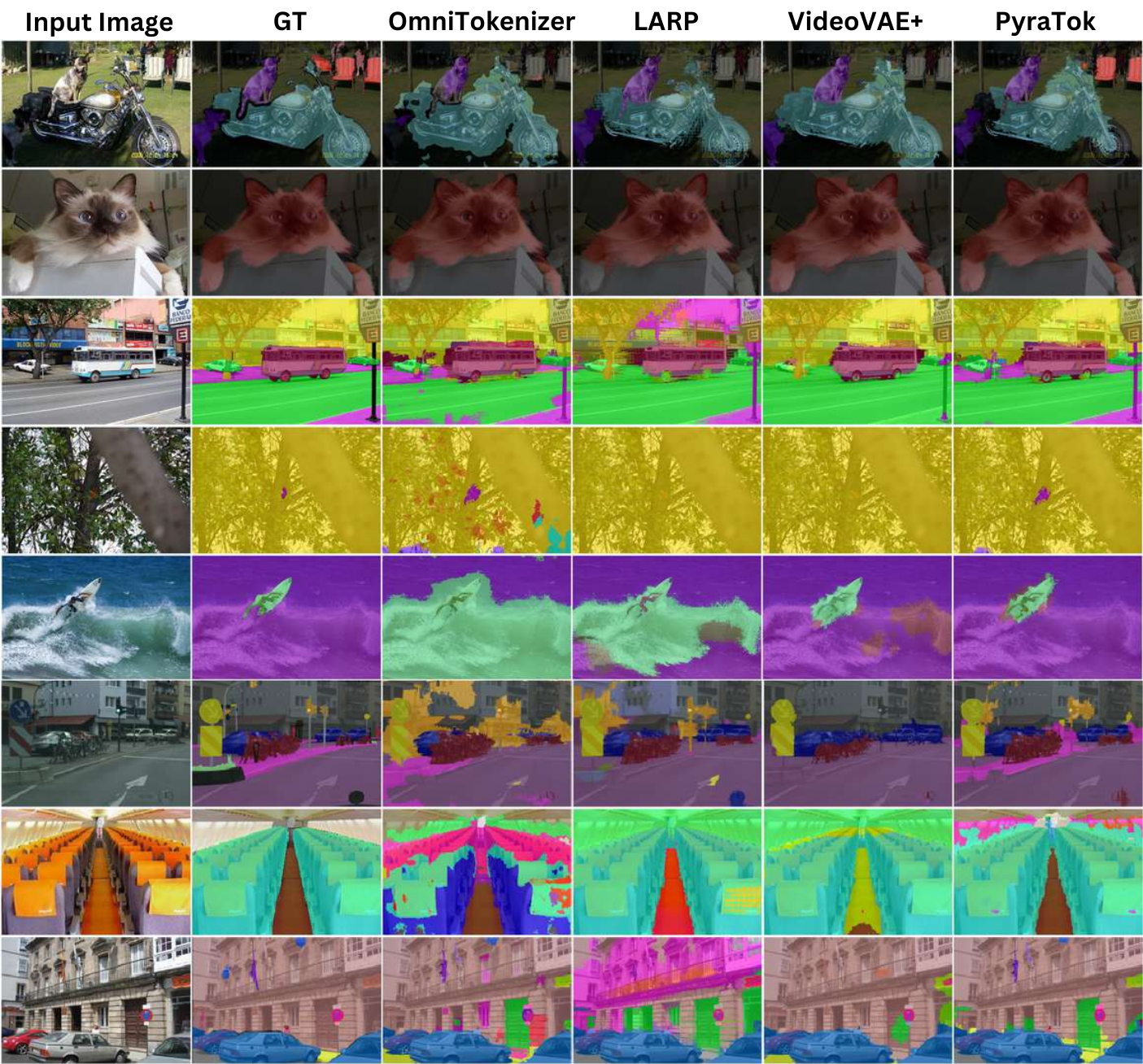}
    \caption{\textbf{Zero-shot semantic segmentation comparison across various scenes.} Results illustrate \modelnamenc{}’s ability to recover fine object boundaries, preserve small structures, and produce semantically coherent segmentations across diverse domains. Details in \ref{app:seg}.}
    \label{fig:seg_sup_one}
\end{figure*}}

\newcommand{\FigSupsegtwo}{
\begin{figure*}[t!]
    \centering
    \includegraphics[width=0.9\linewidth]{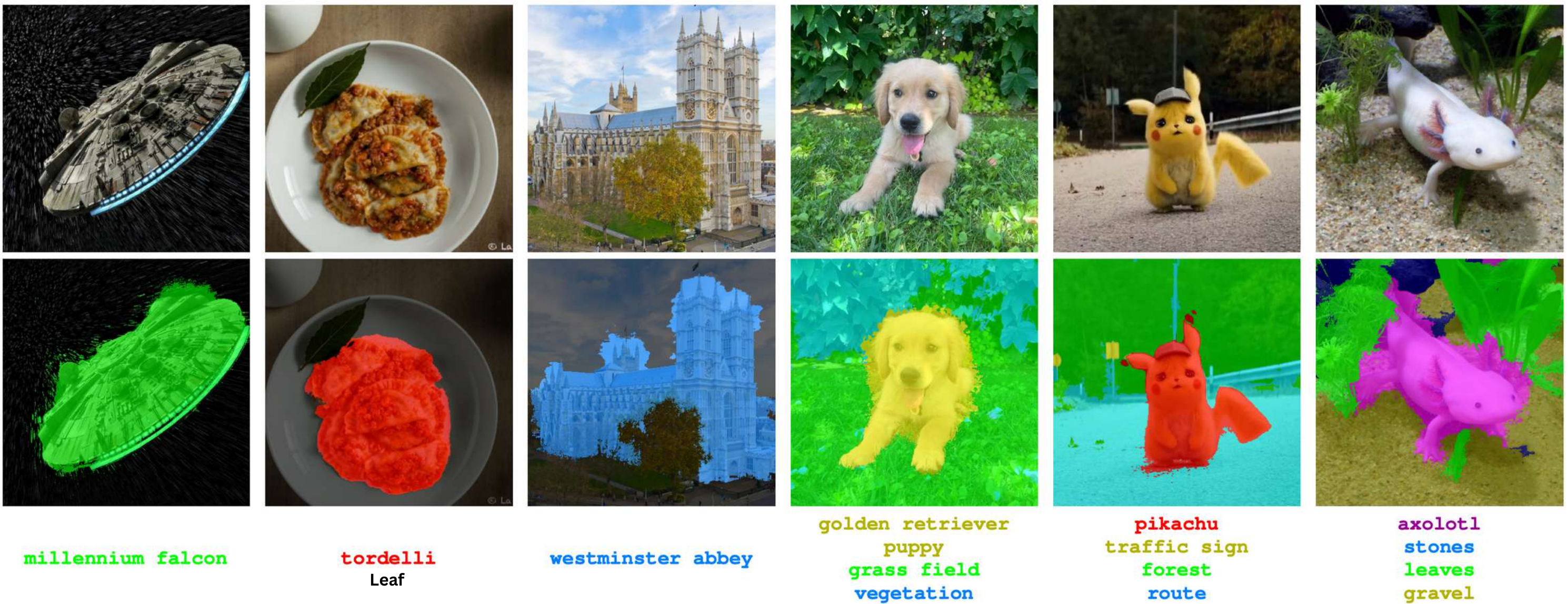}
    \caption{\textbf{Zero-shot semantic segmentation examples produced by \modelnamenc{} using only text prompts.} Each column shows an input image, the corresponding segmentation mask predicted by \modelnamenc{}, and the set of text labels used. Results span diverse object types, demonstrating \modelnamenc{}’s ability to segment both rare and common entities without task-specific training. Details in \ref{app:seg}.}
    \label{fig:seg_sup_two}
\end{figure*}}

\newcommand{\FigSupFrameTVone}{
\begin{figure*}[t!]
    \centering
    \includegraphics[width=0.72\linewidth]{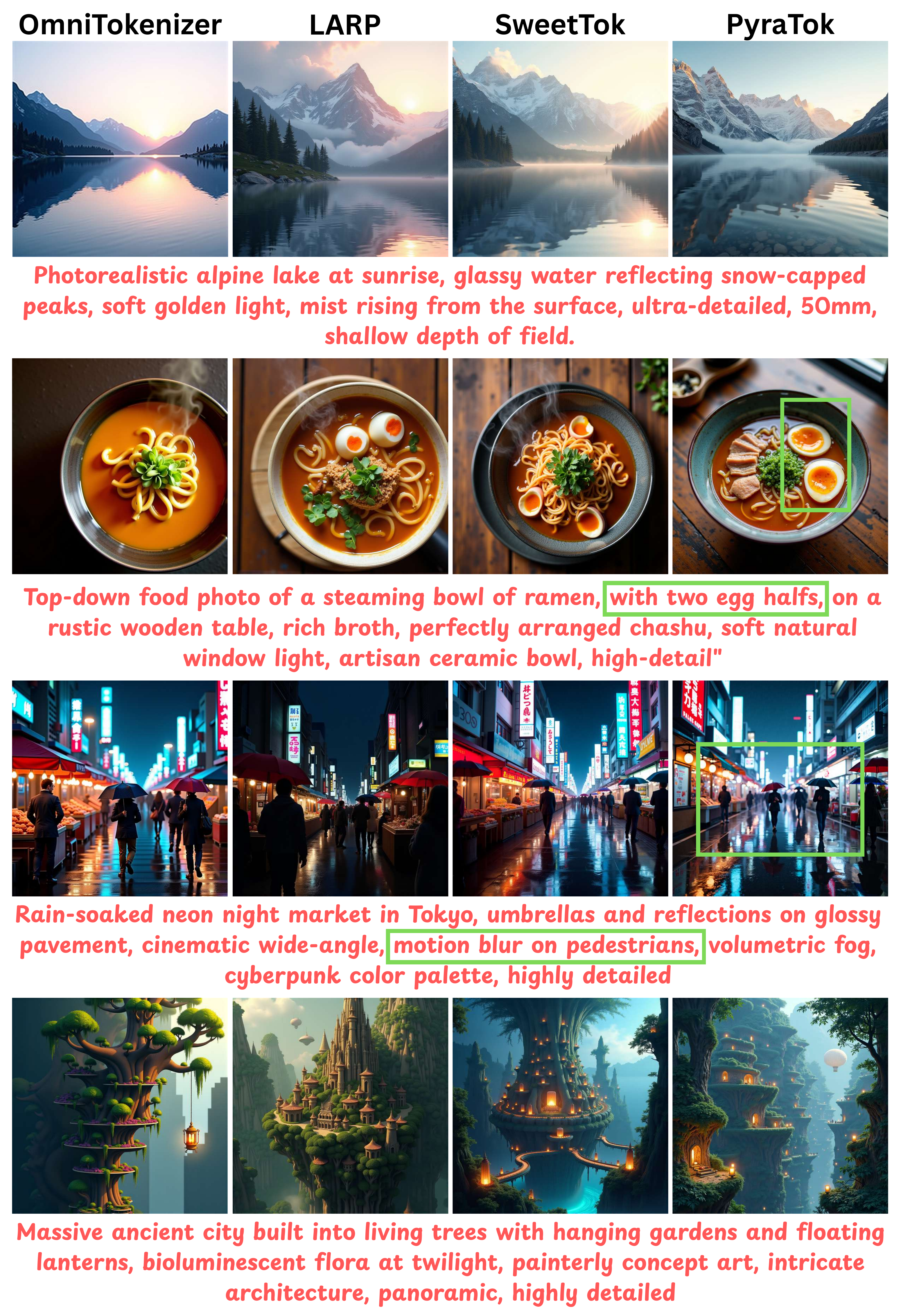}
    \caption{\textbf{Text-to-video qualitative comparison}, showing a representative frame from each generated clip across diverse prompt categories, including photorealistic landscapes, detailed food scenes, night-market environments, and stylized concept art. Although only a single frame per video is shown, the green boxes highlight fine-grained details faithfully produced by \modelnamenc{}, such as the correct depiction of “two half eggs” in the ramen scene and realistic motion blur on pedestrians in the night-market prompt, illustrating \modelnamenc{}’s ability to accurately interpret and render subtle textual attributes in T2V generation. Discussion in \ref{app:t2v}.}    
    \label{fig:frame_tv_one} 
\end{figure*}}

\newcommand{\FigSupFrameTVtwo}{
\begin{figure*}[t!]
    \centering
    \includegraphics[width=0.75\linewidth]{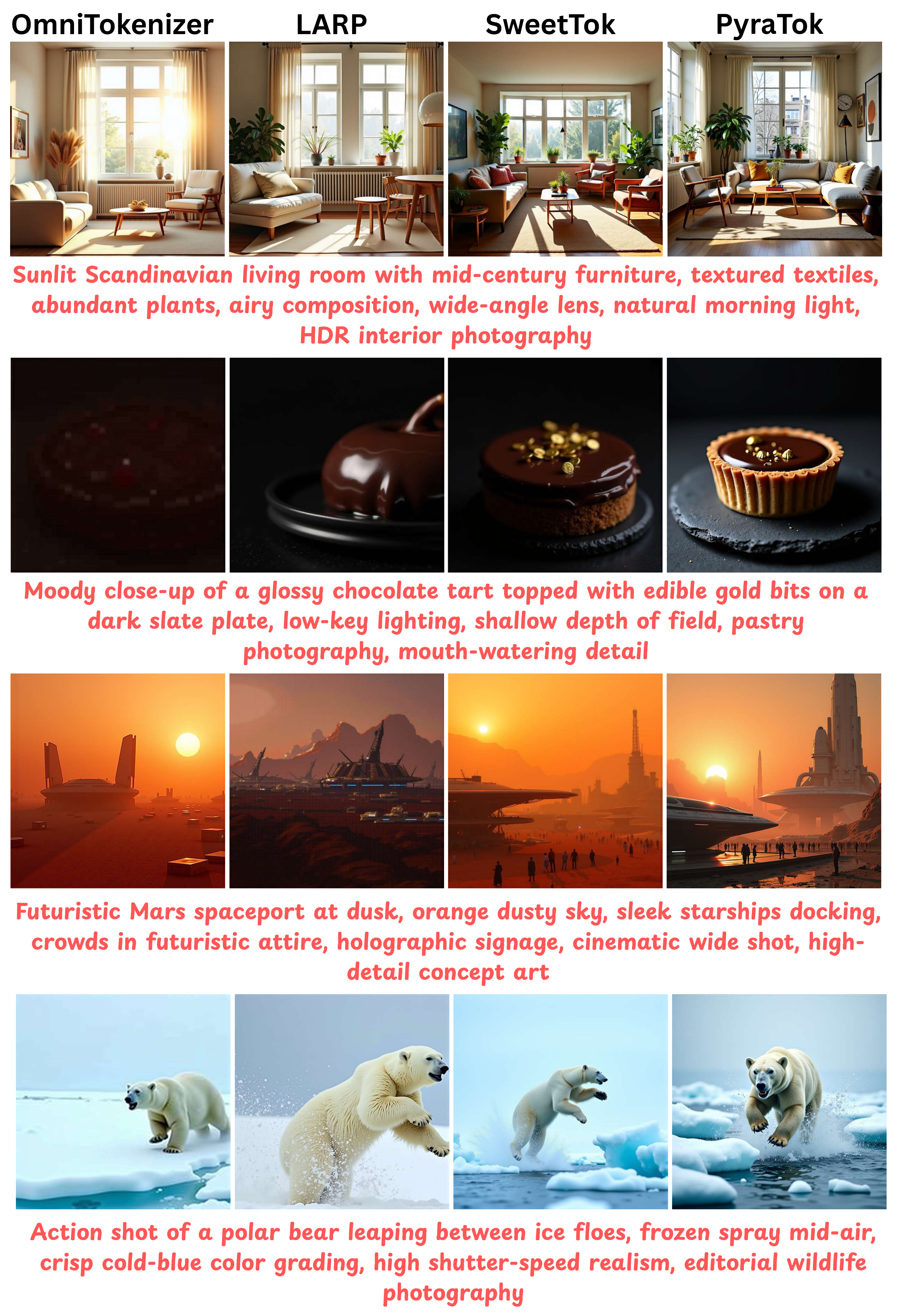}
    \caption{\textbf{Text-to-video generation comparisons}, showing a representative frame from each generated clip across a diverse set of prompts, including interior scenes, food close-ups, sci-fi concept art, and wildlife action. \modelnamenc{} consistently captures fine-grained details, accurate lighting, textures, object geometry, and scene composition, demonstrating strong prompt alignment and high-fidelity generation across varied visual domains. Discussion in \ref{app:t2v}.
    }
    \label{fig:frame_tv_two}
\end{figure*}}

\newcommand{\FigSupHDTV}{
\begin{figure*}[t!]
    \centering
    \includegraphics[width=0.99\linewidth]{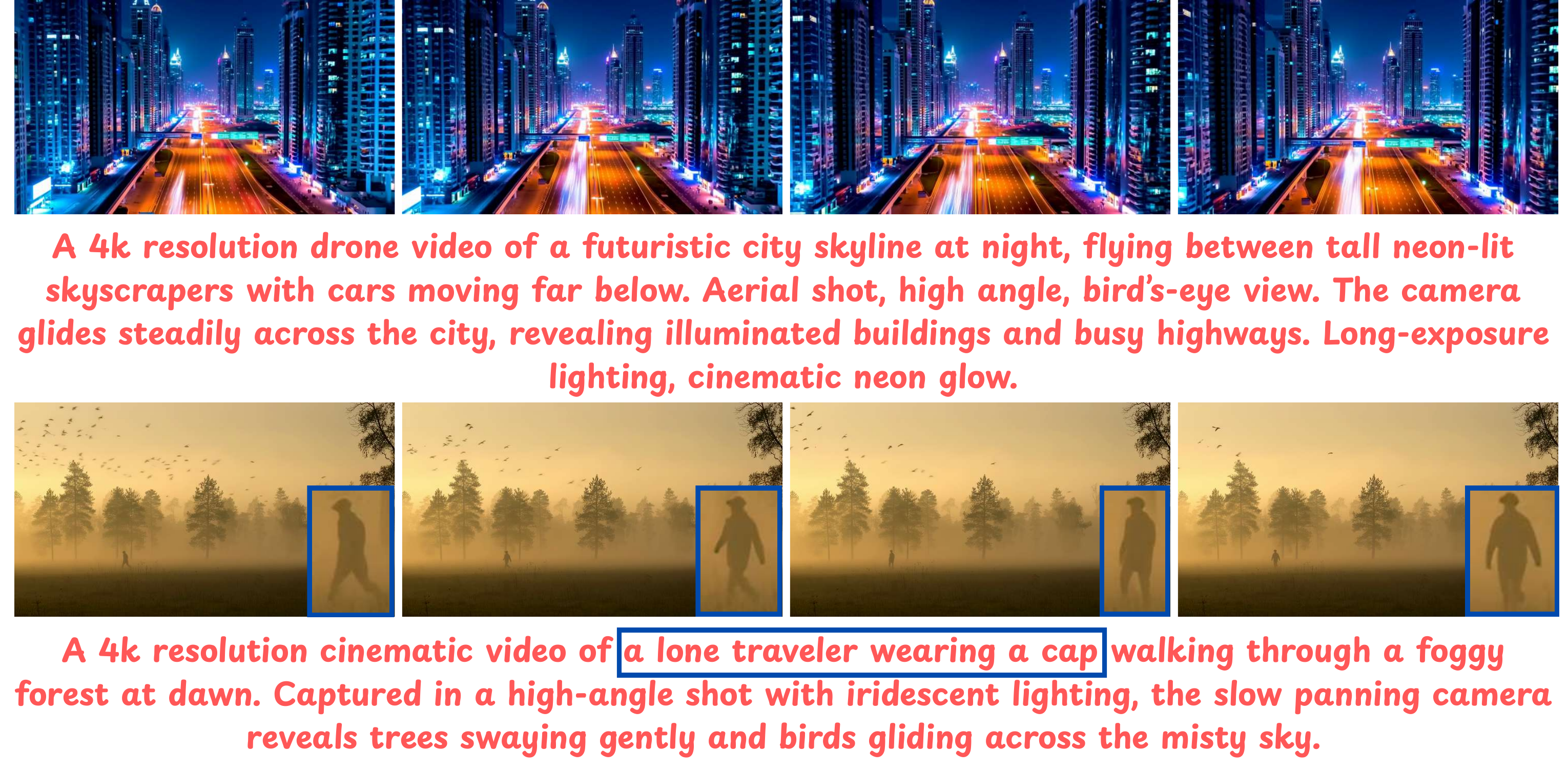}
    \caption{\textbf{Text-to-video generation results at 4K resolution using \modelnamenc{}, shown as representative frames from two distinct prompts.} The first example depicts a futuristic neon-lit city captured by a flying drone, where \modelnamenc{} maintains crisp details, stable long-exposure lighting, and smooth camera motion. The second example illustrates a foggy forest at dawn featuring “a lone traveler wearing a cap.” Even though the person occupies only a tiny fraction of the scene, \modelnamenc{} accurately renders fine-grained details, such as the cap on the traveler’s head, demonstrating strong text-alignment and high-resolution consistency in large-scale, wide-angle video generation. Discussion in \ref{app:t2v}.}
    \label{fig:HDTV}
\end{figure*}}

\newcommand{\FigSupmagvae}{
\begin{figure*}[t!]
    \centering
    \includegraphics[width=0.99\linewidth]{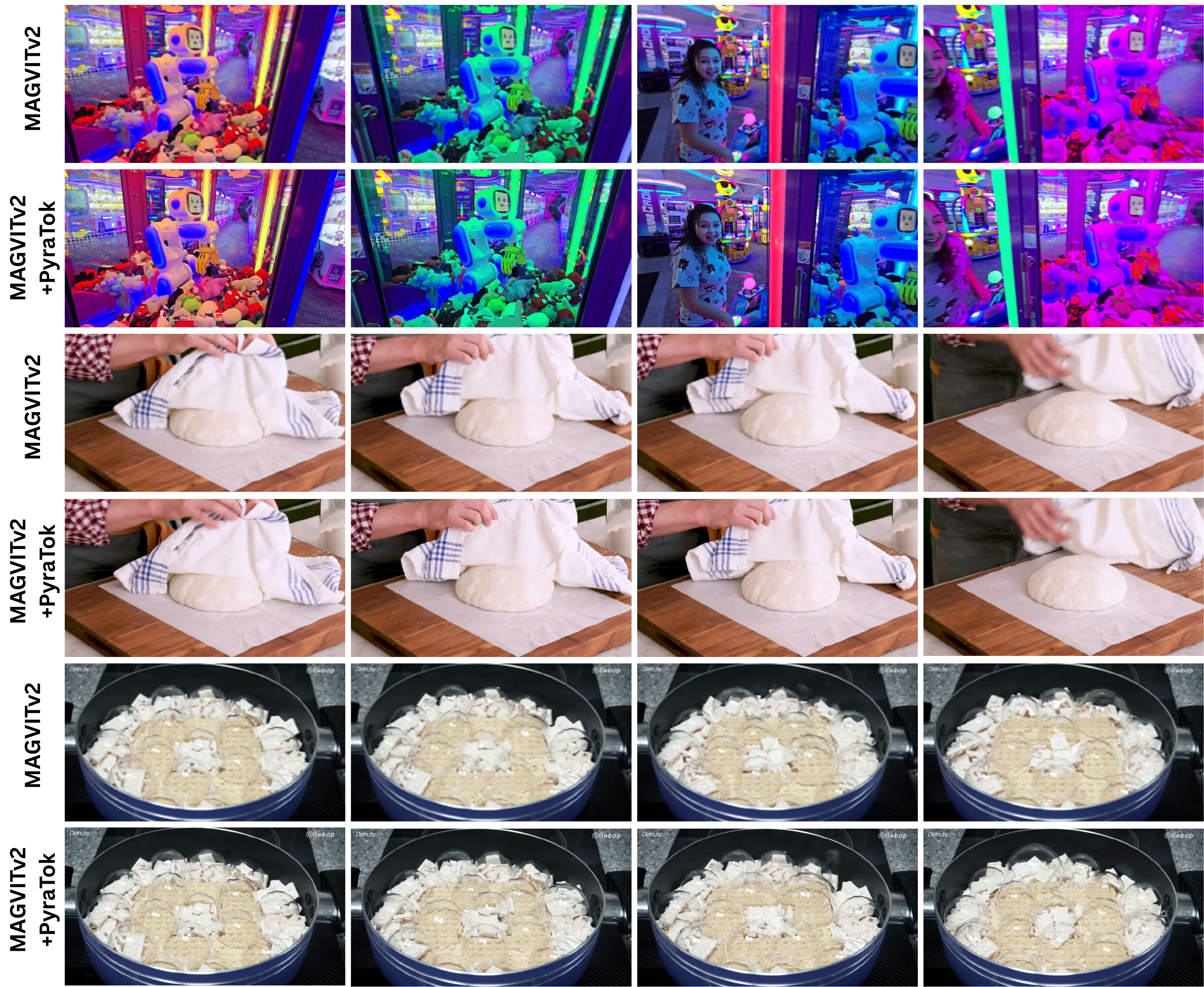}
    \caption{\textbf{Comparison of video reconstruction quality when replacing the default MAGVIT-V2~\cite{yu2024language} VAE with our \modelnamenc{} VAE.} Each pair of rows shows frames generated by the original MAGVIT-V2 (top) and the enhanced MAGVIT-V2 + \modelnamenc{} configuration (bottom). Across diverse scenes, including arcade environments with complex lighting, close-up dough preparation, and detailed cooking sequences, \modelnamenc{} improves visual sharpness, color consistency, and fine-detail preservation, demonstrating its effectiveness as a drop-in VAE replacement for higher-quality video generation. Discussion in \ref{app:t2vprior}.}
    \label{fig:our_mag}
\end{figure*}}

\newcommand{\FigSupomnivae}{
\begin{figure*}[t!]
    \centering
    \includegraphics[width=0.99\linewidth]{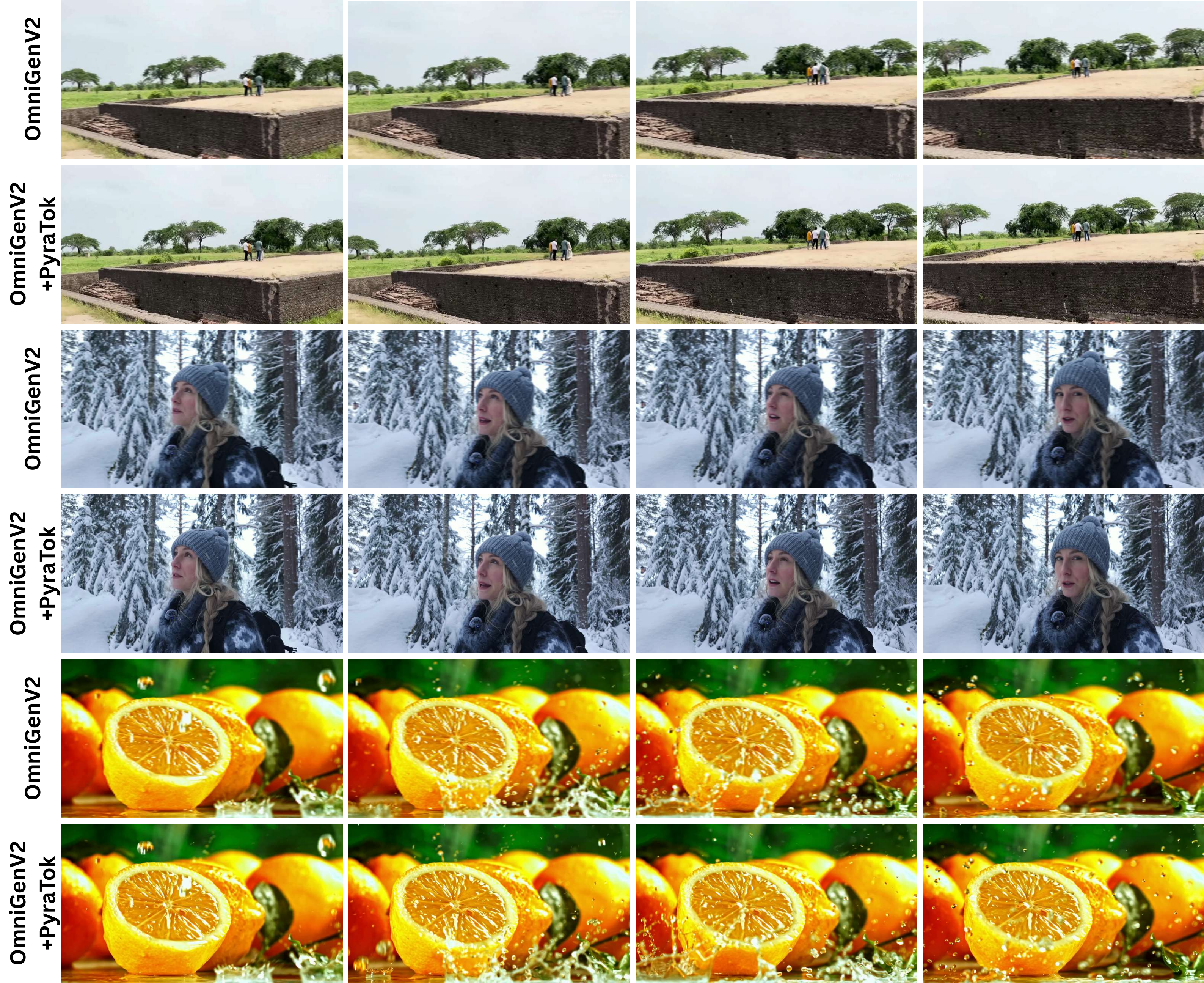}
    \caption{\textbf{Comparison of video generation quality when replacing the default VAE of OmniGen-V2~\cite{wu2025omnigen2} with our \modelnamenc{} VAE.} For each scene, the top row shows frames produced by the original OmniGen-V2, while the bottom row shows frames from OmniGen-V2 + \modelnamenc{}. \modelnamenc{} improves texture sharpness, color fidelity, and fine-detail preservation, demonstrating its effectiveness as a universal, high-quality VAE substitute for diverse video generation pipelines. Discussion in \ref{app:t2vprior}.}
    \label{fig:our_omni}
\end{figure*}}

\newcommand{\FigSupmotionvae}{
\begin{figure*}[t!]
    \centering
    \includegraphics[width=0.99\linewidth]{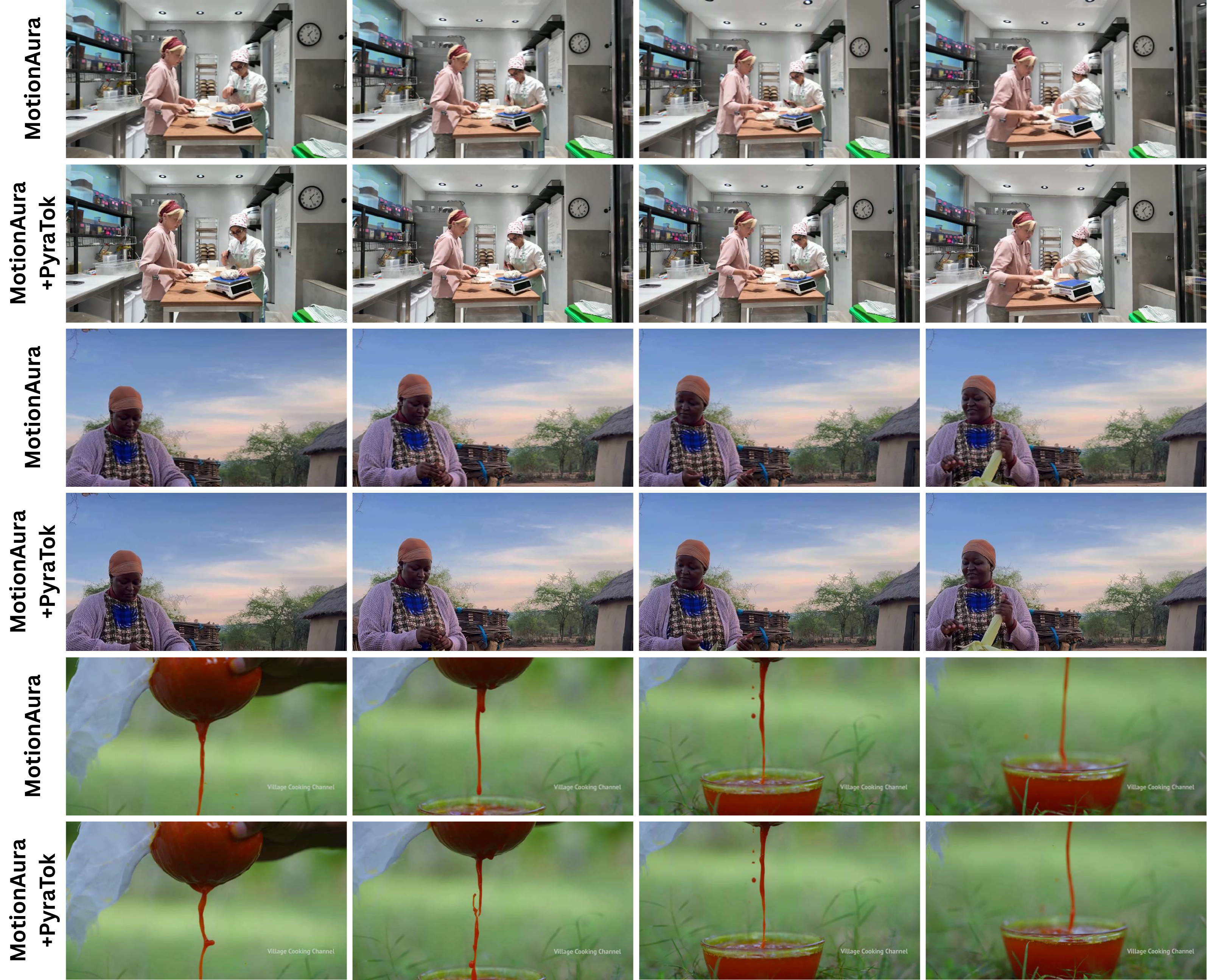}
    \caption{\textbf{Comparison of video generation quality when substituting the default VAE in MotionAura~\cite{susladkar2025motionaura} with our \modelnamenc{} VAE.} For each example, the top row shows frames produced by the original MotionAura, while the bottom row shows results from MotionAura + \modelnamenc{}. Across kitchen scenes, outdoor human activity, and close-up liquid motion, \modelnamenc{} enhances sharpness, preserves fine textures, and improves temporal consistency—demonstrating its effectiveness as a high-quality VAE replacement for improving realism and detail in MotionAura-generated videos. Discussion in \ref{app:t2vprior}.}
    \label{fig:our_motion}
\end{figure*}}

\newcommand{\FigSupactionone}{
\begin{figure*}[t!]
    \centering
    \includegraphics[width=0.85\linewidth]{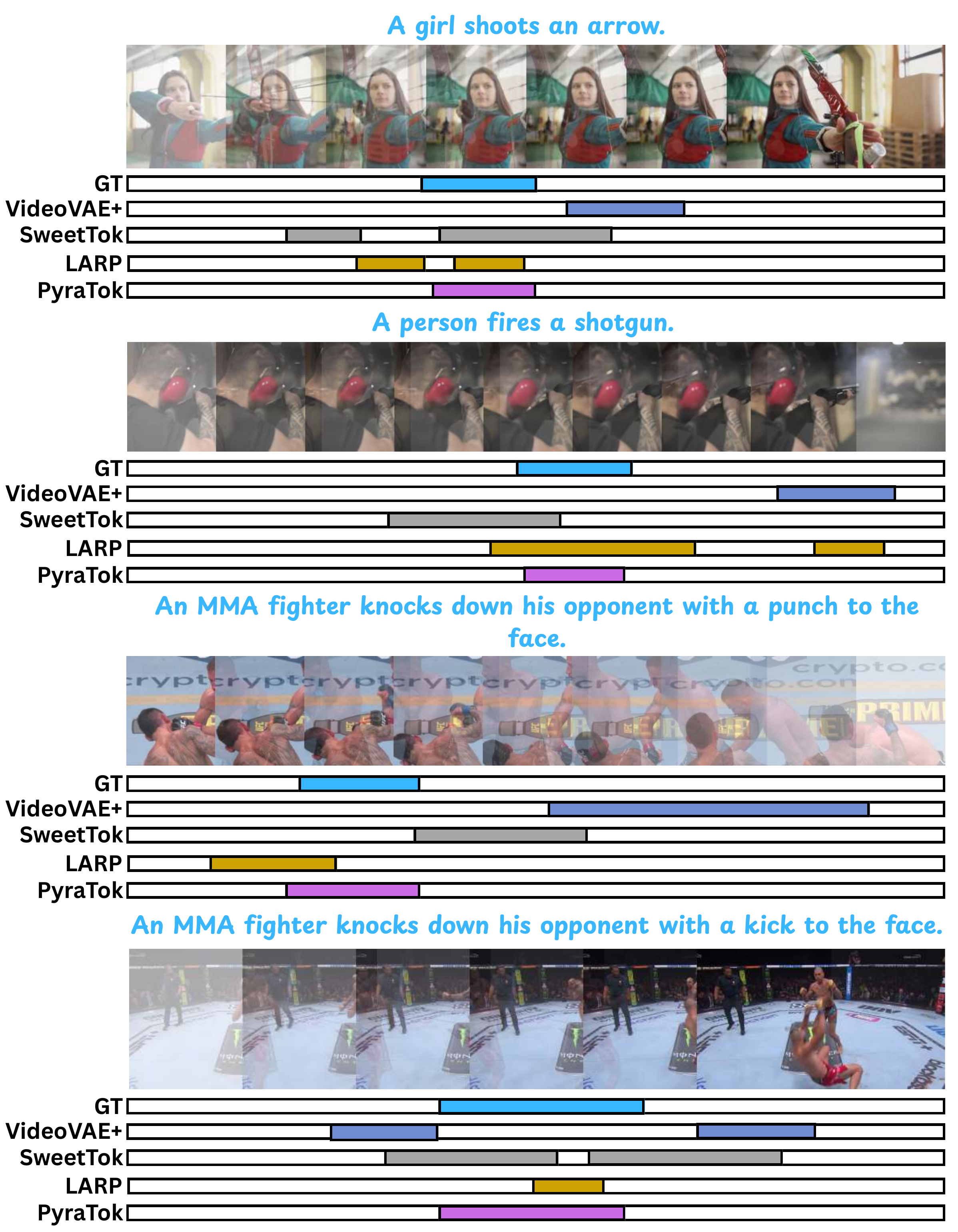}
    \caption{\textbf{Action localization results comparing \modelnamenc{} with several baselines.} For each prompt, the top row shows sampled video frames, followed by temporal action segments for the ground truth and predictions from each method. \modelnamenc{} produces action intervals that align more closely with the ground-truth boundaries, demonstrating improved temporal precision and robustness across diverse actions. Details in \ref{app:action}.\looseness-1}
    \label{fig:action_one}
\end{figure*}}

\newcommand{\FigSupactiontwo}{
\begin{figure*}[t!]
    \centering
    \includegraphics[width=0.85\linewidth]{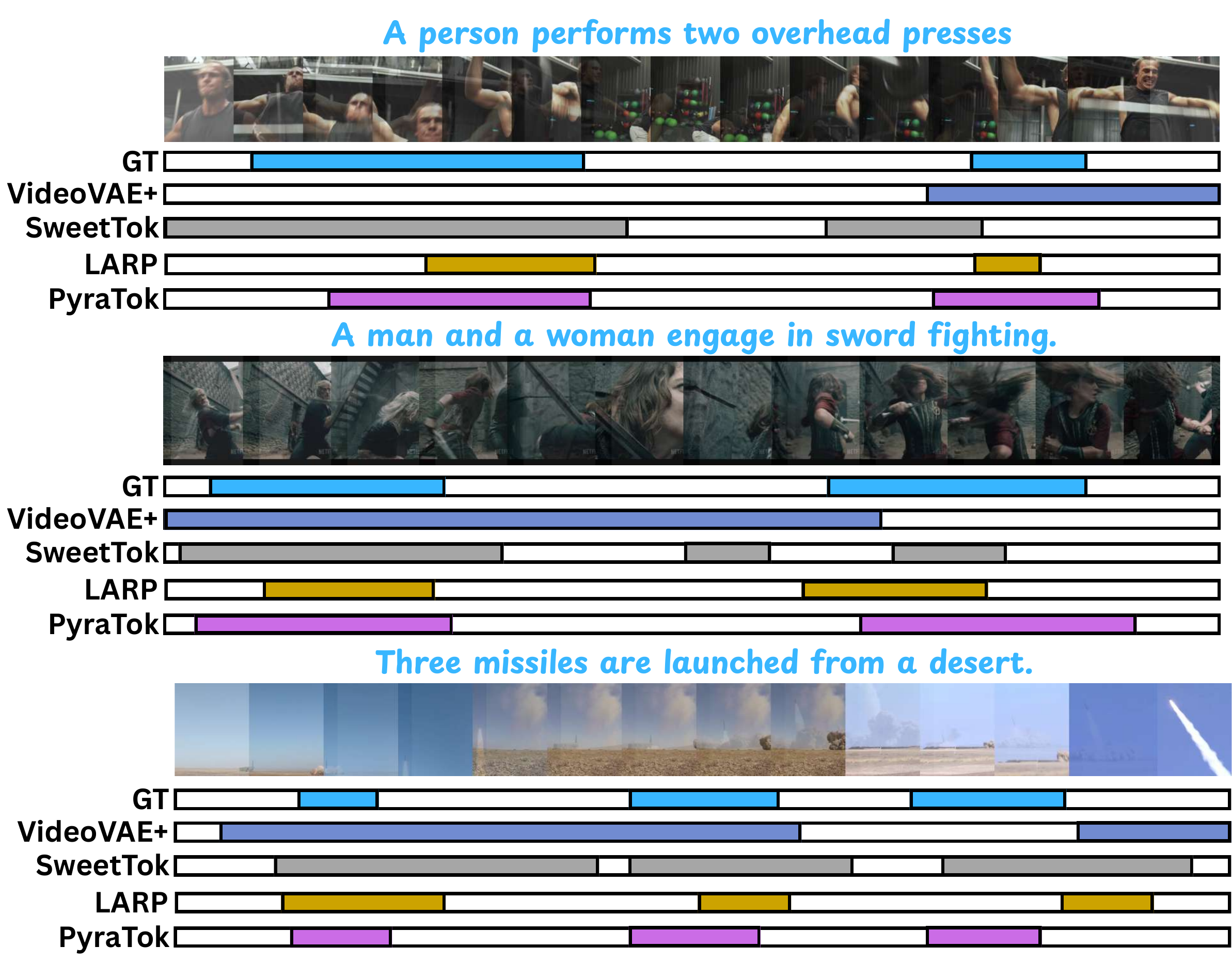}
    \caption{\textbf{Additional action localization comparisons across diverse scenarios.} Each example shows sampled frames followed by ground-truth action intervals and model predictions. \modelnamenc{} consistently yields temporally aligned and coherent action segments, reducing fragmentation and improving boundary accuracy compared to prior baselines. Details in \ref{app:action}.}
    \label{fig:action_two}
\end{figure*}}

\newcommand{\FigSupvideounderstandVUone}{
\begin{figure*}[t!]
    \centering
    \includegraphics[width=0.73\linewidth]{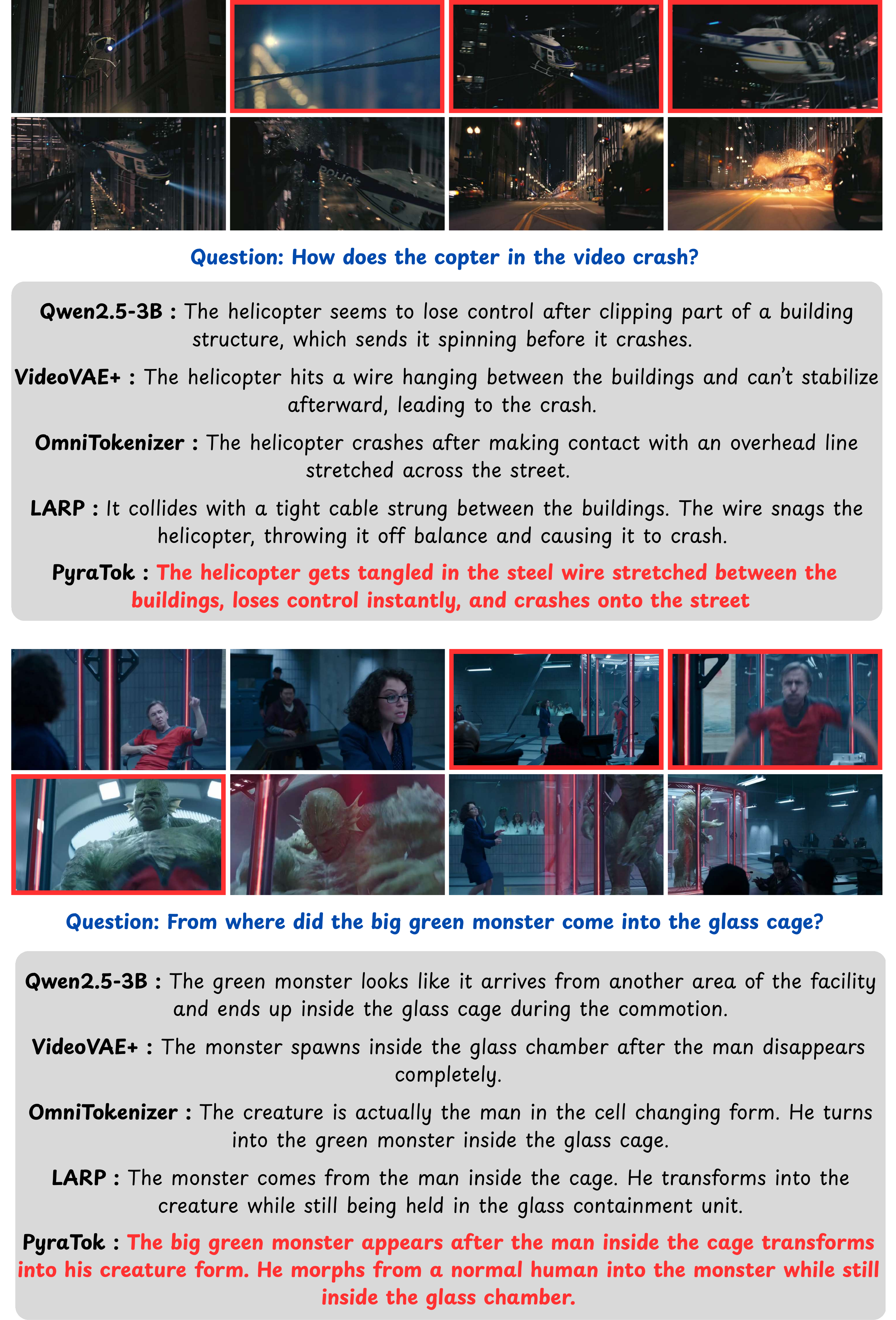}
    \caption{\textbf{Qualitative comparison of video understanding across two challenging temporal-reasoning questions.} For each example, we show representative frames from the input video and the corresponding answers produced by different models. Baselines typically provide generic or partially plausible descriptions that miss key causal events. In contrast, \modelnamenc{} produces precise, temporally grounded explanations that correctly capture fine-grained interactions, entity tracking, and scene dynamics. These examples demonstrate \modelnamenc{}’s ability to support detailed video comprehension and causal reasoning from short video sequences. Discussion in \ref{app:vqa}.}
    \label{fig:video_understand_supp_vu_one}
\end{figure*}}

\newcommand{\FigSupvideounderstandVUtwo}{
\begin{figure*}[t!]
    \centering
    \includegraphics[width=0.8\linewidth]{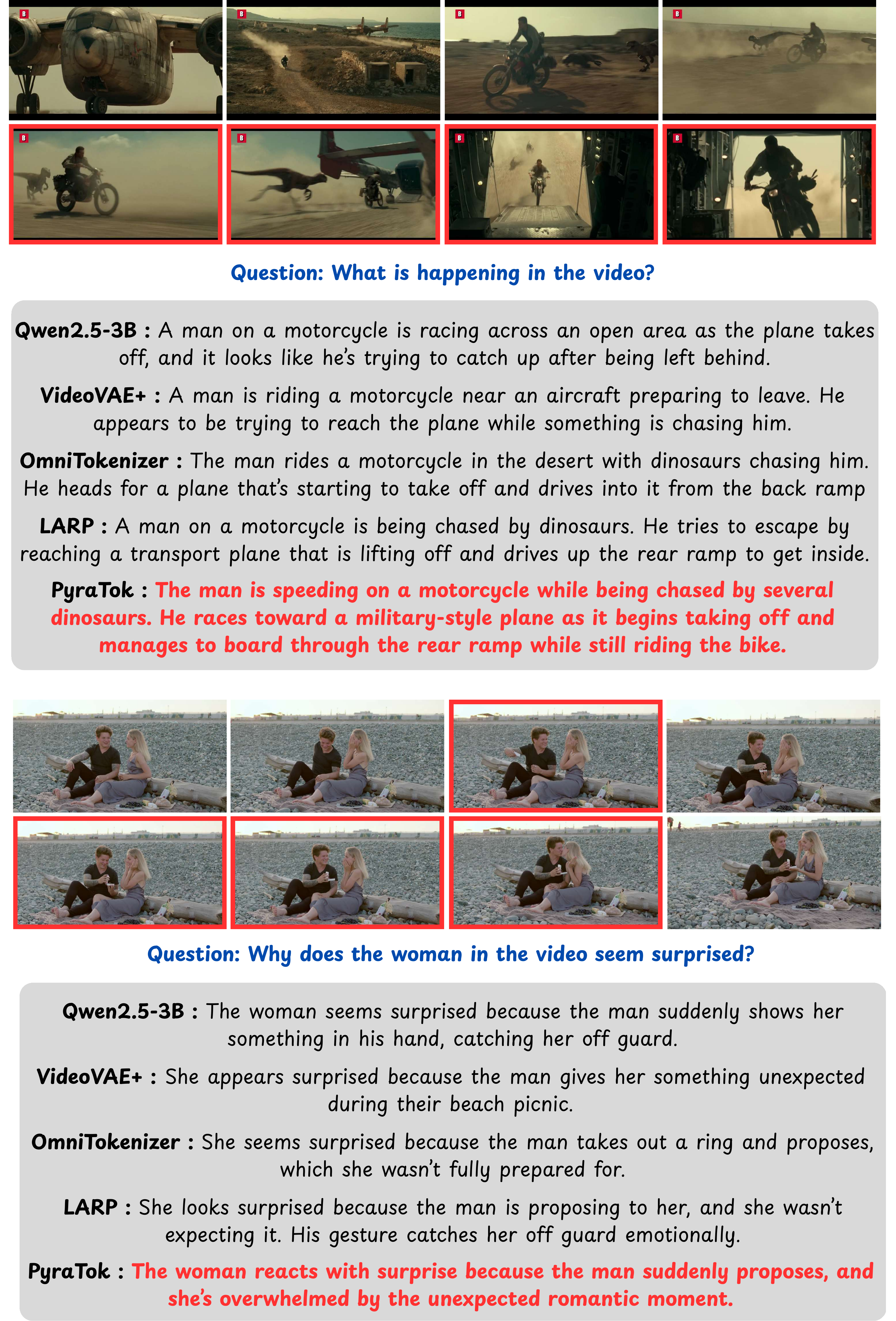}
    \caption{\textbf{Qualitative comparison of video reasoning ability across models.} Models describe major actions in two dynamic scenes (a motorcycle escape from dinosaurs and a surprise beach proposal). \modelnamenc{} delivers the most precise and context-aware answers across both scenarios. Discussion in \ref{app:vqa}.}
    \label{fig:video_understand_supp_vu_two}
\end{figure*}}

\newcommand{\FigSupvideounderstandVUthree}{
\begin{figure*}[t!]
    \centering
    \includegraphics[width=0.73\linewidth]{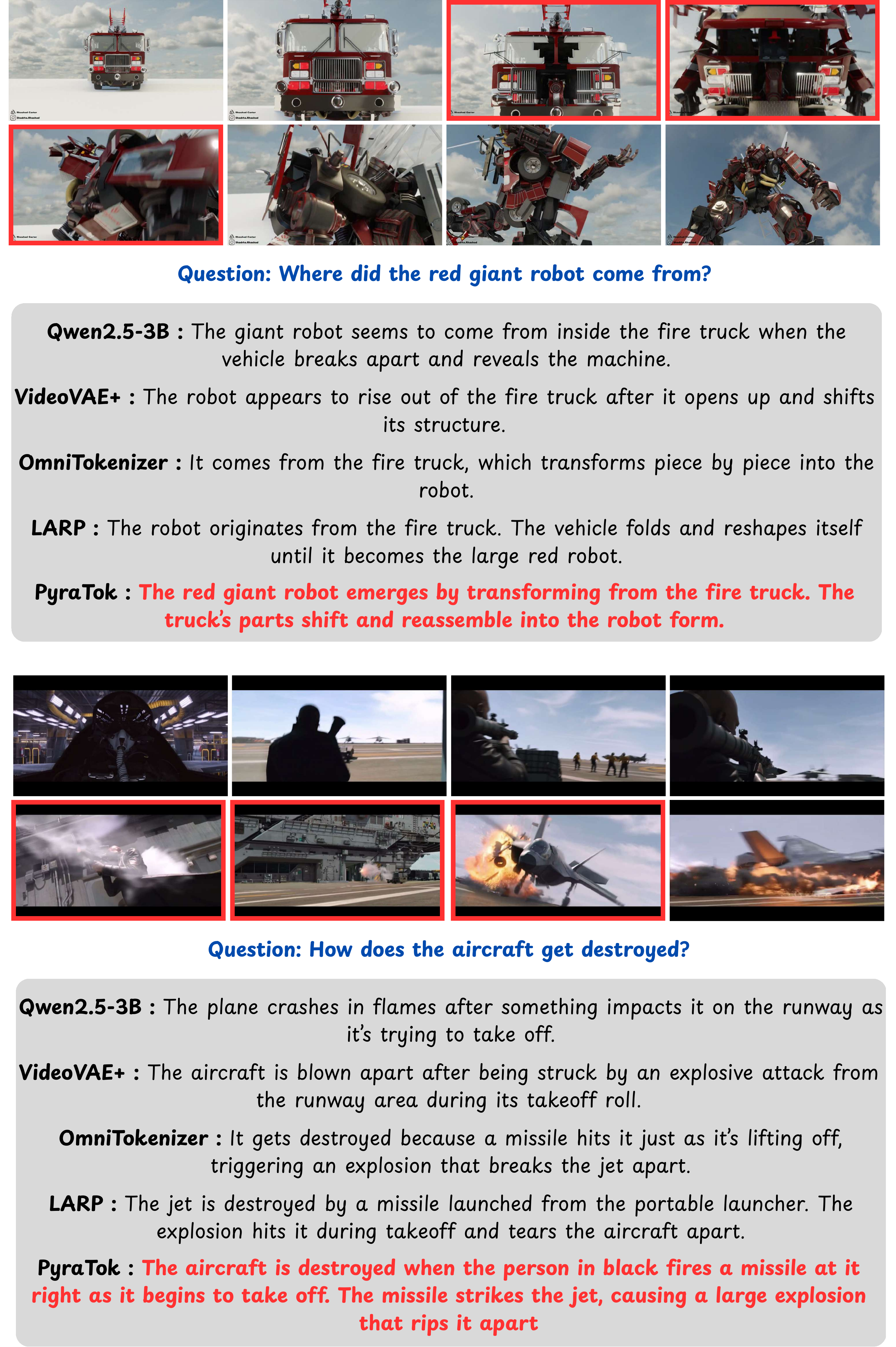}
    \caption{\textbf{Qualitative comparison of video understanding on two transformation- and action-level reasoning tasks.} Baseline methods provide generic or underspecified descriptions (\eg stating that the robot ``comes from the fire truck''), often missing key causal events, responsible agents, and transformation mechanics. In contrast, \modelnamenc{} produces precise, temporally grounded explanations that correctly identify object transformations, causal triggers, and scene dynamics, such as the fire truck’s parts reassembling into the robot or the person in black firing the missile that destroys the aircraft. Discussion in \ref{app:vqa}.}
    \label{fig:video_understand_supp_vu_three}
\end{figure*}}

\newcommand{\FigSuptrainingplot}{
\begin{figure}[H]
    \centering
    \includegraphics[width=0.9\linewidth]{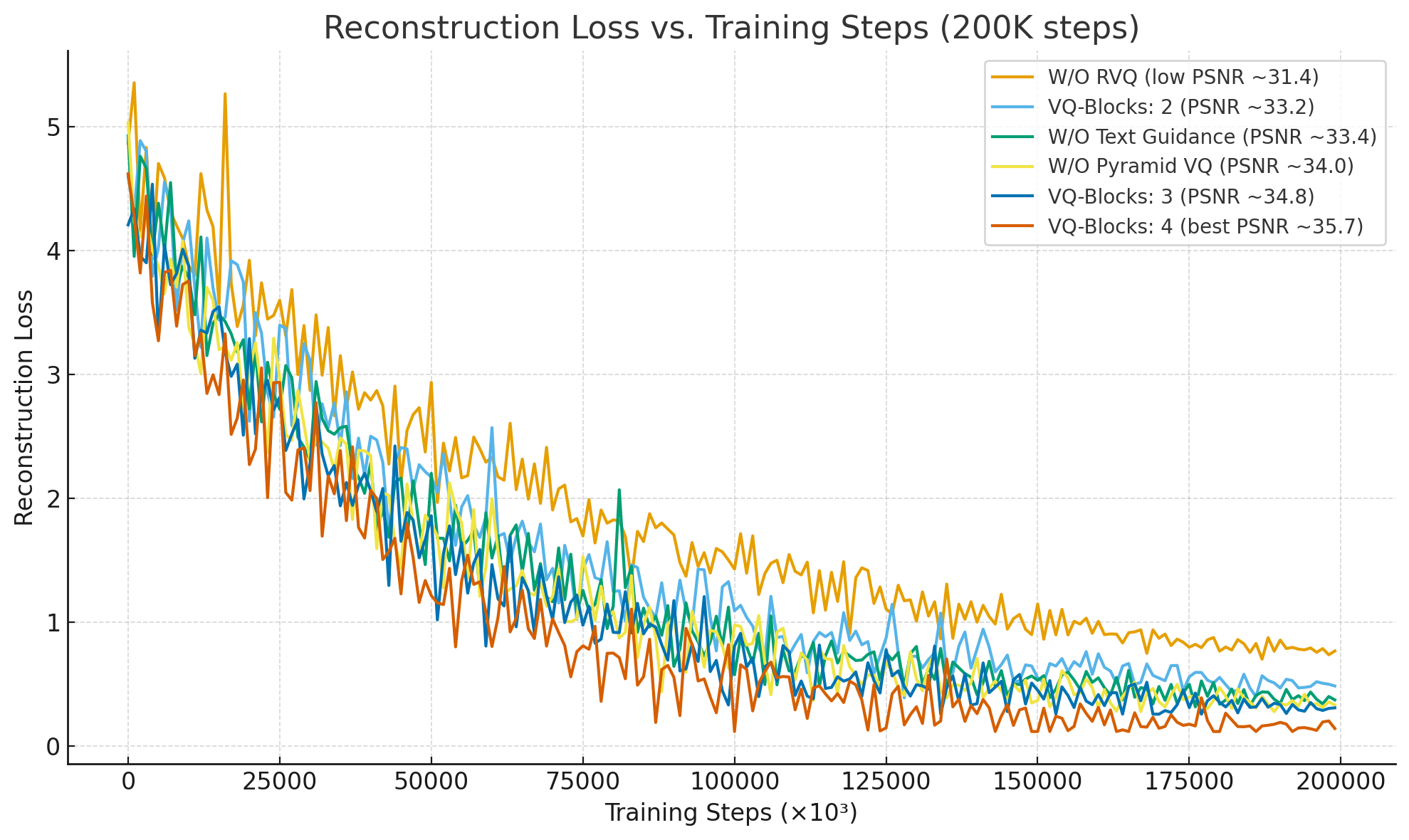}
    \caption{\textbf{Reconstruction loss over 200K training steps.} The best PSNR configuration (VQ-Blocks: 4) converges at a loss of {0.12}, while other ablation variants stabilize above {0.25}.
    }
    \label{fig:recon_loss_200k}  
\end{figure}
}

\newcommand{\FigSupcodebookutio}{
\begin{figure}[H]
    \centering
    \includegraphics[width=0.9\linewidth]{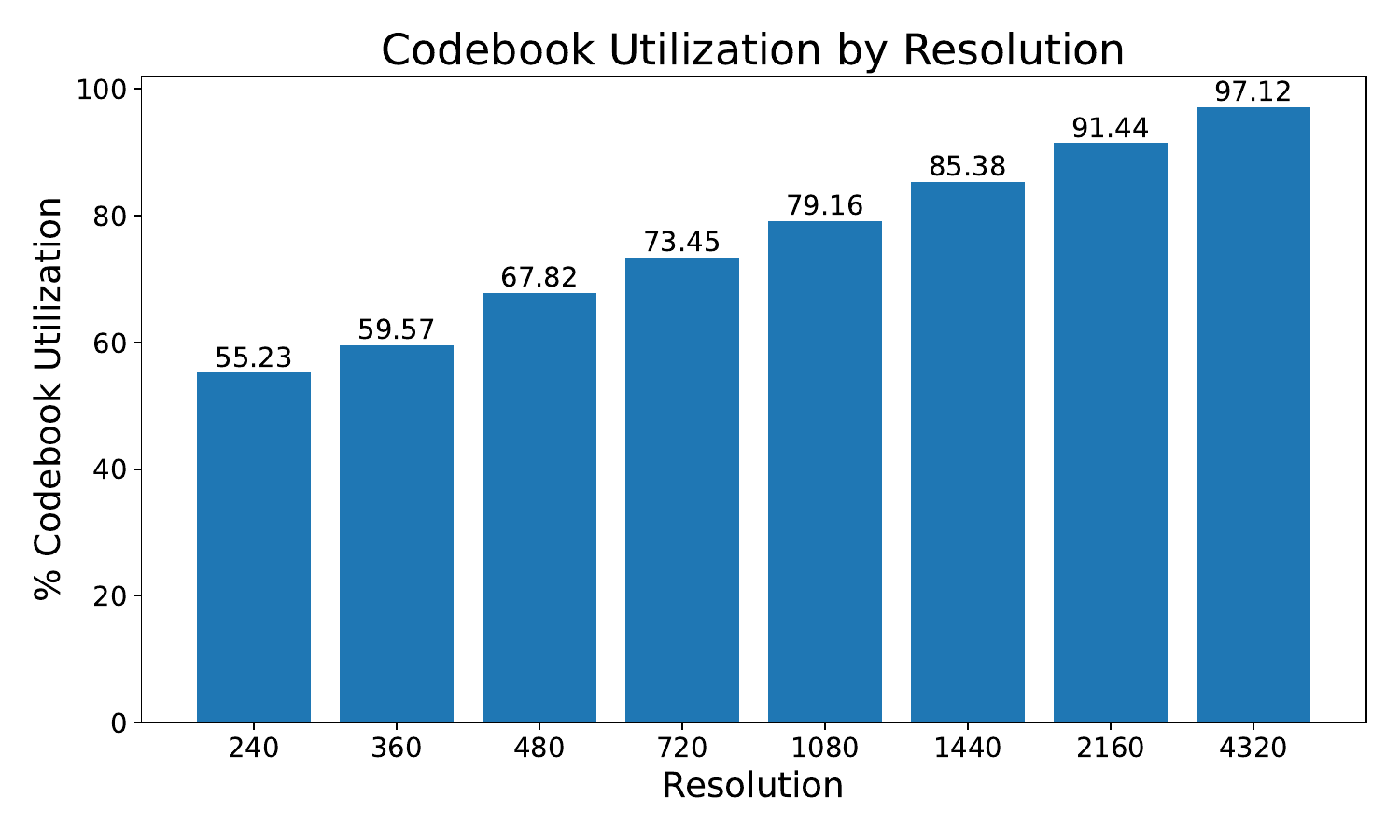}
    \caption{\textbf{Codebook utilization as a function of input resolution.} Higher resolutions activate a larger fraction of the vocabulary, indicating that \modelnamenc{} effectively exploits the increased spatial support to encode more diverse semantics.
}
\label{fig:codebook_resolution}
\vspace{-0.3cm}
\end{figure}}

%% file: assets/tex/tables.tex
\newcommand{\TableRecon}{
\begin{table*}[t!]
\scriptsize
\centering
\caption{\textbf{Reconstruction quality comparison.} Latency measured on 25 frames (256$\times$256) using a single V100 GPU. Best highlighted with \textbf{bold} and second-best \ul{underlined}.}
\vspace{0.1cm}
\resizebox{0.99\linewidth}{!}{
\begin{tabular}{lcc|ccc|ccc}
\toprule
\multirow{2}{*}{\textbf{Methods}} & \multirow{2}{*}{\textbf{Params}} & \multirow{2}{*}{\textbf{Latency}} &
\multicolumn{3}{c|}{\textbf{WebVid-10M}} & \multicolumn{3}{c}{\textbf{COCO-Val}}  \\[2pt]
 & & (ms) & PSNR ($\uparrow$) & SSIM ($\uparrow$) & LPIPS ($\downarrow$) & PSNR ($\uparrow$) & SSIM ($\uparrow$) & LPIPS ($\downarrow$) \\
\hline
CogVideoX~\cite{yang2025cogvideox} & 288M & 712 & 29.92 & 0.811 & 0.141 & 30.11 & 0.833 & 0.111 \\
3D-MBQ-VAE~\cite{susladkar2025motionaura} & 317M & 650 & 33.00 & 0.848 & 0.092 & 32.11 & 0.858 & 0.108 \\
WAN 2.2~\cite{wan2025wan} & 222M & 449 & 32.94 & 0.841 & 0.101 & 33.43 & 0.861 & 0.103 \\
OmniTokenizer~\cite{wang2024omnitokenizer} & 82M & 444 & 32.03 & 0.812 & 0.152 & 32.09 & 0.845 & 0.141 \\
LARP~\cite{wang2025larp} & 183M & 689 & 33.03 & 0.851 & 0.091 & 34.26 & 0.853 & 0.089 \\
\hline
TokenFlow~\cite{qu2025tokenflow} & 176M & 600 & 28.21 & 0.799 & 0.189 & 30.11 & 0.811 & 0.177 \\
VideoVae+~\cite{xing2024large} & 192M & 555 & 29.17 & 0.812 & 0.201 & 31.45 & 0.832 & 0.162 \\
TexTok~\cite{zha2025language} & 173M & 661 & 27.42 & 0.831 & 0.222 & 29.29 & 0.841 & 0.181 \\
LG-VQ~\cite{guotao2024lg} & 168M & 598 & 30.23 & 0.807 & 0.173 & 31.32 & 0.836 & 0.152 \\
TokLIP~\cite{lin2025toklip} & 207M & 604 & 31.28 & 0.837 & 0.152 & \uline{33.42} & \uline{0.849} & \uline{0.105} \\
SweetTok~\cite{tan2025sweettok} & 128M & 432 & \uline{32.32} & \uline{0.842} & \uline{0.137} & 32.78 & 0.847 & 0.123 \\
\hline
\rowcolor{purple!5} \modelname\textbf{(Ours)}& 192M & 492 & \textbf{35.72} & \textbf{0.879} & \textbf{0.066} & \textbf{36.05} & \textbf{0.885} & \textbf{0.071} \\
\hline
\end{tabular}
}
\label{tab:quantitative_comparison}
\end{table*}}

\newcommand{\TableCompression}{
\begin{table}[t!]
\centering
\caption{\textbf{Video compression at 0.034 bitrate.}}
\resizebox{0.9\linewidth}{!}{
\begin{tabular}{lccc}
\toprule
\textbf{Methods} & PSNR ($\uparrow$) & SSIM ($\uparrow$) & LPIPS ($\downarrow$) \\
\midrule
HEVC~\cite{sullivan2012overview} & 30.10 & 0.943 & 0.199 \\
VCC~\cite{bross2021overview} & 32.65 & 0.966 & 0.153 \\
\midrule
MAGVIT~\cite{yu2023magvit}  & 23.70 & 0.846 & 0.144 \\
MAGVIT-v2~\cite{yu2024language} & 26.18 & 0.894 & 0.104 \\
3D-MBQ-VAE~\cite{susladkar2025motionaura} & \uline{29.09} & \uline{0.922} & \uline{0.089} \\
\midrule
\rowcolor{purple!5} \modelname\textbf{(Ours)}& \textbf{29.82} & \textbf{0.942} & \textbf{0.068} \\
\bottomrule
\end{tabular}
}
\label{tab:mcljcv_analysis}
\end{table}
\vspace{-0.3cm}
}

\newcommand{\TableVidSeg}{
\begin{table}[t!]
\centering
\caption{\textbf{Video semantic segmentation results on YouTube-VIS 2021 and OVIS. } Best highlighted with \textbf{bold} and second-best \ul{underlined}.
\textcolor{blue}{\faStar} supervised,
\textcolor{green!60!black}{\faCog} unsupervised, 
\textcolor{purple}{\faRocket} zero-shot methods.
}
\resizebox{\linewidth}{!}{
\begin{tabular}{l >{\centering\arraybackslash}m{1.5cm} c c c c}
\toprule
\multirow{2}{*}{\textbf{Method}} &
\multirow{2}{*}{\textbf{Training}} &
\multicolumn{2}{c}{\textbf{YouTube-VIS 2021}} &
\multicolumn{2}{c}{\textbf{OVIS}} \\
\cmidrule(lr){3-4} \cmidrule(lr){5-6}
 &  & \textbf{mAP ($\uparrow$)} & \textbf{Jaccard ($\uparrow$)} &
\textbf{mAP ($\uparrow$)} & \textbf{Jaccard ($\uparrow$)} \\
\midrule
CLIP-VIS~\cite{zhu2024clip} & \textcolor{blue}{\faStar} & 44.2 & 76.31 & 18.6 & 60.09 \\
VideoCutLER~\cite{wang2024videocutler} & \textcolor{green!60!black}{\faCog}  & 17.1 & 62.23 & -- & -- \\
UVIS~\cite{huang2024uvis} & \textcolor{green!60!black}{\faCog}  & 17.5 & 63.11 & 3.5 & 36.71 \\
\midrule
VideoVae+~\cite{xing2024large} & \textcolor{purple}{\faRocket} & 12.33 & 51.21 & 2.8 & 29.91 \\
LARP~\cite{wang2025larp} & \textcolor{purple}{\faRocket} & 10.52 & 49.37 & 1.7 & 28.45 \\
OmniTokenizer~\cite{wang2024omnitokenizer} & \textcolor{purple}{\faRocket} & \uline{14.54} & \uline{51.12} & \uline{2.8} & \uline{33.27} \\
\midrule
\rowcolor{purple!5} \modelname\textbf{(Ours)} & \textcolor{purple}{\faRocket} & \textbf{24.54} & \textbf{66.56} & \textbf{8.9} & \textbf{49.44} \\
\bottomrule
\end{tabular}
}
\label{tab:vis_results}
\vspace{-0.3cm}
\end{table}}

\newcommand{\TableVidAct}{
\begin{table}[t!]
\centering
\caption{\textbf{Video action localization under the 50\% Seen / 50\% Unseen setup.} Best highlighted with \textbf{bold} and second-best \ul{underlined}. \textcolor{blue}{\faStar} supervised and
\textcolor{purple}{\faRocket} zero-shot methods.}
\resizebox{\linewidth}{!}{
\begin{tabular}{l >{\centering\arraybackslash}m{1.5cm} c c c c}
\toprule
\multirow{2}{*}{\textbf{Method}} &
\multirow{2}{*}{\textbf{Training}} &
\multirow{2}{*}{\textbf{VAE}} &
\textbf{THUMOS14} & \textbf{ActivityNet v1.3} \\ 
 & & &  \textbf{Avg. mAP} ($\uparrow$) & \textbf{Avg. mAP} ($\uparrow$) \\
\midrule
STALE~\cite{nag2022zero} & \textcolor{blue}{\faStar} & \xmark & 22.2 & 20.5 \\
DeTAL~\cite{li2024detal} & \textcolor{blue}{\faStar} & \xmark & 24.1 & 22.4 \\
STOV-TAL~\cite{hyun2025exploring}  & \textcolor{blue}{\faStar} & \xmark & 48.8 & 29.6 \\
\cmidrule(lr){1-5}
STOV-TAL~\cite{hyun2025exploring}  & \textcolor{purple}{\faRocket} & \xmark & 31.5 & 28.0 \\
VideoVae+~\cite{xing2024large} & \textcolor{purple}{\faRocket} & \cmark & 23.12 & 21.37 \\
OmniTokenizer~\cite{wang2024omnitokenizer} & \textcolor{purple}{\faRocket} & \cmark & 23.47 & 22.48 \\
SweetTok~\cite{tan2025sweettok} & \textcolor{purple}{\faRocket} & \cmark & 25.32 & 24.53 \\
LARP~\cite{wang2025larp} & \textcolor{purple}{\faRocket} & \cmark & \uline{27.42} & \uline{25.53} \\
\midrule
\rowcolor{purple!5} \modelname\textbf{(Ours)}& \textcolor{purple}{\faRocket} & \cmark & \textbf{33.17} & \textbf{29.11} \\
\bottomrule
\end{tabular}
}
\label{tab:action_localization}
\vspace{-0.3cm}
\end{table}}

\newcommand{\TableVidUnderstand}{
\begin{table}[t!]
\centering
\caption{\textbf{General video understanding results on MVBench.} Higher scores indicate better overall performance.}
\resizebox{0.6\linewidth}{!}{
\begin{tabular}{l c}
\toprule
\textbf{Methods} & \textbf{Overall} ($\uparrow$) \\
\midrule
InternVL3-78B* & 79.2 \\
Qwen2.5-72B* & 71.3 \\
InternVL3-38B* & 76.0 \\
Qwen2.5VL-7B* & 67.2 \\
Qwen2.5VL-3B & 67.0 \\
\midrule
VILA-U & 81.21 \\
OmniTokenizer & 79.44 \\
LARP & 83.21 \\
\rowcolor{purple!5} \modelname\textbf{(Ours)}& \textbf{86.03} \\
\bottomrule
\end{tabular}
}
\label{tab:mvbench_results}
\end{table}}

\newcommand{\TableVidClass}{
\begin{table}[t!]
\centering
\caption{\textbf{Video classification results on Kinetics-400, Kinetics-600, and Kinetics-700 datasets.} Higher accuracy (\%) indicates better performance.}
\resizebox{\linewidth}{!}{
\begin{tabular}{l c c c c}
\toprule
\multirow{2}{*}{\textbf{Model}} &
\multirow{2}{*}{\textbf{VAE}} &
\multicolumn{3}{c}{\textbf{Kinetics}} \\ 
 & & \textbf{400} & \textbf{600} & \textbf{700} \\
\midrule
InternVL & \xmark & 69.1 & 68.9 & 60.6 \\
InternVideo2 & \xmark & 73.1 & 72.8 & 64.9 \\
VideoPrism-g* & \xmark & 76.4 & -- & -- \\
SigLIP2-g-opt† & \xmark & 69.8 & 67.0 & 61.8 \\
PEcoreG & \xmark & \textbf{76.9} & \textbf{76.1} & \textbf{69.1} \\
\midrule
VideoVae+~\cite{xing2024large} & \cmark & 63.32 & 61.27 & 55.55 \\
OmniTokenizer~\cite{wang2024omnitokenizer} & \cmark & 65.03 & 62.75 & 58.71 \\
SweetTok~\cite{tan2025sweettok} & \cmark & 67.54 & 65.01 & 61.45 \\
LARP~\cite{wang2025larp} & \cmark & 69.27 & 68.52 & 66.89 \\
\rowcolor{purple!5} \modelname\textbf{(Ours)}& \cmark & \textbf{78.43} & \textbf{77.11} & \textbf{74.08} \\
\bottomrule
\end{tabular}
}
\label{tab:video_classification}
\end{table}}

\newcommand{\TableVidUnified}{
\begin{table}[t!]
\centering
\caption{\textbf{Accuracy (\%) on general video understanding and video classification.} Best highlighted with \textbf{bold} and second-best \ul{underlined}.}
\vspace{0.1cm}
\resizebox{\linewidth}{!}{
\begin{tabular}{l c c c c c}
\toprule
\multirow{2}{*}{\textbf{Method}} &
\multirow{2}{*}{\textbf{VAE}} &
\textbf{MVBench} & 
\multicolumn{3}{c}{\textbf{Kinetics}} \\
\cmidrule(lr){3-3}\cmidrule(lr){4-6}
 & & \textbf{Overall} & \textbf{400} & \textbf{600} & \textbf{700} \\
\midrule
InternVL3-78B~\cite{zhu2025internvl3} & \xmark & 79.2 & -- & -- & -- \\
Qwen2.5-72B~\cite{bai2025qwen2} & \xmark & 71.3 & -- & -- & -- \\
InternVL3-38B~\cite{zhu2025internvl3} & \xmark & 76.0 & -- & -- & -- \\
Qwen2.5VL-7B~\cite{bai2025qwen2}  & \xmark & 67.2 & -- & -- & -- \\
Qwen2.5VL-3B~\cite{bai2025qwen2}  & \xmark & 67.0 & -- & -- & -- \\
InternVL~\cite{zhu2025internvl3} & \xmark & -- & 69.1 & 68.9 & 60.6 \\
InternVideo2~\cite{wang2024internvideo2} & \xmark & -- & 73.1 & 72.8 & 64.9 \\
VideoPrism-g~\cite{zhao2024videoprism} & \xmark & -- & 76.4 & -- & -- \\
SigLIP2-g-opt\cite{tschannen2025siglip} & \xmark & -- & 69.8 & 67.0 & 61.8 \\
PEcoreG~\cite{bolya2025perception} & \xmark & -- & 76.9 & 76.1 & 69.1 \\
\midrule
VILA-U~\cite{wu2025vila} & \cmark & 81.21 & -- & -- & -- \\
VideoVae+~\cite{xing2024large} & \cmark & -- & 63.32 & 61.27 & 55.55 \\
OmniTokenizer~\cite{wang2024omnitokenizer} & \cmark & 79.44 & 65.03 & 62.75 & 58.71 \\
SweetTok~\cite{tan2025sweettok} & \cmark & -- & 67.54 & 65.01 & 61.45 \\
LARP~\cite{wang2025larp} & \cmark & \uline{83.21} & \uline{69.27} & \uline{68.52} & \uline{66.89} \\
\midrule
\rowcolor{purple!5} \modelname\textbf{(Ours)}& \cmark & \textbf{86.03} & \textbf{78.43} & \textbf{77.11} & \textbf{74.08} \\
\bottomrule
\end{tabular}
}
\label{tab:unified_video_results}
\end{table}}

\newcommand{\TableAlbs}{
\begin{table}[t!]
\centering
\caption{\textbf{Ablations on \modelnamenc components.}}
\resizebox{\columnwidth}{!}{
\begin{tabular}{lcc}
\toprule
 & COCO-Val & WebVid-10M \\
 & PSNR / SSIM / LPIPS & PSNR / SSIM / LPIPS \\
\midrule
\rowcolor{gray!10} \multicolumn{3}{c}{\textcolor{fire1}{\textbf{1. Component Ablation}}}\\\midrule
\texttt{w/o LaPQ} & 31.41 / 0.831 / 0.101 & 31.47 / 0.799 / 0.118 \\
\texttt{w/o Text Guidance} & 33.43 / 0.861 / 0.081 & 36.02 / 0.833 / 0.082 \\
\texttt{w/o Pyramidal-Q} & 34.02 / 0.859 / 0.082 & 34.02 / 0.839 / 0.094 \\
\midrule
\rowcolor{gray!10} \multicolumn{3}{c}{\textcolor{fire6}{\textbf{2. Quantization($\mathcal{Q}$)-Blocks Ablation}}} \\\midrule
2 Blocks & 33.21 / 0.821 / 0.092 & 33.98 / 0.844 / 0.101 \\
3 Blocks & 34.78 / 0.862 / 0.089 & 35.14 / 0.867 / 0.085 \\
4 Blocks (Default) & 35.72 / 0.879 / 0.066 & 36.05 / 0.885 / 0.071 \\
\midrule
\rowcolor{gray!10}  \multicolumn{3}{c}{\textcolor{fire5}{\textbf{3. Loss Function Ablation}}}\\\midrule
w/o $\mathcal{L}_{\text{drift}}$ & 33.48 / 0.839 / 0.082 & 34.52 / 0.853 / 0.081 \\
w/o $\mathcal{L}_{\text{AR}}$ & 33.42 / 0.842 / 0.079 & 34.01 / 0.844 / 0.079 \\
w/o $\mathcal{L}_{\text{drift}}$ \& $\mathcal{L}_{\text{AR}}$ & 32.17 / 0.832 / 0.093 & 32.32 / 0.831 / 0.092 \\
\midrule
\rowcolor{gray!10} \multicolumn{3}{c}{\textcolor{fire4}{\textbf{4. Codebook Loss Ablation}}}\\\midrule
w/o $\mathcal{L}_{\text{vision-commitment}}$ & 32.88 / 0.819 / 0.097 & 33.45 / 0.839 / 0.101 \\
w/o $\mathcal{L}_{\text{text-cond. alignment}}$ & 33.27 / 0.822 / 0.092 & 34.12 / 0.855 / 0.091 \\
w/o $\mathcal{L}_{\text{text-codebook alignment}}$ & 34.11 / 0.849 / 0.087 & 34.78 / 0.872 / 0.083 \\
\midrule
\rowcolor{gray!10} \multicolumn{3}{c}{\textcolor{fire3}{\textbf{5. Multi-Modal Models}}} \\\midrule
Qwen-2.5 VL~\cite{bai2025qwen2} (Default)  & 35.72 / 0.879 / 0.066 & 36.05 / 0.885 / 0.071 \\
LLaMA-3 8B~\cite{grattafiori2024llama}  & 35.62 / 0.871 / 0.069 & 35.34 / 0.878 / 0.079 \\
Gemma-3 4B~\cite{team2025gemma}  & 35.29 / 0.865 / 0.069 & 35.92 / 0.882 / 0.078 \\
\midrule
\rowcolor{gray!10} \multicolumn{3}{c}{\textcolor{fire2}{\textbf{6. Pretrained VAEs}}} \\\midrule
3D-MBQ-VAE~\cite{susladkar2025motionaura}  & 35.01 / 0.869 / 0.069 & 35.33 / 0.878 / 0.075 \\
CogVideoX-VAE~\cite{yang2025cogvideox}  & 34.92 / 0.861 / 0.069 & 35.12 / 0.873 / 0.080 \\
SVD-VAE~\cite{blattmann2023stable}  & 34.18 / 0.859 / 0.074 & 34.78 / 0.865 / 0.083 \\
Mochi-VAE~\cite{genmo2024mochi}  & 34.95 / 0.864 / 0.071 & 35.06 / 0.873 / 0.076 \\
\midrule
\rowcolor{purple!5} \modelname & 36.05 / 0.885 / 0.071 & 35.72 / 0.879 / 0.066  \\
\bottomrule
\end{tabular}}
\label{tab:ablation_study}
\vspace{-0.4cm}
\end{table}}

\newcommand{\Tabletextv}{
\begin{table}[t!]
\centering
\caption{\textbf{T2V performance on WebVid-10M.} Incorporating \modelnamenc{} (\cmark) consistently improves perceptual quality and semantic alignment compared to base models without it (\xmark).}
\resizebox{0.99\columnwidth}{!}{
\begin{tabular}{lcccc}
\toprule
\multirow{2}{*}{\textbf{Base Model}} & \multirow{2}{*}{\textbf{Type}} & \multicolumn{2}{c}{\textbf{FVD} ($\downarrow$) / \textbf{TC} ($\uparrow$) } \\
\cmidrule(lr){3-4}
& & \textbf{\xmark~\modelname} & \textbf{\cmark~\modelname}\\
\midrule
MotionAura~\cite{susladkar2025motionaura}& Discrete Diffusion & 374 / 204 & \textbf{365 / 246} \\
Open MAGVITv2~\cite{luo2024open}  & AutoRegressive & 433 / 191 & \textbf{411 / 214} \\
Omnigenv2~\cite{wu2025omnigen2}  & AutoRegressive & 398 / 185 & \textbf{377 / 208} \\
\bottomrule
\end{tabular}
}
\label{tab:t2v}
\vspace{-0.3cm}
\end{table}
}

\newcommand{\Tablesupunderstanding}{
\begin{table}[t!]
\centering
\caption{\textbf{Ablation on loss functions.}}
\label{tab:lossund}
\resizebox{0.97\columnwidth}{!}{
\begin{tabular}{lccc}
\hline
 & \textbf{THUMOS14} & \textbf{ActivityNet} & \textbf{MVBench} \\
\hline
\xmark~~$\mathcal{L}_{\text{dift}}$                    & 31.27 & 27.62 & 83.32 \\
\xmark~~$\mathcal{L}_{\text{AR}}$                      & 32.45 & 27.98 & 79.45 \\
\xmark~~$\mathcal{L}_{\text{dino}}$ \& $\mathcal{L}_{\text{AR}}$ & 29.29 & 26.78 & 81.57 \\
\xmark~~$\mathcal{L}_{\text{text-cond.\ alignment}}$   & 30.22 & 27.55 & 83.56 \\
\xmark~~$\mathcal{L}_{\text{vision\_commitment}}$         & 32.67 & 28.21 & 84.23 \\
\xmark~~$\mathcal{L}_{\text{text-codebook\ alignment}}$ & 31.11 & 27.07 & 83.91 \\
\cmark~~\textbf{All losses}                   & \textbf{33.17} & \textbf{29.11} & \textbf{86.03} \\
\hline
\end{tabular}

}
\end{table}}

\newcommand{\Tablesupvq}{
\begin{table*}[t!]
\caption{\textbf{Ablation study of different quantization techniques in \modelnamenc.} Each method is specified by its quantization type, codebook vocabulary size, and embedding dimensionality.}
\label{tab:vqablation}
\centering
\resizebox{0.99\textwidth}{!}{
\begin{tabular}{lccccccccc}
\hline
\multirow{2}{*}{\textbf{Quantization}} 
& \multirow{2}{*}{\textbf{Vocab}} 
& \multirow{2}{*}{\textbf{Dim}} 
& \multicolumn{3}{c}{\textbf{COCO-Val}} 
& \multicolumn{3}{c}{\textbf{WebVid-10M}} 
& \multirow{2}{*}{\textbf{Inf. Time}} \\
\cline{4-9}
& & 
& PSNR ($\uparrow$) & SSIM ($\uparrow$) & LPIPS ($\downarrow$) 
& PSNR ($\uparrow$) & SSIM ($\uparrow$) & LPIPS ($\downarrow$) & \\
\hline
VQ~\cite{van2017neural}                & 4096  & 256 & 31.45 & 0.825 & 0.093 & 32.91 & 0.838 & 0.092 & 409 \\
GVQ~\cite{jang2017categorical}               & 4096  & 256 & 32.25 & 0.836 & 0.089 & 33.34 & 0.842 & 0.089 & 438 \\
LFQ~\cite{yu2024language}              & 32800 & 16  & 34.22 & 0.842 & 0.084 & 33.92 & 0.855 & 0.085 & 419 \\
RVQ~\cite{lee2022autoregressive}             & 8000  & 512 & 33.92 & 0.849 & 0.078 & 34.22 & 0.865 & 0.079 & 489 \\
LaPQ            & 8000  & 512 & 34.45 & 0.855 & 0.073 & 34.98 & 0.871 & 0.076 & 503 \\
RVQ~\cite{lee2022autoregressive}       & 32800 & 16  & 34.78 & 0.869 & 0.076 & 35.27 & 0.879 & 0.074 & 488 \\
\hline
\rowcolor{purple!5} \textbf{LaPQ (Ours)} & 48000 & 16  & \textbf{35.72} & \textbf{0.879} & \textbf{0.066} 
                          & \textbf{36.05} & \textbf{0.885} & \textbf{0.071} & \textbf{492} \\
\hline
\end{tabular}
}
\end{table*}
}

\newcommand{\Tablesupclass}{
\begin{table}[t!]
\caption{\textbf{Class-guided video generation.}}
\label{tab:ucf_class_guided}
\centering
\setlength{\tabcolsep}{4pt}
\small
\resizebox{\linewidth}{!}{%
\begin{tabular}{lcccc}
\toprule
\textbf{Tokenizer} & \textbf{Type} & \textbf{\#Tokens} & \textbf{\#Params (Gen.)} & gFVD ($\downarrow$) \\
\midrule
MAGVIT~\cite{yu2023magvit}            & AR   & 1024 & 306M & 265 \\
MAGVIT-V2~\cite{yu2024language}             & AR   & 1280 & 307M & 109 \\
MAGVIT~\cite{yu2023magvit}            & MLLM & 1024 & 306M & 76  \\
MAGVIT-V2~\cite{yu2024language}             & MLLM & 1280 & 307M & 58  \\
LARP-L~\cite{wang2025larp}            & AR   & 1024 & 632M & 57  \\
CogVideoX~\cite{yang2025cogvideox}    & AR   & 6800 & 9.4B & 626 \\
TATS~\cite{ge2022long}                & AR   & 4096 & 321M & 332 \\
Video-LaVIT~\cite{jin2024video}         & AR   & 512  & 7B   & 280 \\
OmniTok~\cite{wang2024omnitokenizer}  & AR   & 5120 & 650M & 191 \\
LARP-L~\cite{wang2025larp}            & AR   & 1024 & 632M & 99  \\
SweetTok~\cite{tan2025sweettok}       & AR   & 1280 & 1.9B & 65  \\
\hline
\rowcolor{purple!5} \modelname\textbf{(Ours)}                                  & AR   & 1024 & 2.3B & \textbf{51} \\
\bottomrule
\end{tabular}
}
\end{table}}

%% file: sections/01_introduction.tex
\section{Introduction}
\label{sec:intro}
In recent years, multimodal video generation has gained significant attention~\cite{kondratyuk2023videopoet, lin2024open, ruan2023mm}. Text-to-video models such as VideoGPT~\cite{yan2021videogpt}, CogVideoX~\cite{yang2025cogvideox}, and OmniGen2~\cite{wu2025omnigen2} are particularly prominent. Most of these models are built on latent diffusion frameworks~\cite{blattmann2023stable, chen2023videocrafter1}, which generate videos in a compact latent space rather than directly modeling high-dimensional pixel distributions, improving efficiency and reducing computational cost. 
Variational Autoencoders (VAEs) are central to this setup. In particular, discrete VAEs~\cite{van2017neural, esser2021taming} have proven especially effective, as their learned codebooks quantize the latent space into discrete tokens, enabling scalable and high-quality video synthesis.\looseness-1

Although discrete VAEs offer strong compression and generation capabilities, their codebooks are typically learned solely from visual data~\cite{van2017neural, yu2024language}. This limits performance on downstream tasks such as text-to-video generation or video understanding, due to the semantic gap between textual input and visual representation. Bridging this gap during downstream training increases convergence time and resource demands. 
Recent works have integrated text supervision directly within VAE architectures~\cite{qu2025tokenflow, zha2025language, guotao2024lg, lin2025toklip, tan2025sweettok}.\looseness-1

However, despite these advances, current methods have few major limitations: (1) They largely capture semantics at a single scale, \ie only after obtaining latent representations from the encoder, which limits their ability to leverage the hierarchical nature of VAEs that model features from low-level spatial details to high-level semantics~\cite{vahdat2020nvae, takida2024hq}, leaving potential for more fine-grained text-video alignment. (2) They typically employ small codebooks (4K–8K tokens), which are sufficient for basic visual patterns but limit the representational capacity of both visual and textual modalities~\cite{yu2024language}. These smaller codebooks hinder effective cross-modal alignment and constrain the expressiveness of text-conditioned video generation models.
(3) Shallow, single-site text alignment causes semantic drift. Most existing methods inject language either \emph{globally} through sequence-level contrastive objectives~\cite{guotao2024lg,lin2025toklip} or \emph{locally} via token-level codebook distillation~\cite{zha2025language}, during codebook learning only. As a result, the learned representations exhibit semantic drift and temporal inconsistency, where local visual tokens fail to remain aligned with global textual intent.

To address the aforementioned limitations, we introduce \modelname, a video tokenizer that leverages a novel \uline{L}anguage \uline{a}ligned \uline{P}yramidal \uline{Q}uantization (\modelnamecpa) to hierarchically encode coarse-to-fine video features using an expressive codebook of large vocabulary.
To bridge visual and text semantics, we introduce a dual semantic alignment strategy that jointly aligns text and video representations via multi-scale quantization and autoregressive refinement. Empirically, \modelnamenc achieves SoTA performance across video generation and various video understanding tasks. 
\modelnamenc surpasses the best prior VAE baseline by +5.75 mAP on temporal action localization, +2.82 on videoQA, and up to +9.16 on video classification. Notably, \modelnamenc is the first VAE to reach SoTA zero-shot video semantic segmentation, outperforming zero-shot and unsupervised methods by up to +10 and +7.0 mAP, respectively.
Fig.~\ref{fig:activation_map} illustrates \modelnamenc's interpretable text-guided cross-modal attention.\looseness-1

\noindent \textbf{Contributions:} In summary, our contributions are:
\begin{itemize}[itemsep=0.5ex, parsep=0pt, topsep=-2.3pt, leftmargin=0.5cm]
 \item We introduce \textbf{\modelnamenc}, a multi-scale semantically aligned Video VAE that couples spatiotemporal quantization with dual semantic alignment, enabling coarse-to-fine understanding and efficient video generation.\looseness-1

\item \modelnamenc leverages \textbf{\modelnamecpa}, a novel language-aligned pyramidal quantization framework, designed to hierarchically encode multi-scale video representations through lateral encoder connections at each stage. Our design enables efficient use of a large $\sim$48K token vocabulary, with up to 95\% codebook utilization.\looseness-1

\item We propose a \textbf{dual semantic alignment} strategy that injects text-conditioned priors at every \modelnamecpa level (\emph{local} alignment) and refines them with an autoregressive objective over the sequence of quantized tokens (\emph{global} alignment). This jointly enforces token-level grounding and sequence-level (temporal and relational) coherence, preventing semantic drift across scales and time.\looseness-1

\item We further introduce a \textbf{hierarchical semantic codebook loss} that ties a shared binary codebook to text embeddings and preserves semantic consistency across pyramid levels through stage-wise KL regularization.\looseness-1
\end{itemize}

\modelnamenc achieves SoTA reconstruction fidelity and downstream performance across 10 diverse video benchmarks, scaling to 4K and 8K resolutions. For example, \modelnamenc is, to our knowledge, the first discrete quantized VAE to demonstrate zero-shot text-guided video segmentation, with up to 2$\times$ improvement in mAP on OVIS over strong baselines.\looseness-1

%% file: sections/02_related_work.tex
\section{Related Work}
\label{sec:relatedwork}
\noindent \textbf{Visual Quantized VAEs for Video.}  VAEs have become a cornerstone in video generation~\cite{kondratyuk2023videopoet, lin2024open, singer2023make} and downstream tasks such as text-to-video~\cite{wu2023tune, yang2025cogvideox, zhang2025show} and video understanding~\cite{wang2024omnivid, lin2019tsm, bertasius2021space}, enabling efficient sampling and scalable generation by learning compact latent spaces. A key advance is discrete latent VAEs, introduced in VQ-VAE~\cite{van2017neural}. Unlike continuous VAEs, which map inputs to Gaussian spaces, VQ-VAEs tokenize features into a learnable codebook. This yields structured, non-redundant representations suitable for sequence modeling and scalable training. VQ-GAN~\cite{esser2021taming} adds adversarial training to reduce blur, while ViT-VQGAN~\cite{yu2022vector} replaces CNNs with Vision Transformers~\cite{dosovitskiy2021image} for long-range modeling.\looseness-1

These models have been adapted to video through spatiotemporal extensions. VideoGPT~\cite{yan2021videogpt} introduces a 3D VQ-VAE by replacing 2D CNNs with 3D convolutions to maintain temporal coherence. MAGVITv2~\cite{yu2024language, luo2024open} further improves fidelity via Lookup-Free Quantization (LFQ), enabling substantially larger codebooks with efficient training. More recent tokenizers extend this direction. For instance, OmniTokenizer~\cite{wang2024omnitokenizer} unifies image–video tokenization via a spatial–temporal decoupled design, LARP~\cite{wang2025larp} introduces an autoregressive-friendly latent prior, and 3D-MBQ-VAE~\cite{susladkar2025motionaura} improves efficiency and temporal consistency with mobile inverted blocks and full-frame masking. However, these approaches remain limited in capturing fine-grained spatial details because quantization is performed at a fixed spatial scale.\looseness-1

\noindent \textbf{Text Quantization in VAEs.} While vanilla VQ-VAEs effectively compress visual information, they inherently lack cross-modal alignment, limiting their applicability to tasks requiring semantic consistency, such as text-to-video generation and VideoQA. Early methods like Frozen~\cite{tsimpoukelli2021multimodal} attempted alignment using frozen language models but required large paired datasets. To address this, several image generation methods such as TokLIP~\cite{lin2025toklip}, LG-VQ~\cite{guotao2024lg}, and TokenFlow~\cite{qu2025tokenflow} have proposed unified quantization strategies that embed visual data into language-informed spaces in VAEs.\looseness-1

\FigIntro

Despite significant progress in image generation, only a few methods extend such strategies to video VAEs. For example, VideoVAE+~\cite{xing2024large} integrates captions into the quantization stage using frozen BERT embeddings to align spatiotemporal latents with language semantics. SweetTok~\cite{tan2025sweettok} introduces a motion-aware language codebook with decoupled spatial-temporal tokenization for compact, semantically rich video representations. However, these models typically align semantics at a single resolution, overlooking the hierarchical, coarse-to-fine structure of visual understanding. In contrast, we propose \modelnamenc, a language-enhanced video VAE for video generation and understanding, that introduces multi-scale semantic alignment within discrete latent spaces, enabling joint reasoning over both global context and fine-grained details.\looseness-1

%% file: sections/03_method.tex
\FigArch

\section{Method}\label{sec:method}
\subsection{Problem Definition}
Given an input video $\mathbf{X} \in \mathbb{R}^{C \times T \times H \times W}$ with $T$ frames, $H \times W$ spatial resolution, and $C$ channels. The goal is to learn a compact latent representation that preserves both spatiotemporal fidelity and semantic correspondence with conditioning text embedding $\mathbf{e_t}$. 
The input video is masked ($\tilde{\mathbf{X}}$) and encoded by $\mathcal{E}n$ to produce latent features $\mathbf{Z}\!=\!\mathcal{E}n(\tilde{\mathbf{X}})$, where $\mathbf{Z} \in \mathbb{R}^{T' \times H' \times W' \times d}$ and $T'\!=\!T/f+1$, $H'\!=\!H/2f$, $W'\!=\!W/2f$ denote compressed temporal and spatial dimensions with compression factor $f$ with $d$ dimensions. Encoded features are discretized through a text-conditioned quantization process $\mathbf{q}\!=\!\mathcal{Q}(\mathbf{Z}, \mathbf{e_t})$, and the decoder reconstructs the video as $\hat{\mathbf{X}}\!=\!\mathcal{D}e(\mathbf{q})$. 
This yields a text-guided video autoencoding objective that learns compact representations for efficient downstream generative modeling.\looseness-1

\subsection{\modelname Architecture}

\subsubsection{Language-aligned Pyramidal Quantization}
Videos exhibit rich structure across multiple spatial and temporal scales, but single-scale quantization methods~\cite{van2017neural, esser2021taming} tend to overfit global patterns or miss fine-grained details. While larger codebooks can improve generation quality~\cite{yu2023magvit}, they introduce prohibitive memory and compute costs. 
To address this, we introduce Language-aligned Pyramidal Quantization (\modelnamecpa), a novel framework that discretizes features at multiple encoder depths via lateral connections, capturing global semantics from deeper layers and local details from shallower ones without high-dimensional codebooks.\looseness-1

In addition, \modelnamecpa aligns both the quantization assignments and codewords with text embeddings, ensuring that each discrete token is informative of the associated language description. This language alignment is essential for text-conditioned video generation and zero-shot video understanding, as it produces a discrete video token space that is natively compatible with multimodal models.

Formally, the encoder $\mathcal{E}n$ processes a masked input video through $L$ hierarchical stages to extract multi-scale spatiotemporal representations 
$\mathbf{F}^{(l)}\!=\!\mathcal{E}n(\mathbf{F}^{(l-1)})$, with $\mathbf{F}^{(0)}\!=\!\tilde{\mathbf{X}}$,
where $\mathbf{F}^{(l)} \in \mathbb{R}^{C_l \times T_l \times H_l \times W_l}$ denotes the feature map at the $l^\text{th}$ stage of the encoder, with progressive downsampling along spatial and temporal dimensions.
To capture both fine and coarse spatiotemporal details, we quantize $\mathbf{Z}$ in a pyramidal manner across encoder depths. Specifically, at each stage $l$, we introduce a Quantization Block $\mathcal{Q}_l$ that receives the current encoder feature $\mathbf{F}^{(l)}$, the previous quantized representation $\mathbf{q}^{(l-1)}$, and the query text embedding $\mathbf{e_t}$ for semantic alignment, producing a new semantically aligned quantized representation $\mathbf{q}^{(l)}$ at stage $l$:
\begin{equation}
\setlength{\abovedisplayskip}{6pt}
\setlength{\belowdisplayskip}{6pt}
\mathbf{q}^{(l)}=\mathcal{Q}_l(\mathbf{q}^{(l-1)}, \mathbf{F}^{(l)}, \mathbf{e_t})
\end{equation}
This hierarchical process enables progressive semantic alignment across $L$ stages. Fig.~\ref{fig:PVQ-VAE-arch} illustrates the whole architecture of \modelnamenc. The internal architecture of $\mathcal{Q}$ is detailed in the following subsection. 

\subsubsection{Dual Semantic Alignment}
We propose a novel alignment strategy to ensure that quantized video tokens remain both locally faithful to visual structure and globally consistent with textual semantics.\looseness-1

\noindent \textbf{\ding{182} Multi-scale Semantic Alignment in Quantization Blocks (local):} In each Quantization Block $\mathcal{Q}_l$ of \modelnamecpa, semantic discretization is performed at a specific encoder depth by integrating visual and text information, capturing semantics across multiple scales. Given encoder features $\mathbf{F}^{(l)}$, we incorporate lateral connections to retain spatial and temporal locality. Semantic context is introduced by attending to the text embedding $\mathbf{e_t}$, extracted from a pretrained VLM, via multi-head self-attention, enabling language-guided modulation of visual features. 
The attended visual–text features are subsequently fused through projection layers, yielding modality-aligned representations suitable for quantization.

To discretize these representations efficiently, we adopt Lookup-Free Quantization (LFQ)~\cite{yu2024language}, which replaces the conventional learned codebook $\mathbf{C}\!\in\!\mathbb{R}^{K \times d}$ with compact binary codewords $\mathbf{C}_v\!=\!\{-1, 1\}^{\log_2 K}$. This eliminates high-dimensional embedding lookups and significantly reduces memory overhead, enabling efficient scaling to a large vocabulary.
The binary codebook $\mathbf{C}_v$ is shared across all $\mathcal{Q}_l$ quantization blocks, ensuring consistency across pyramid levels while minimizing parameter growth.
The codebook is used only during training to compute alignment losses and guide structure. During inference, quantization operates without lookups, preserving the efficiency of LFQ. To jointly optimize quantization and semantic alignment, we introduce a \emph{hierarchical semantic codebook loss}:
\begin{equation}\label{eq:codebook}
\setlength{\abovedisplayskip}{8pt}
\setlength{\belowdisplayskip}{8pt}
\scalebox{0.90}{$
\begin{aligned}
\mathcal{L}_{\text{codebook}}\!=\!\sum_{l=1}^{L} \Bigg[
\underbrace{\left\| \mathbf{q}^{(l)} - \text{sg}(\mathbf{C}_v) \right\|^2}_{\text{ vision-commitment}} + \underbrace{\mathbb{E}\!\left[ -\mathbf{q}^{(l)} \log \mathbf{q}^{(l)} \right]}_{\text{entropy regularization}} \\
+ \underbrace{\mathrm{D_{KL}}\!\left( \mathbf{q}^{(l)} \,\|\, \mathbf{q}^{(l-1)} \right)}_{\text{hierarchical consistency}} 
+ \underbrace{\mathbb{E}_{\mathbf{q}_i \in \mathbf{q}^{(l)}} \left[ \mathrm{D_{KL}}\!\left(\mathbf{q}_i \,\|\, \text{sg}(\mathbf{e_t}) \right)\right]}_{\text{text-conditioned alignment}} \\
 + \underbrace{\mathbb{E}_{\mathbf{c} \in \mathbf{C}_{v}} \mathrm{D_{KL}}\!\left( \mathbf{c} \,\|\, \text{sg}(\mathbf{e_t}) \right)}_{\text{text–codebook alignment}} 
\Bigg].
\end{aligned}
$}
\end{equation}
Here, $\text{sg}(\cdot)$ denotes the stop-gradient operator. The first term encourages vision-commitment by pulling $\mathbf{q}^{(l)}$ toward the binary code vectors $\mathbf{C}_v$, while entropy regularization sharpens the assignments toward near one-hot distributions. The hierarchical KL term enforces hierarchical consistency across quantization levels. The remaining KL terms introduce semantic structure through text-conditioned alignment of assignments and text–codebook alignment of the LFQ codebook.
Together, these terms enable stable multi-scale quantization with strong cross-modal coherence. Fig.~\ref{fig:pca} illustrates this refinement, with deeper stages producing clearer semantic structure. For example, in the first row, later stages reveal more distinct separation of road lanes, vehicles, and background elements.\looseness-1

\FigPCA

\noindent \textbf{\ding{183} Autoregressive Semantic Alignment (global):}
To enforce global semantic consistency between language and discrete latents, we introduce an autoregressive alignment objective over the quantized token sequence. Given a text query $t$, we obtain its embedding $\mathbf{e_t}=\text{VLM}(t)$ and extract discrete tokens from each quantization block using the shared codebook $\mathbf{C}_v$. Tokens from all levels are concatenated with separator tokens $\langle \text{Q-SEP} \rangle$ to retain hierarchical structure, and a start-of-image token $\langle \text{SOI} \rangle$ is prepended after the text. The resulting sequence is fed into the VLM decoder, which autoregressively predicts each visual token conditioned on the text and preceding tokens: $\mathcal{L}_{\text{AR}}\!=\!-\sum_{l=1}^{L} \log p(\mathbf{q}^{(l)} \mid \mathbf{q}^{(<l)}, \mathbf{e_t})$. 
By making visual tokens predictable from the text prefix, this objective encourages the shared codebook to encode globally consistent, language-aligned semantics. The separator tokens retain hierarchical structure while enabling unified sequential modeling, improving both reconstruction quality and latent-space controllability.\looseness-1

\TableRecon

\subsubsection{Pretrained VAE Encoder and LoRA}
\modelnamenc leverages a pretrained video VAE, keeping both encoder $\mathcal{E}n$ and decoder $\mathcal{D}e$ frozen to preserve high-fidelity reconstruction and focus learning on multi-scale semantic alignment. To enable efficient adaptation to high-resolution inputs, we insert LoRA modules~\cite{hu2022lora} into encoder blocks, enabling lightweight feature modulation without modifying pretrained weights.
Text-conditioned supervision can cause latent drift from the pretrained visual manifold. To stabilize adaptation, we add a drift-regularization term that anchors adapted features to a frozen large-scale reference encoder ${En}$: $\mathcal{L}_{\text{drift}}\!=\!\mathrm{D_{KL}}\left( \mathcal{E}n(\tilde{\mathbf{X}}) \middle|\middle| {En}(\tilde{\mathbf{X}}) \right)$, This stabilizes training by preserving alignment with the original visual prior while allowing semantically guided updates.

\subsubsection{Total Objective and Regularization.}
\modelnamenc is trained with a composite loss balancing reconstruction quality, semantic alignment, and quantization consistency $\lambda_{\text{recon}} \mathcal{L}_{\text{recon}}\!+\!\lambda_{\text{codebook}} \mathcal{L}_{\text{codebook}}\!+\!\lambda_{\text{AR}} \mathcal{L}_{\text{AR}}\!+\!\lambda_{\text{drift}} \mathcal{L}_{\text{drift}}$, where  $\lambda_{\text{recon}}$, $\lambda_{\text{codebook}}$, $\lambda_{\text{AR}}$, and $\lambda_{\text{drift}}$ coefficients. The reconstruction loss combines pixel-level and perceptual terms, $\mathcal{L}_{\text{recon}}\!=\!\mathcal{L}_{\text{SSIM}} + \mathcal{L}_{\text{L1}} + \mathcal{L}_{\text{LPIPS}}$, while $\mathcal{L}_{\text{codebook}}$ enforces multi-scale semantic alignment, $\mathcal{L}_{\text{drift}}$ ensures that low-rank adapters do not drift using alignment, and $\mathcal{L}_{\text{AR}}$ promotes autoregressive alignment with the query text.\looseness-1 

%% file: sections/04_experiments.tex
\section{Experiments}\label{sec:experiments}
We comprehensively evaluate \modelnamenc on frame reconstruction, text-o-video generation, and a diverse set of multimodal understanding tasks, including zero-shot segmentation, temporal action localization, general video understanding, and text-to-video generation. Evaluations are conducted across 10 real-world benchmarks, such as WebVid-10M~\cite{bain2021frozen}, YouTube-VIS 2021~\cite{yang20213rd}, MVBench~\cite{li2024mvbench}, \etc

\modelnamenc is trained on a large-scale subset of Droplet-10M~\cite{zhang2025dropletvideo} comprising HD videos, augmented with additional HD samples from OpenVid-1M~\cite{nan2025openvid} and ultra-high-resolution (4K/8K) videos with reconstructed captions from UltraVideo~\cite{xue2025ultravideo}. Additional implementation and experimental setup details are provided in the supplementary material.
\FigFrameRecon
\FigTSNE

\subsection{Video Generation Tasks}
\noindent \textbf{Frame Reconstruction.} As shown in Table \ref{tab:quantitative_comparison}, \modelnamenc achieves the best frame reconstruction quality on both WebVid-10M~\cite{bain2021frozen} and COCO-Val~\cite{lin2014microsoft}, surpassing all prior semantic and non-semantic video VAEs. Compared to SweetTok~\cite{tan2025sweettok} and TokLIP~\cite{lin2025toklip}, which also incorporate semantic alignment, \modelnamenc achieves 10.51\% and 14.19\% higher PSNR, and 51.62\% and 56.57\% lower LPIPS, respectively.
SweetTok decouples spatial and temporal tokens but processes them independently, hindering global semantic consistency, while TokLIP enriches visual tokens with CLIP-level~\cite{radford2021learning} semantics but lacks temporal modeling. \modelnamenc overcomes both limitations by combining fine-grained, text-guided quantization at each \modelnamecpa level with a global autoregressive prior that enforces temporal coherence. 
Furthermore, SoTA non-semantic VAEs such as 3D-MBQ-VAE~\cite{susladkar2025motionaura}, CogVideoX~\cite{yang2025cogvideox}, and LARP~\cite{wang2025larp} are also outperformed, highlighting \modelnamenc’s ability to capture text semantics while maintaining high fidelity. 

These trends are clearly reflected in the qualitative results. As shown in Fig.~\ref{fig:video_reconstruction}, \modelnamenc reconstructs legible text in the street scene, crisp leaf textures in the ramen and plant examples, and fine facial structures on the polar bear, whereas all baselines exhibit noticeable blurring or distortion.
The t-SNE visualization in Fig.~\ref{fig:t_sne_plots} further reveals that \modelnamenc{}’s latent space forms compact, well-separated clusters corresponding to coherent semantic categories, evidencing effective multi-scale semantic organization.
\Tabletextv

\noindent \textbf{Text-2-Video (T2V) Generation.} Table~\ref{tab:t2v} and Fig.~\ref{fig:video_generation} show that substituting the native VAEs in MotionAura~\cite{susladkar2025motionaura}, MAGVITv2~\cite{yu2024language, luo2024open}, and OmniGenV2~\cite{wu2025omnigen2} with \modelnamenc{} consistently improves perceptual fidelity, texture sharpness, and text–video semantic alignment. Quantitatively, \modelnamenc{} reduces FVD by 9–22 points and increases TC by 20–27 points across all backbones. Qualitatively (shown in Fig.~\ref{fig:video_generation}), \modelnamenc{} recovers details such as clearer facial structure, and more coherent structure like robotic hand geometry in the OmniGenV2 example.

\FigTV
\subsection{Video Understanding Tasks}
\noindent \textbf{Video Segmentation.}
As shown in Table~\ref{tab:vis_results}, \modelnamenc demonstrates strong zero-shot performance on YouTube-VIS 2021~\cite{yang20213rd} and OVIS~\cite{qi2022occluded}. Compared to the zero-shot SoTA OmniTokenizer~\cite{lin2025toklip}, which lacks explicit text-semantic supervision, \modelnamenc achieves 68.8\% and 30.2\% relative improvements in mAP and Jaccard on YouTube-VIS 2021, and remarkable gains of 217.9\% and 48.6\% on OVIS, respectively. These results underscore the effectiveness of our semantically aligned video representation in enabling robust zero-shot generalization. To the best of our knowledge, \modelnamenc is the first demonstration of zero-shot video semantic segmentation using a language-aligned discrete VAE.\looseness-1

\TableVidSeg
\FigSeg
Compared to unsupervised baselines like VideoCutLER~\cite{wang2024videocutler} and UVIS~\cite{huang2024uvis}, which suffer from motion ambiguity and inconsistent grouping, \modelnamenc{}'s multi-scale text-conditioned quantization achieves coherent segmentation with enhanced spatial–temporal consistency.
Qualitative results in Fig.~\ref{fig:video_seg} further validate these findings. \modelnamenc accurately segments complex multi-object scenes (\eg players, soccer ball, and field) with precise boundaries and strong semantic correspondence between textual and visual cues.
\TableVidAct

\noindent \textbf{Video Action Localization.}
As shown in Table~\ref{tab:action_localization}, \modelnamenc achieves the best zero-shot performance on THUMOS14 and ActivityNet, outperforming the previous zero-shot SoTA LARP~\cite{wang2025larp} by +5.75 mAP and +3.58 mAP, respectively. Although LARP and SweetTok~\cite{tan2025sweettok} incorporate semantics, their alignment remains limited. For instance, SweetTok separates spatial and temporal streams, and LARP lacks explicit text-conditioned supervision. In contrast, \modelnamenc combines multi-scale text-guided quantization with a global autoregressive prior, enabling fine-grained temporal reasoning and stronger cross-modal consistency.\looseness-1

These advantages are evident in Fig.~\ref{fig:video_action}, where \modelnamenc more accurately localizes the baseball bat swing action than others. This design also allows \modelnamenc to surpass supervised approaches such as STALE and DeTAL~\cite{li2024detal}, highlighting the strength of semantically aligned discrete latents for action localization.\looseness-1

\FigAction
\TableVidUnified

\noindent \textbf{General Video Understanding and Classification.}
As shown in Table~\ref{tab:unified_video_results}, \modelnamenc achieves SoTA performance on both the MVBench~\cite{li2024mvbench} and Kinetics benchmarks~\cite{kay2017kinetics}. Specifically, our model attains an overall accuracy of 86.03\% across diverse video understanding tasks on MVBench. Furthermore, it demonstrates substantial improvements of 13.22\%, 12.54\%, and 10.75\% over LARP~\cite{wang2025larp} on the Kinetics-400, -600, and -700 benchmarks, respectively. \modelnamenc surpass prior VAE-based and large-scale non-VAE foundation models, including InternVL3-78B~\cite{zhu2025internvl3}, Qwen2.5-VL-7B~\cite{bai2025qwen2}, and VideoPrism-g~\cite{zhao2024videoprism}. This performance gain stems from \modelnamenc{}’s multi-scale text-guided quantization, which offers stronger semantic grounding and temporal coherence. By contrast, although SweetTok~\cite{tan2025sweettok} and LARP~\cite{wang2025larp} incorporate semantic cues, their limited text–video alignment constrains temporal reasoning.
Within VAE-based methods, \modelnamenc further outperforms VILA-U~\cite{wu2025vila}, OmniTokenizer~\cite{wang2024omnitokenizer}, and VideoVAE+~\cite{xing2024large}, demonstrating the effectiveness of language-conditioned quantized representations. The consistent gains across understanding and classification tasks highlight \modelnamenc{}’s capability as a unified, semantically grounded video representation model with robust zero-shot generalization.\looseness-1

\FigCodebook
\subsection{Ablations}
Fig.~\ref{fig:codebook_vs_loss} and Table~\ref{tab:ablation_study} present ablations on key \modelnamenc{} components, including codebook size, loss configuration, the presence of pyramidal and recurrent quantization modules, the number of quantization blocks, and variations in the multimodal encoder or pretrained video VAE.\looseness-1

\noindent \textbf{Codebook Size.} As shown in Fig.~\ref{fig:codebook_vs_loss}, increasing codebook size and dimensionality consistently improves reconstruction and perceptual quality. Larger and higher-dimensional codebooks provide a richer latent space, enabling finer feature representation and reducing quantization error. However, performance gains saturate beyond 80K vocab size, suggesting a trade-off between model capacity and efficiency.

\noindent \textbf{Component Ablation.} Removing \modelnamecpa leads to the largest degradation across all metrics, highlighting the importance of hierarchical language-aligned quantization. Excluding text guidance noticeably weakens semantic grounding, reducing both fidelity and perceptual quality. Dropping the pyramidal-Q design similarly harms performance, confirming the effectiveness of multi-scale quantization.

\noindent \textbf{Quantization-Blocks.} Performance improves consistently as the number of $\mathcal{Q}$ blocks increases, with four blocks yielding the best results. This shows that deeper quantization hierarchies enhance semantic representation and reconstruction fidelity by capturing both coarse and fine visual details.
\TableAlbs

\noindent \textbf{Loss Functions.} Excluding $\mathcal{L}_{\text{drift}}$ or $\mathcal{L}_{\text{AR}}$ weakens semantic coherence and structure preservation, while removing both leads to the largest performance drop. This confirms that feature-level alignment and variance regularization jointly stabilize semantic learning and reconstruction.\looseness-1

\noindent \textbf{Codebook Loss.}
Without vision-commitment, assignments become unstable, whereas without text-conditioned alignment, semantic guidance weakens. Removing text–codebook alignment disrupts global semantic structure, producing the largest degradation. This demonstrates all three terms are crucial for stable and semantically coherent quantization.

\noindent \textbf{Multimodal Models.} Using different vision-language encoders demonstrates the generality of \modelnamenc. Qwen2.5-VL achieves the best overall performance, while LLaMA-3 and Gemma-3 variants maintain competitive results.\looseness-1

\noindent \textbf{Pretrained VAEs.} Substituting the pretrained backbone shows that \modelnamenc maintains consistent improvements across encoders. The Wan 2.2 VAE~\cite{wan2025wan} (default) delivers the best results, but strong performance with 3DMBQ-VAE, CogVideoX, and Mochi-VAE confirms the robustness and transferability of the proposed semantic quantization design.

%% file: sections/05_conclusion.tex
\section{Conclusion}\label{sec:conclusion}

\noindent We introduce \modelnamenc{}, a language-aligned pyramidal video tokenizer that performs multi-scale vector quantization with a shared large binary codebook. Our dual semantic alignment couples text-conditioned, per-level quantization with a global autoregressive objective, producing semantically consistent discrete latents. \modelnamenc{} delivers state-of-the-art 4K/8K reconstruction and strong zero-shot transfer on video segmentation, temporal action localization, VideoQA, and video classification. Compatibility studies show consistent gains across vision–language encoders and diverse VAE backbones. Ablations confirm the necessity of the pyramidal path and RVQ, the benefit of four quantization blocks, and the contributions of the autoregressive and drift terms, as well as codebook alignment losses. These results establish \modelnamenc{} as a practical, general-purpose tokenizer for modern video–language systems.\looseness-1

%% file: sections/06_appendix.tex
\section{Theoretical Analysis of Language-aligned Pyramidal Quantization}
\label{app:theory}
We analyze the behavior of the Language-aligned Pyramidal Quantization (LaPQ) objective and the conditions under which the model avoids posterior collapse.
Let $\theta$ denote all trainable parameters.
LaPQ is composed of smooth losses (reconstruction, codebook, autoregressive, and drift), each of which is an expectation over the training distribution $\mathcal{D}$ of video--text pairs $(\mathbf{X},t)$, \ie
$\mathcal{L}_{\theta}\!=\!\mathbb{E}_{(\mathbf{X},t)\sim\mathcal{D}} \big[\ell(\theta;\mathbf{X},t)\big].$  All LaPQ modules (LoRA layers, AR head, LFQ quantizers, \etc) use differentiable operations, so $\mathcal{L}_{\theta}$ is a smooth, lower-bounded deep-network objective.

\noindent \textbf{Why LaPQ Mitigates Posterior Collapse.}
At LaPQ level $l$, let $\mathbf{q}^{(l)}\!=\!\mathcal{Q}_l(\mathbf{q}^{(l-1)}, \mathbf{F}^{(l)}, \mathbf{e_t})$ be the (soft) assignment distribution, where $\mathbf{F}^{(l)}$ are encoder features and $\mathbf{e_t}$ is the text embedding extracted from the text $t$.
LaPQ at level $l$ is \emph{collapsed} if there exists a fixed distribution $\bar{\mathbf{q}}^{(l)}$ s.t. $ \mathbf{q}^{(l)}\!\equiv\!\bar{\mathbf{q}}^{(l)} \text{ for all } (\mathbf{X},t)\sim\mathcal{D}.$
A \emph{fully collapsed} LaPQ posterior satisfies this for all
levels $l=1,\dots,L$. 
Assume the following conditions:\looseness-1
\begin{enumerate}[itemsep=0ex, parsep=0pt, topsep=-2.3pt, leftmargin=0.5cm]
    \item \textbf{Data non-degeneracy:} The data distribution $\mathcal{D}$ is non-degenerate, \ie
    there exist $(\mathbf{X},t)$ and $(\mathbf{X}',t')$ 
    s.t. the corresponding optimal reconstruction outputs under reconstruction loss $\mathcal{L}_{\mathrm{recon}}$ differ.
    \item \textbf{Decoder injectivity:} For any two distinct latent code sequences $\mathbf{q} \neq \mathbf{q}'$, the decoder produces distinct reconstructions $\mathcal{D}e(\mathbf{q}) \neq \mathcal{D}e(\mathbf{q}')$.
    \item \textbf{Model expressiveness:} For any measurable mapping
    $(\mathbf{X},t)\mapsto \mathbf{q}^{(1:L)}$, realizable via encoder features $\mathbf{F}^{(l)}$ and text embedding $\mathbf{e_t}$, there exists a parameter $\theta$ that realizes it to arbitrary precision.
\end{enumerate}
\begin{proposition}[Non-optimality of Collapsed LaPQ Posteriors]
\label{prop:noncollapse}
Any fully collapsed LaPQ posterior $ \mathbf{q}^{(l)}\equiv \bar{\mathbf{q}}^{(l)}$ cannot minimize the LaPQ objective.
\end{proposition}

\begin{proof}
\vspace{-0.4cm}
Consider any parameter vector $\theta_{\mathrm{c}}$ that yields a fully
collapsed posterior. Then, by definition, every quantizer output distribution $\mathbf{q}^{(l)}$ is constant across inputs, hence the decoder input (the discrete code sequence $\mathbf{q}_c$) is also constant. Hence, all reconstructions are equal to $\hat{\mathbf{X}}_{\mathrm{c}}=\mathcal{D}e(\mathbf{q}_\mathrm{c})$. Then, the reconstruction loss $\mathcal{L}_{\mathrm{recon}}(\theta_{\mathrm{c}})$ is the expected reconstruction loss under a \emph{constant} prediction, \ie $\mathcal{L}_{\text{recon}}(\theta_{\mathrm{c}})
= \mathbb{E}_{(\mathbf{X},t)\sim\mathcal{D}}\big[\ell_{\text{recon}}(\hat{\mathbf{X}}_{\mathrm{c}},\mathbf{X})\big]$.
By the non-degeneracy of $\mathcal{D}$ and standard properties of
$L_1$/SSIM/LPIPS reconstructions, there exists a non-constant mapping
$\mathbf{X}\mapsto \hat{\mathbf{X}}(\mathbf{X})$ that achieves strictly lower expected reconstruction error than any constant
prediction. Using the model expressiveness assumption, we can approximate such a mapping with some parameter vector $\theta_{\mathrm{nc}}$ that yields non-collapsed assignments $\mathbf{q}^{(l)}$ and reconstructions $\hat{\mathbf{X}}(\mathbf{X})$.
Therefore $\mathcal{L}_{\text{recon}}(\theta_{\mathrm{nc}})
< \mathcal{L}_{\text{recon}}(\theta_{\mathrm{c}}).$
We now inspect the remaining terms in the objective.

\noindent \textbf{(i) Hierarchical KL and entropy terms.}
For a fully collapsed posterior, the hierarchical KL terms
$\mathrm{D_{KL}}(\mathbf{q}^{(l)}\|\mathbf{q}^{(l-1)})$ vanish only if all levels share exactly the same constant distribution; otherwise, they incur a positive penalty. Moreover, the entropy term $\mathbb{E}\!\left[ -\mathbf{q}^{(l)} \log \mathbf{q}^{(l)} \right]$ is minimized by near one-hot distributions. A fully collapsed solution that is both constant and sharply peaked is incompatible with representing
the variability in $\mathbf{X}$ and induces suboptimal hierarchical penalties.

\noindent \textbf{(ii) Text-conditioned and AR terms.}
For a collapsed posterior, assignments $\mathbf{q}^{(l)}$ are independent of the text embedding $\mathbf{e_t}$, \ie if $\mathbf{q}^{(l)}$ is constant, it cannot match varying text embeddings. Consequently, the text-conditioned KL terms
$\mathrm{D_{KL}}\!\left(\mathbf{q}_i \,\|\, \text{sg}(\mathbf{e_t}) \right)$ for $\mathbf{q}_i \in \mathbf{q}^{(l)}$ and the global
text--codebook alignment terms cannot be minimized across distinct texts.
Similarly, the autoregressive loss $\mathcal{L}_{\mathrm{AR}}$ cannot exploit visual or textual information because the discrete tokens do not depend on $(\mathbf{X},t)$. By contrast, a non-collapsed posterior can strictly reduce both.

Combining all pieces, $\mathcal{L}(\theta_{\mathrm{nc}}) < \mathcal{L}(\theta_{\mathrm{c}})$
since $\mathcal{L}_{\mathrm{recon}}$ is strictly lower and the remaining
terms can be made no worse, and typically strictly better, by making
assignments depend on $(\mathbf{X},t)$ while respecting regularizers.
Thus $\theta_{\mathrm{c}}$ cannot be a global minimizer of $\mathcal{L}$.\looseness-1
\vspace{-0.3cm}
\end{proof}
Proposition~\ref{prop:noncollapse} states that any fully collapsed
LaPQ posterior is suboptimal under the proposed objective, provided
natural structural assumptions on the data and model capacity. Therefore, gradient-based training of LaPQ is driven toward stationary points that preserve dependent discrete representations, thereby mitigating posterior collapse and encouraging high-utilization codebooks.\looseness-1

\FigSupsegone
\FigSupsegtwo

\FigSupvideounderstandVUone
\FigSupvideounderstandVUtwo
\FigSupvideounderstandVUthree

\FigSupactionone
\FigSupactiontwo

\FigSupFrameTVone

\FigSupFrameTVtwo

\FigSupHDTV

\FigSupFramebig
\FigSupvideoone
\FigSupvideotwo
\FigSupvideothree
\FigSupvideofour

\FigSupmagvae
\FigSupomnivae
\FigSupmotionvae

\newpage

\section{Additional Results}\label{app:results}
\subsection{Zero-shot Video Segmentation}\label{app:seg}
Given an input video and a natural language text $t$, we leverage the language-aligned discrete representation produced by \modelnamenc{} to obtain zero-shot, text-guided spatio-temporal masks. Specifically, we first pass the video through the frozen \modelnamenc{} encoder and its Language-aligned Pyramidal Quantization (LaPQ) hierarchy and extract the quantized features from the last quantization block, denoted by $\mathbf{q}^{(L)} \in \mathbb{R}^{T' \times H' \times W' \times d}$, which capture high-level, text-aligned semantics at a compressed spatio-temporal resolution. We then decompose the input text into a set of semantic units (typically content words or short phrases), $\{w_1,\dots,w_K\}$, and obtain a language embedding $\mathbf{e}_{w_k}$ for each unit using the same vision--language model employed during \modelnamenc{} training. For every semantic unit $w_k$, we compute a similarity score between $\mathbf{e}_{w_k}$ and each token in $\mathbf{q}^{(L)}$ (\eg via cosine similarity in the shared embedding space), yielding a token-level relevance map $\mathbf{S}_{w_k}^{\text{tok}}(t',h',w')$. This relevance map is then upsampled to the original video resolution, following the encoder downsampling pattern (or via decoder-aligned projection), to produce a dense per-pixel score volume $\mathbf{S}_{w_k}(x,y,t)$ for each semantic unit. We treat these volumes as unary potentials in a fully connected 3D Conditional Random Field (CRF) defined over the spatio-temporal lattice $(x,y,t)$, with pairwise terms encouraging spatial smoothness aligned to image edges and temporal consistency across adjacent frames. Running mean-field inference in this 3D-CRF refines the raw scores into a binary segmentation mask $\mathbf{M}_{w_k}(x,y,t) \in \{0,1\}$ that assigns each pixel in each frame to semantic unit $w_k$. Repeating this procedure iteratively over all semantic units in the prompt yields a set of
word-level, zero-shot, text-guided segmentation masks that are both spatially precise and temporally coherent across the video.\looseness-1

We compare our language-guided tokenizer with OmniTokenizer~\cite{wang2024omnitokenizer}, LARP~\cite{wang2025larp}, and VideoVAE+~\cite{xing2024large}, on diverse scenes in Fig.~\ref{fig:seg_sup_one} and novel-category examples in Fig.~\ref{fig:seg_sup_two}. Existing tokenizers often yield coarse, blob-like masks with strong label confusion: OmniTokenizer and LARP tend to over-smooth object boundaries and merge adjacent instances (\eg bus and road, trees and background), while VideoVAE+ frequently misses thin structures such as bike frames, surfboards, and traffic signs, or hallucinates spurious regions in uniform areas. These methods also struggle with rare or fine-grained concepts, leading to incomplete segmentation of small objects (\eg cat ears, surfboard tips) and inconsistent labeling across the image. In contrast, \modelnamenc{} produces masks that are both sharper and more semantically aligned with the ground truth, accurately separating foreground from background and preserving thin structures. Fig.~\ref{fig:seg_sup_two} further demonstrates strong zero-shot generalization: \modelnamenc{} cleanly segments unseen categories such as millennium falcon, tordelli, golden retriever, Pikachu, and axolotl, and simultaneously grounds multiple text queries (\eg ``golden retriever / puppy / grass field / vegetation'') in the correct regions, highlighting that our language-aligned tokens carry richer semantic information than prior VAE-based tokenizers.\looseness-1

\subsection{Video Question Answering}\label{app:vqa}
For all question answering results, we adopt Qwen2.5-VL-3B~\cite{bai2025qwen2} as the default
vision--language (VLM) backbone to generate open-ended answers from our video
representations. Given an input clip, we first encode the video with our proposed
\modelnamenc{} VAE and extract the discrete representations from \emph{all}
quantization blocks. These multi-scale features are projected into the language embedding space and prepended to the question tokens, yielding a unified conditioning sequence for the autoregressive decoder. The Qwen2.5-VL-3B model then performs conditional text generation to produce the final answer.
All VQA inferences are executed using the Text Generation Inference (TGI)
pipeline from HuggingFace~\cite{huggingface_tgi_docs}, which provides a stable and reproducible
deployment for our qualitative analysis.\looseness-1

Furthermore, across Fig.~\ref{fig:video_understand_supp_vu_one} to Fig.~\ref{fig:video_understand_supp_vu_three}, we compare \modelnamenc{} against Qwen2.5-3B, VideoVAE+, OmniTokenizer, and LARP on diverse video scenarios, including action sequences (helicopter crash, motorcycle chase, aircraft destruction), transformation events (monster emergence, firetruck-to-robot), and emotional interactions (a surprise proposal). The lower-capacity baselines (Qwen2.5-3B and VideoVAE+) often produce vague or partially incorrect explanations, while  OmniTokenizer and LARP capture events more reliably but still miss finer details. \modelnamenc{} consistently provides the most accurate, complete, and context-aware interpretations across all scenarios, demonstrating stronger temporal reasoning, causal understanding, and fine-grained visual grounding compared to competing models.\looseness-1

\subsection{Action Localization}\label{app:action}
We tackle temporal action localization in long, untrimmed videos by directly operating in the discrete latent space of \modelnamenc{}.  
Given a video of $N$ RGB frames and a textual description of the target action, we first encode every frame with our pyramidal tokenizer.  
Empirically, we observe that $\mathbf{q}^{(1)}$ offers the best trade-off between semantic expressiveness and temporal resolution: it preserves subtle motion cues (\eg arm swing before an arrow release, the instant of impact in a punch, see Fig.~\ref{fig:action_one}) that are strongly smoothed out in deeper levels.  
We therefore use $q^{(1)}$ as our frame-level features.
For each frame $t$, we spatially pool the tokens $\mathbf{q}^{(1)}_t$ (mean-pooling over space) to obtain a compact frame descriptor $\mathbf{v}_t \in \mathbb{R}^d$.  
The textual query is encoded by the same language backbone used for \modelnamenc{}’s cross-modal training, producing a normalized embedding $\mathbf{z} \in \mathbb{R}^d$.  
We compute cosine similarity scores $s_t\!=\!\langle \mathbf{v}_t, \mathbf{z} \rangle$ for all frames, which yield a dense text–video alignment signal over time.\looseness-1

To robustly localize an action interval, we evaluate similarities in a sliding-window fashion.  
The video is partitioned into overlapping chunks $(t, t+K-1)$ of length $K{=}25$ frames (with stride 1 in all experiments).  
For each chunk we aggregate the frame scores, $S_t\!=\!\frac{1}{K} \sum_{i=t}^{t+K-1} s_i$, resulting in a 1D confidence trajectory $\{S_t\}_{t=1}^{N-K+1}$ that reflects how strongly the query is grounded in each temporal neighborhood.  
We then decode this trajectory into contiguous segments using a longest-connected-sequence algorithm:  
(i) we threshold $S_t$ at a fixed confidence $\tau$ to obtain a binary sequence;  
(ii) identify all maximally connected high-confidence segments; and  
(iii) select the segment with the highest average score as the predicted action interval. For multi-action queries, we iteratively remove the selected interval and repeat, merging overlapping segments when necessary. The resulting segments define our temporal action predictions.\looseness-1

Fig.~\ref{fig:action_one} and Fig.~\ref{fig:action_two} visualize localized action segments for different tokenizers on several challenging examples.  
For each text query, the ground-truth (GT) segment is shown in blue, and the predictions obtained from VideoVAE~\cite{xing2024large},+, SweetTok~\cite{tan2025sweettok}, LARP~\cite{wang2025larp}, and \modelnamenc{} are displayed as colored bars beneath.  
The baselines consistently exhibit temporally diffuse and fragmented activations: their similarity signals tend to fire on visually similar but semantically off-target frames, producing multiple short segments or systematically shifted intervals.  For instance, in Fig.~\ref{fig:action_one}, in the clip  “A girl shoots an arrow”, both VideoVAE\,+ and SweetTok activate broadly over the whole sequence and fail to concentrate probability on the actual release moment, while LARP on several disjoint intervals before and after the shot. In contrast, \modelnamenc{} yields a single, compact segment that tightly aligns with the GT span around the arrow release.  
A similar pattern appears for text query “A person fires a shotgun”, where baseline tokenizers localize earlier or later segments, whereas \modelnamenc{} localizes correctly.\looseness-1

The advantages of our fine-grained features are even more evident for actions with multiple sub-events.  
In Fig.~\ref{fig:action_one} example “An MMA fighter knocks down his opponent with a punch to the face”  and  “…with a kick to the face”, the motion unfolds rapidly and is preceded by visually similar feints.  
VideoVAE\,+ and SweetTok tend to spread confidence over the entire exchange, leading to overly long or misaligned segments, while LARP often localizes only part of the motion (\eg the wind-up but not the impact).  
\modelnamenc{}, by contrast, localizes a concise window centered around the decisive contact, closely matching the GT.  
In Fig.~\ref{fig:action_two}, for  “A person performs two overhead presses”, \modelnamenc{} produces two high-confidence video segments that track both overhead press repetitions, whereas baselines either miss the second repetition or merge the two into one coarse interval.  
For complex, extended actions such as  “A man and a woman engage in sword fighting”  and  “Three missiles are launched from a desert”, baseline tokenizers again show scattered activations, localizing short segments around high-motion frames or transient explosions, and resulting in under-coverage of the GT.  
In contrast, \modelnamenc{} yields more accurate localization.\looseness-1

\subsection{Text-2-Video Generation}\label{app:t2v}
To assess the usefulness of our tokens for generative modeling, we couple \modelnamenc{} with a conditional video decoder built on Qwen-2.5VL~\cite{bai2025qwen2}. Concretely, we treat the text encoder of Qwen-2.5VL as a frozen condition network and fine-tune its video decoder to autoregressively predict \modelnamenc{} codes. Given a textual prompt, we first encode the prompt into language features, which are injected into a transformer-based decoder that models the joint distribution over all spatio–temporal tokens from our four quantizers. The decoder predicts the next token conditioned on the text and all previously generated tokens, until a full sequence of discrete video codes is obtained. These codes are then passed through the \modelnamenc{} VAE decoder to synthesize the final video. Thanks to \modelnamenc{}'s compact yet expressive representation, this pipeline can generate videos at 20 FPS with resolutions up to 4K. \looseness-1

Fig.~\ref{fig:frame_tv_one} and Fig.~\ref{fig:frame_tv_two} show qualitative comparisons on text-to-video generation where we keep the Qwen-2.5VL decoder architecture fixed and only swap the underlying tokenizer. OmniTokenizer and LARP tend to under-utilize fine-grained textual cues, often missing localized attributes such as the “two egg halves’’ in the ramen bowl or the “motion blur on pedestrians’’ in the neon street scene, and producing over-smoothed or distorted structures in complex compositions like the tree city and Mars spaceport. 
SweetTok better preserves global layout but still struggles with high-frequency details and subtle style descriptors (\eg HDR interior lighting, crisp spray around the polar bear), leading to muted textures and inconsistent object shapes.\looseness-1 

In contrast, \modelnamenc{} yields samples that more faithfully reflect both global scene descriptions and fine-grained phrases in the prompts. The additional objects specified in the text appear at the correct locations, motion-related cues are rendered more plausibly, and material and lighting properties (glossy chocolate surface, bioluminescent foliage, cinematic city glow) are captured with higher fidelity. Fig.~\ref{fig:HDTV} further illustrates 4K text-to-video generation for a 3-second clip, where \modelnamenc{} renders fine-grained details and maintains sharp structures, demonstrating that our multi-scale quantization supports high-resolution, text-aligned video synthesis.

\subsection{High-resolution Frame Reconstruction}\label{app:recon}
We further evaluate \modelnamenc{} on 4K frame reconstruction in Fig.~\ref{fig:frame_reconcs_sup}. At this resolution, prior tokenizers struggle to preserve fine structures and high-frequency textures. VideoVAE+~\cite{xing2024large} produces strong over-smoothing: the coral branches, tree leaves, and fur on the buffalo become noticeably blurred, and small objects such as distant boats and fire lamps nearly vanish in the zoomed-in crops. OmniTokenizer~\cite{wang2024omnitokenizer} improves sharpness but introduces ringing and aliasing along high-contrast boundaries (\eg the product watch edges and mountain silhouettes), and often exhibits color bleeding in specular regions. SweetTok~\cite{tan2025sweettok} and LARP~\cite{wang2025larp} retain more detail yet still suffer from blocky artifacts on repetitive textures (grass, foliage, brick walls) and inconsistent reconstruction of tiny highlights, such as reflections on the watch bezel and lights on the night harbor.
In contrast, our \modelnamenc{} reconstructions remain consistently crisp and coherent. Objects across all scenes—from coral polyps and reef fish to product shots and distant architectural details—retain sharp contours and clean textures without haloing. Fine-grained elements such as fur strands, leaf veins, and small fruits are faithfully preserved, demonstrating that our pyramidal tokenization scales effectively to ultra-high resolutions while avoiding the blurring and aliasing present in prior methods.

In qualitative video reconstruction comparisons (Figs.~\ref{fig:video_recon_one}–\ref{fig:video_recon_four}), existing tokenizers show consistent limitations across diverse scenes. TokenFlow~\cite{qu2025tokenflow} and SweetTok often oversmooth high-frequency content, causing foliage, clothing textures, and facial details to blur, and small or thin structures to distort or disappear; they also introduce blocky artifacts under large motion. LARP better preserves local contrast but frequently produces ringing around boundaries and unstable illumination, leading to flickering highlights and shadows. MotionAura~\cite{susladkar2025motionaura} improves temporal smoothness yet still suffers from identity drift in talking-head sequences and ghosting around fast movements, reducing perceptual realism. Moreover, as previous methods were trained on low-resolution data, their high-resolution reconstructions exhibit temporal artifacts such as frame stuttering. In contrast, our 4K-trained \modelnamenc{} preserves high-frequency detail and temporal coherence, producing smooth and stable video. \looseness-1

\TableCompression

\subsection{Adapting Pretrained T2V Priors with \modelnamenc{}}\label{app:t2vprior} 
We further study whether \modelnamenc{} can serve as a drop-in tokenizer for existing text-to-video priors. To this end, we replace the original VAE/tokenizer in three pretrained models, \ie Open source version of MAGVIT-V2~\cite{luo2024open, yu2024language} and OmniGenV2~\cite{wu2025omnigen2} (autoregressive priors) and MotionAura~\cite{susladkar2025motionaura} (discrete diffusion prior), and fine-tune only the prior on 10k clips from OpenVid-1M~\cite{nan2025openvid} so that it models \modelnamenc{} codes. Under identical prompts and sampling hyper-parameters, and across all architectures, using the native tokenizer leads to typical failure modes: colors and exposure drift over time, geometry “breathes’’ (\eg wobbling backgrounds and object contours), high-frequency details such as dough surface texture or water droplets quickly collapse into smooth blobs, and object semantics are weakly preserved (\eg inconsistent shape of the claw-machine robot or citrus slices). 
After swapping in \modelnamenc{}, the same priors produce videos that are both more semantically aligned with the prompts and markedly more temporally consistent.\looseness-1

In \Cref{fig:our_mag}, MAGVITv2+\modelnamenc{} maintains stable neon lighting in the arcade, preserves the dough’s volume and hand pose across frames, and keeps the boiling dumplings sharp with coherent bubble motion. In \Cref{fig:our_omni}, OmniGenV2+\modelnamenc{} yields crisp tree trunks and facial details with reduced frame-to-frame jitter, while the splashing juice exhibits smoother, physically plausible trajectories. 

Similarly, MotionAura+\modelnamenc{} in \Cref{fig:our_motion} suppresses diffusion-induced flicker in backgrounds. These improvements indicate that \modelnamenc{}’s multi-scale discrete representation reduces quantization artifacts and exposes a more structured latent space, making it easier for both autoregressive and diffusion priors to model long-range spatio-temporal dependencies and maintain object identity over time, even with minimal fine-tuning data.

\section{Ablations and Additional Analyses}\label{app:ablations}
\subsection{Video Compression}
As reported in Table~\ref{tab:mcljcv_analysis}, \modelnamenc attains the lowest LPIPS and competitive PSNR/SSIM on MCL-JCV~\cite{wang2016mcl} at a bitrate of 0.034, surpassing traditional codecs like HEVC~\cite{sullivan2012overview} and VCC~\cite{bross2021overview} in perceptual fidelity (LPIPS) by preserving fine texture and temporal coherence through semantically guided quantization.\looseness-1

\Tablesupclass

\Tablesupvq

\subsection{Video Generation}
We evaluate our tokenizer and generator on class-guided video generation using the UCF-101~\cite{soomro2012ucf101} dataset. Given a target action class, the model is conditioned on the class label and asked to synthesize a short video clip from scratch. This setting measures not only low-level visual fidelity (appearance, motion smoothness) but also whether the generated sequence is semantically consistent with the requested action category.\looseness-1

We compare \modelnamenc{} against a broad set of video generative models that rely on different tokenizers and generator architectures, including MAGVIT/MAGVIT-V2, LARP-L, CogVideo, TATS, Video-LaVIT, OmniTok, and SweetTok. For all methods, we report the generative Fréchet Video Distance (gFVD), where lower values indicate better alignment with the distribution of real videos. As shown in \Cref{tab:ucf_class_guided}, our method achieves the lowest gFVD on UCF-101, improving upon the strongest prior tokenizer by a substantial margin. These results indicate that our representation is better suited for high-quality, temporally coherent video synthesis, and that scaling the generator on top of our tokens directly translates into stronger video generation performance.

\subsection{Ablation on VQ Techniques}
\Cref{tab:vqablation} presents an ablation of the quantization module in \modelnamenc{}, where each row corresponds to a different way of discretizing the encoder features, specified by its quantization type, vocabulary size, and embedding dimensionality. The simple single-codebook baseline VQ~\cite{van2017neural} (4096 / 256), with vocal size of 4096 and a dimension of 256, yields the weakest reconstruction quality on both COCO~\cite{lin2014microsoft} and WebVid~\cite{bain2021frozen}, confirming that a single global codebook is insufficient to capture the variability of natural image–video data. Introducing a group structure in GVQ~\cite{jang2017categorical} (4096 / 256) slightly improves PSNR and SSIM, and reduces LPIPS; however, the gains are modest because each group still operates with a relatively small shared codebook. The lookup-free single-block variant, LFQ~\cite{yu2024language} (32800 / 16), increases the effective vocabulary while reducing the per-code dimension, resulting in a clear improvement in PSNR and SSIM, and a lower LPIPS, indicating that finer local code assignment is beneficial.\looseness-1
 
Residual quantization with a higher-dimensional code space, RVQ~\cite{lee2022autoregressive} (8000 / 512), further reduces distortion over vanilla VQ, and replacing the residual codebook with our latent product quantizer, LaPQ (8000 / 512), yields another consistent improvement, showing that decomposing the latent channels into product codebooks makes better use of the same vocabulary size. When we combine residual modeling with LFQ-style factorization, RVQ (32800 / 16) achieves even better performance, but our full LaPQ (Ours, 48000 / 16) achieves the best performance across all metrics on both validation sets, with the highest PSNR/SSIM and lowest LPIPS, while incurring only a small increase in inference time compared to simpler schemes. Overall, results demonstrate that LaPQ’s combination of lookup-free factorization and product–residual coding provides a significantly more expressive and distortion-resilient discrete representation than standard VQ, GVQ, LFQ, or RVQ under comparable computational budgets.\looseness-1

\subsection{Codebook Utilization  vs.\ Resolution}
We further analyze how the effectiveness of our tokenizer scales with input resolution by measuring the percentage of active codewords at different spatial resolutions (see Fig.~\ref{fig:codebook_resolution}). As the resolution increases from $240$p to $4320$p, codebook utilization rises monotonically from $55.23\%$ to $97.12\%$, indicating that higher-resolution inputs naturally excite a richer subset of the learned vocabulary rather than collapsing to a small set of frequently used tokens. In particular, utilization already exceeds $79\%$ at $1080$p and surpasses $90\%$ in the $4$K regime ($2160$p and $4320$p), suggesting that the proposed pyramidal design can effectively exploit the larger spatial support to express more diverse and fine-grained semantics. This trend confirms that our discrete latent space remains expressive and well-populated as we scale to high-resolution video, which is critical for both faithful reconstruction and downstream video-language understanding tasks.\looseness-1

\FigSupcodebookutio

\subsection{Ablation on Losses for Video Understanding}
We ablate each component of the training objective on three video understanding benchmarks: THUMOS14~\cite{idrees2017thumos} and ActivityNet v1.3~\cite{caba2015activitynet} for temporal action localization, and MVBench for video question answering (Table~\ref{tab:lossund}).
With the full objective, \modelnamenc{} achieves $33.17/29.11$ Avg.\ mAP on THUMOS14/ActivityNet and $86.03$ mAP on MVBench.
Removing the drift regularizer $\mathcal{L}_{\text{drift}}$, which anchors the adapted encoder to the pretrained VAE manifold, degrades performance by $1.90/1.49$ mAP on THUMOS14/ActivityNet and by $2.71$ points on MVBench, indicating that maintaining a stable latent space is important for robust transfer across both localization and QA.\looseness-1

The autoregressive alignment loss $\mathcal{L}_{\text{AR}}$ has a different effect: dropping it leads to a relatively small drop on temporal localization ($0.72/1.13$ mAP), but causes a pronounced $6.58$-point decline on MVBench.
This suggests that sequence-level token modeling is especially critical for high-level video reasoning, where the model must integrate information over longer temporal horizons.
When we remove both the DINO-guided visual loss and the autoregressive loss ($\mathcal{L}_{\text{dino}} + \mathcal{L}_{\text{AR}}$), performance drops most severely on TAL (by $3.88$ and $2.33$ mAP on THUMOS14 and ActivityNet, respectively) and by $4.46$ points on MVBench, highlighting the complementarity between discriminative visual supervision and global token prediction.\looseness-1
\Tablesupunderstanding

We further study the codebook-related objectives, as described in \Cref{eq:codebook}.
Ablating the text-conditioned alignment term $\mathcal{L}_{\text{text-cond.}}$ reduces performance by $2.95/1.56$ mAP on THUMOS14/ActivityNet and by $2.47$ points on MVBench, while removing the text--codebook alignment $\mathcal{L}_{\text{text-codebook}}$ yields a similar degradation ($2.06/2.04$ mAP and $2.12$ points).
These results confirm that both local token--text alignment and global codeword--text alignment are necessary to maintain semantically structured latents that generalize well across detection and QA tasks.
In contrast, dropping the vision-comment loss $\mathcal{L}_{\text{vision\_commitment}}$ produces the smallest degradation (at most $0.50/0.90$ mAP on THUMOS14/ActivityNet and $1.80$ points on MVBench), suggesting that, for downstream understanding, the semantic shaping of the codebook is more critical than the pure vision commitment penalty.
Overall, the complete loss formulation is consistently superior, validating our multi-part objective for unified video understanding.\looseness-1

\section{Implementation Details}\label{app:impl}
\modelnamenc is implemented using the pretrained Wan 2.2L~\cite{wan2025wan} video VAE as the backbone to ensure high-fidelity visual reconstruction. We initialize the encoder with pretrained WAN-2.2 weights, while the LaPQ module and decoder are randomly initialized. Both the encoder and decoder of Wan 2.2L are kept frozen to preserve the pretrained visual quality. To encourage the model to capture long-range temporal dependencies and motion continuity, we temporally mask 30\% of frames and apply cosine-based spatial masking on each frame following~\cite{he2022masked}.  

To enable efficient adaptation to our multi-scale semantic learning objective without full fine-tuning, we incorporate LoRA adapters~\cite{hu2022lora} with rank 16 and alpha 32 into all encoder blocks. 
These adapters provide lightweight parameterization while preserving the representational capacity of the backbone. For text conditioning, we employ the Qwen2.5-VL (3B)~\cite{bai2025qwen2}, referred as pretrained VLM in main paper, to extract semantically rich textual embeddings that guide both the quantization and the multimodal semantic alignment. Loss weights are set to $\lambda_{\text{recon}}{=}2.5$, $\lambda_{\text{codebook}}{=}2.5$, $\lambda_{\text{AR}}{=}1.5$, and $\lambda_{\text{drift}}{=}0.6$.  To reduce memory footprint and accelerate training, we apply VAE tiling for latent-space tokenization and quantize the alignment VLM to AWQ INT-4~\cite{lin2024awq}. In \modelnamenc{}, {En}($\cdot$) refers to the frozen {DINOv3}~\cite{dinov3} encoder, which serves as a strong pretrained visual encoder. It is used to provide stable, high-quality visual features that anchor adaptation and prevent drift from the pretrained visual manifold.\looseness-1

All baselines are trained under identical dataset settings to ensure fair comparison. The average prompt length during training is $\sim$60 tokens. Training is conducted in three progressive stages, each designed to incrementally strengthen multimodal alignment and visual–temporal consistency.

\noindent \textbf{Stage 1 — Self-Supervised Pretraining.}
In the first stage, we perform self-supervised pretraining focused on language alignment. 
Input spatial resolutions vary from $512 \times 512$ up to $2048 \times 2048$, and we train across multiple aspect ratios, including $1{:}1$, $4{:}3$, $3{:}2$, $16{:}9$, and $2{:}1$.  
For temporal modeling, the number of frames ranges from $16{+}1$ to $96{+}1$, where the additional frame denotes the conditioning key frame. This stage establishes robust cross-modal grounding and spatial-temporal coherence.

\noindent \textbf{Stage 2 — Text–Visual Token Alignment.}
The second stage incorporates text–visual token alignment through the pretrained {Qwen-2.5-VL (3B)} backbone. We maintain the same spatial and temporal configurations as Stage~1 for training stability. 
This stage refines the alignment between linguistic tokens and visual embeddings, enhancing the semantic consistency of multimodal representations.

\noindent \textbf{Stage 3 — Full-Scale Fine-Tuning.}
In the final stage, the model is exposed to multi-resolution and multi–aspect-ratio inputs, ranging from $128 \times 128$ to $4096 \times 4096$, covering the same aspect ratios ($1{:}1$, $4{:}3$, $3{:}2$, $16{:}9$, $2{:}1$). The number of frames is kept consistent with previous stages.  
Due to increased resolution and GPU memory demands, the batch size is reduced from 4~$\rightarrow$~2 per GPU. This stage optimizes the model with both alignment loss and a frame-level retention loss computed using {DINOv3}~\cite{dinov3}, ensuring long-range temporal retention and fine-grained visual correspondence.\looseness-1

All training stages are optimized using {AdamW}~\cite{loshchilov2019decoupled} with an initial learning rate of $1\times10^{-5}$ and a cosine annealing scheduler. Gradient accumulation steps are kept constant across stages. We train on a cluster of 128$\times$NVIDIA A100  (80~GB) GPUs. The total number of optimization steps is 30K for Stage~1, 60K for Stage~2, and 180K for Stage~3.

\FigSuptrainingplot

\section{Datasets}\label{app:data}
To comprehensively train, validate, and evaluate \modelnamenc{}, we employ a diverse collection of large-scale video--text datasets spanning various resolutions, domains, and task-specific settings. 

\subsection{Training Datasets}

\textbf{Droplet-10M~\cite{zhang2025dropletvideo} (Subset).}
We curate a subset of the {Droplet-10M} dataset, consisting of approximately {4--5 million HD videos (720p)}. 
This subset serves as the foundation for pretraining, providing broad coverage of human activities, natural scenes, and diverse motion patterns, and a dense caption distribution and consistent temporal dynamics, crucial for learning fine-grained video--text alignment.
To ensure data quality and maintain high spatial fidelity, only videos at 720p or higher resolution are retained.

\noindent \textbf{OpenVid-1M~\cite{nan2025openvid} (300K Subset).}
We supplement training with {300K high-quality video--caption pairs} sampled from {OpenVid-1M}. Only HD videos are selected to maintain visual consistency. This subset contributes to expanding linguistic diversity and contextual variation, improving open-domain caption understanding and cross-modal reasoning.

\noindent \textbf{UltraVideo~\cite{xue2025ultravideo} (40K with Reconstructed Captions).}
To enrich representation at extreme resolutions, we incorporate {40K ultra--high-definition videos (4K and 8K)} from the {UltraVideo} dataset. 
Since many of these videos lack high-quality textual descriptions, we generate captions using a multimodal LLM pipeline. 
This enables the model to learn from high-fidelity visual data and supports scalability to higher-resolution downstream applications.

\subsection{Testing and Validation Datasets}

\textbf{OpenVid-1M~\cite{nan2025openvid}  (Test Split).}
We employ {100K samples} from the {OpenVid-1M} test split for evaluating generalization to unseen open-domain video--text pairs. 
This ensures consistency with the distribution of the training data while validating model generalization under identical data conditions.\looseness-1

\noindent \textbf{WebVid-10M~\cite{bain2021frozen} (Validation)  and COCO~\cite{lin2014microsoft} (Validation).} 
For generative evaluation, we follow the {WebVid-10M} and COCO-Val validation protocols. For class-guided video generation, we further evaluate on the \textbf{UCF-101}~\cite{soomro2012ucf101} dataset.

\noindent \textbf{MCL-JCV~\cite{wang2016mcl} (Compression Validation).}
To assess the effectiveness of our video compression and reconstruction, we employ the MCL-JCV benchmark. This dataset provides a controlled setup for evaluating perceptual quality and rate--distortion tradeoffs under varying compression levels.

To evaluate generalization beyond supervised training, we test the model under zero-shot conditions across diverse downstream video understanding tasks.
For zero-shot action localization, we evaluate on \textbf{ActivityNet}~\cite{caba2015activitynet} and \textbf{THUMOS14}~\cite{idrees2017thumos}, which contain diverse and complex activities. 
For zero-shot video segmentation, we benchmark on \textbf{YouTube-VIS 2021}~\cite{yang20213rd}  and \textbf{OVIS}~\cite{qi2022occluded}. Both datasets present challenging dynamic scenes with multiple interacting objects and frequent occlusions. For video classification, we utilize \textbf{Kinetics}~\cite{kay2017kinetics}, while for VideoQA, we adopt \textbf{MVBench}~\cite{li2024mvbench}, a comprehensive multi-task benchmark covering spatiotemporal reasoning, action understanding, and commonsense interpretation.\looseness-1